\documentclass{article}

\usepackage{iclr2027_conference,times}
\usepackage{tabularx}
\usepackage{makecell}
\usepackage[utf8]{inputenc}
\usepackage[T1]{fontenc}
\usepackage{hyperref}
\usepackage{url}
\usepackage{booktabs}
\usepackage{multirow}
\usepackage{amsmath}
\usepackage{amsfonts}
\usepackage{nicefrac}
\usepackage{microtype}
\usepackage{xcolor}
\usepackage{graphicx}
\usepackage{subcaption}
\usepackage{tikz}
\usepackage{titletoc}
\usetikzlibrary{arrows.meta,calc,decorations.pathreplacing,positioning}
\newtheorem{theorem}{Theorem}
\newtheorem{proposition}{Proposition}
\newtheorem{lemma}{Lemma}

\newcommand{\vt}[1]{\mathbf{#1}}
\DeclareMathOperator{\softmax}{softmax}

\title{Score Broadcast and Decorrelation: A General Framework for Broadcast-Based Credit Assignment}

\author{%
\parbox{\dimexpr\textwidth-2\tabcolsep\relax}{%
\centering
\bfseries
Mustafa Uzun$^{1,2}$ \quad
Mete Erdogan$^{3}$ \quad
Cengiz Pehlevan$^{4,5,6}$ \quad
Alper T. Erdogan$^{1,2}$
\\[5pt]
\normalfont\small
$^{1}$KUIS AI Center, Koc University, Turkey
\quad
$^{2}$Electrical and Electronics Engineering, Koc University, Turkey
\\
$^{3}$Department of Electrical Engineering, Stanford University, USA
\\
$^{4}$John A. Paulson School of Engineering \& Applied Sciences, Harvard University, USA
\\
$^{5}$Kempner Institute for the Study of Natural and Artificial Intelligence, Harvard University, USA
\\
$^{6}$Center for Brain Science, Harvard University, USA
\\[4pt]
\footnotesize
\texttt{muzun22@ku.edu.tr}
\quad
\texttt{merdogan@stanford.edu}
\\
\texttt{cpehlevan@seas.harvard.edu}
\quad
\texttt{alperdogan@ku.edu.tr}
}%
}

\iclrfinalcopy

\begin{document}

\maketitle

\fancyhead{}
\renewcommand{\headrulewidth}{0pt}
\begin{abstract}
We introduce \textit{Score Broadcast and Decorrelation} (SBD), a principled framework for broadcast-based credit assignment for general families of differentiable losses. Error broadcast is a biologically plausible alternative to backpropagation that sends output information to hidden layers without weight transport. The Error Broadcast and Decorrelation (EBD) framework, recently introduced for the mean-squared-error (MSE) setting, grounded this mechanism in the stochastic orthogonality of optimal estimators, under which the error of the optimal estimator is orthogonal to functions of the input. We generalize that foundation by introducing an orthogonality principle between the \textit{output score} (the gradient of loss with respect to the final-layer output) and hidden-layer activations, which holds whenever the optimal score has conditional mean zero. This single principle unifies broadcast-based credit assignment across the standard differentiable-loss families, including cross-entropy, Bregman divergences (with MSE as a special case), proper scoring rules through unconstrained links, and exponential-family negative log-likelihoods. The framework supplies a theoretical grounding for the three-factor learning rule under general losses, with the neuromodulatory factor derived as the broadcast loss score rather than postulated. We derive the cross-entropy case explicitly, characterize the admissible loss class, and introduce a \textit{score vector expansion} technique that enriches the broadcast signal while preserving the orthogonality framework. Experiments on CIFAR-10 and Tiny ImageNet show that SBD substantially improves over existing broadcast approaches, with score vector expansion delivering further gains. Overall, this work identifies the loss score as the signal to broadcast, supplies the orthogonality theory and theoretical grounding for the three-factor learning rule from neuroscience, and shows how score vector expansion enriches the decorrelation directions of the resulting objective.\looseness=-1
\end{abstract}

\section{Introduction}
\label{sec:introduction}

Neural networks serve as the fundamental computational models of natural
intelligence and form the core engine of modern AI. Across both domains, a key
challenge is the credit assignment problem: how to update local
synaptic weights to optimize a global performance metric. The dominant solution
in machine learning is backpropagation (BP), which derives exact gradients by
transmitting output errors backward through the network \citep{Rumelhart:86}.
Although highly effective, BP requires symmetric backward pathways and exact
weight transport, constraints long viewed as
biologically implausible \citep{crick1989recent,lillicrap2020backpropbrain}.\looseness=-1

These constraints have driven the search for alternative credit assignment
frameworks that relax the strict routing requirements of BP
\citep{whittington2019theories}. One notable family of such approaches relies
on \textit{error broadcast}: distributing global output information directly to
hidden layers rather than propagating it sequentially. Although attractive and
biologically plausible, this raises a fundamental question: \emph{ for a given loss, what specific quantity should be broadcasted, and what theoretical principle justifies that such a decentralized mechanism can drive learning? }\looseness=-1

The \textit{Error Broadcast and Decorrelation} (EBD) framework
\citep{erdoganerror} provided a principled answer in the mean-squared-error
(MSE) setting. Its starting point is the stochastic orthogonality property of
MMSE estimation: at optimality, the estimation error is orthogonal to every
input-measurable function whose product with the error is integrable. EBD turns this into layerwise decorrelation
objectives between hidden activations and the broadcast output error,
yielding local \textit{three-factor learning rules}, a long-standing class of
biologically motivated synaptic update rules, in which a presynaptic activity
term and a postsynaptic sensitivity term are gated by a third, neuromodulatory
factor \citep{fremaux2016neuromodulated, gerstner2018eligibility,  kusmierz2017learning,schultz1998predictive}. However, this
foundation is specific to squared error. In classification, the standard objective is cross entropy (CE), and more generally one wishes to optimize
differentiable losses for which the MSE output error is not the natural
error signal. A general theory of broadcast learning therefore requires a
loss-dependent notion of what should be broadcast.\looseness=-1

In this article, we identify that quantity as the \textit{output score}, defined
as the gradient of the loss with respect to the final layer output. For cross
entropy, this score is $\boldsymbol{\delta}=\mathbf{p}-\mathbf{y}$. We show that, at
the population cross entropy optimum, this score is conditionally mean zero and
therefore orthogonal to any deterministic function of the input, including hidden
layer activations. More generally, the same principle applies to differentiable
losses whose conditional risks are characterized by a zero score
condition. Thus, the MSE output error used by EBD and the cross entropy score used
in classification are instances of the more general score-based orthogonality principle.

\begin{figure}[tt]
    \centering
    \begin{subfigure}[t]{0.62\linewidth}
        \centering
        \resizebox{\linewidth}{!}{    \begin{tikzpicture}[
        font=\normalsize,
        >=Stealth,
        neuron/.style={circle, draw=black!75, line width=0.8pt, minimum size=17mm, inner sep=0pt},
        presyn/.style={neuron, top color=gray!5, bottom color=gray!30},
        postsyn/.style={neuron, top color=blue!5, bottom color=blue!30},
        scoreblock/.style={rectangle, draw=black!65, line width=0.65pt, minimum width=26mm, minimum height=7.5mm, inner sep=2.5pt, font=\small, fill=orange!14},
        modblock/.style={rectangle, rounded corners=1.2pt, draw=black!65, line width=0.55pt, minimum width=32mm, minimum height=8mm, inner sep=2.5pt, font=\small, fill=orange!14},
        synapse/.style={line width=3.4mm, gray!50, -{Triangle[length=4.6mm,width=5.8mm]}},
        broadcast/.style={line width=1.0pt, draw=black!75},
        sectionlabel/.style={font=\normalsize\itshape, color=black!70}
    ]

    \node[sectionlabel] at (0.0, 2.85) {Pre-synaptic};
    \node[sectionlabel] at (3.0, 2.85) {Post-synaptic};
    \node[sectionlabel] at (7.4, 2.85) {Output score};

    \node[presyn] (pre)  at (0.0, 0.7) {$h_j^{(k-1)}$};
    \node[postsyn] (post) at (3.0, 0.7) {$h_i^{(k)}$};

    \draw[synapse] (pre.east) -- (post.west) coordinate[pos=0.55] (synmid);
    \node[font=\small, color=black!65, above=2.5mm of synmid] {synapse};

    \node[scoreblock] (score1) at (7.4,  1.95) {$\boldsymbol{\delta}$};
    \node[scoreblock] (score2) at (7.4,  0.95) {$\phi_2(\mathbf{x}) \odot \boldsymbol{\delta}$};
    \node[font=\Large, color=black!70] (dots) at (7.4, 0.0) {$\vdots$};
    \node[scoreblock] (scoreM) at (7.4, -0.95) {$\phi_M(\mathbf{x}) \odot \boldsymbol{\delta}$};

    \draw[decorate, decoration={brace, amplitude=5.5pt}, line width=0.55pt, draw=black!70]
        (score1.north east) -- (scoreM.south east);
    \node[font=\small, color=black!75, anchor=west] at
        ($(score1.north east)!0.5!(scoreM.south east) + (0.45, 0)$)
        {$\tilde{\boldsymbol{\delta}} \in \mathbb{R}^{M D_{\mathrm{out}}}$};

    \node[modblock] (proj) at (3.0, -1.85)
        {$q_i^{(k)} = \hat{\tilde{\mathbf{R}}}^{(k)}_i\, \tilde{\boldsymbol{\delta}}$};

    \coordinate (busTop) at (5.95,  1.95);
    \coordinate (busBot) at (5.95, -0.95);

    \draw[broadcast] (score1.west) -- (busTop);
    \draw[broadcast] (score2.west) -- (5.95, 0.95);
    \draw[broadcast] (scoreM.west) -- (busBot);
    \draw[broadcast] (busTop) -- (busBot);

    \draw[broadcast, ->] (5.95, -1.85) -- (proj.east);
    \draw[broadcast] (busBot) -- (5.95, -1.85);

    \coordinate (synBottom) at ($(synmid) + (0, -0.36)$);
    \draw[broadcast, ->] (proj.north) -- ++(0, 0.45) -| (synBottom);

    \node[font=\normalsize] at (3.35, -3.75) {%
        $\Delta W_{ij}^{(k)} \;\propto\;
        \tikz[baseline]\node[fill=gray!15, rounded corners=1pt, inner sep=2.5pt, draw=black!50, line width=0.4pt]{$h_j^{(k-1)}$};
        \;\cdot\;
        \tikz[baseline]\node[fill=orange!18, rounded corners=1pt, inner sep=2.5pt, draw=black!50, line width=0.4pt]{$q_i^{(k)}$};
        \;\cdot\;
        \tikz[baseline]\node[fill=blue!18, rounded corners=1pt, inner sep=2.5pt, draw=black!50, line width=0.4pt]{$g_i^{\prime(k)}(h_i^{(k)})\, f^{\prime(k)}(u_i^{(k)})$};$
    };

    \end{tikzpicture}}
        \caption{Three-factor SBD update with score-vector expansion. At layer $k$, the synaptic update uses presynaptic activity $h_j^{(k-1)}$, postsynaptic sensitivity $g_i^{\prime(k)}(h_i^{(k)})\, f^{\prime(k)}(u_i^{(k)})$, and a broadcast modulation $q_i^{(k)}$ obtained by projecting the expanded score vector $\tilde{\boldsymbol{\delta}} = [\boldsymbol{\delta};\, \phi_2(\mathbf{x}) \odot \boldsymbol{\delta};\, \ldots;\, \phi_M(\mathbf{x}) \odot \boldsymbol{\delta}]$, where $\boldsymbol{\delta} = \nabla_{\mathbf{a}} \mathcal{L}$ is the output score.}
        \label{fig:sbd-three-factor}
    \end{subfigure}\hfill
    \begin{subfigure}[t]{0.34\linewidth}
        \centering
        \includegraphics[width=\linewidth]{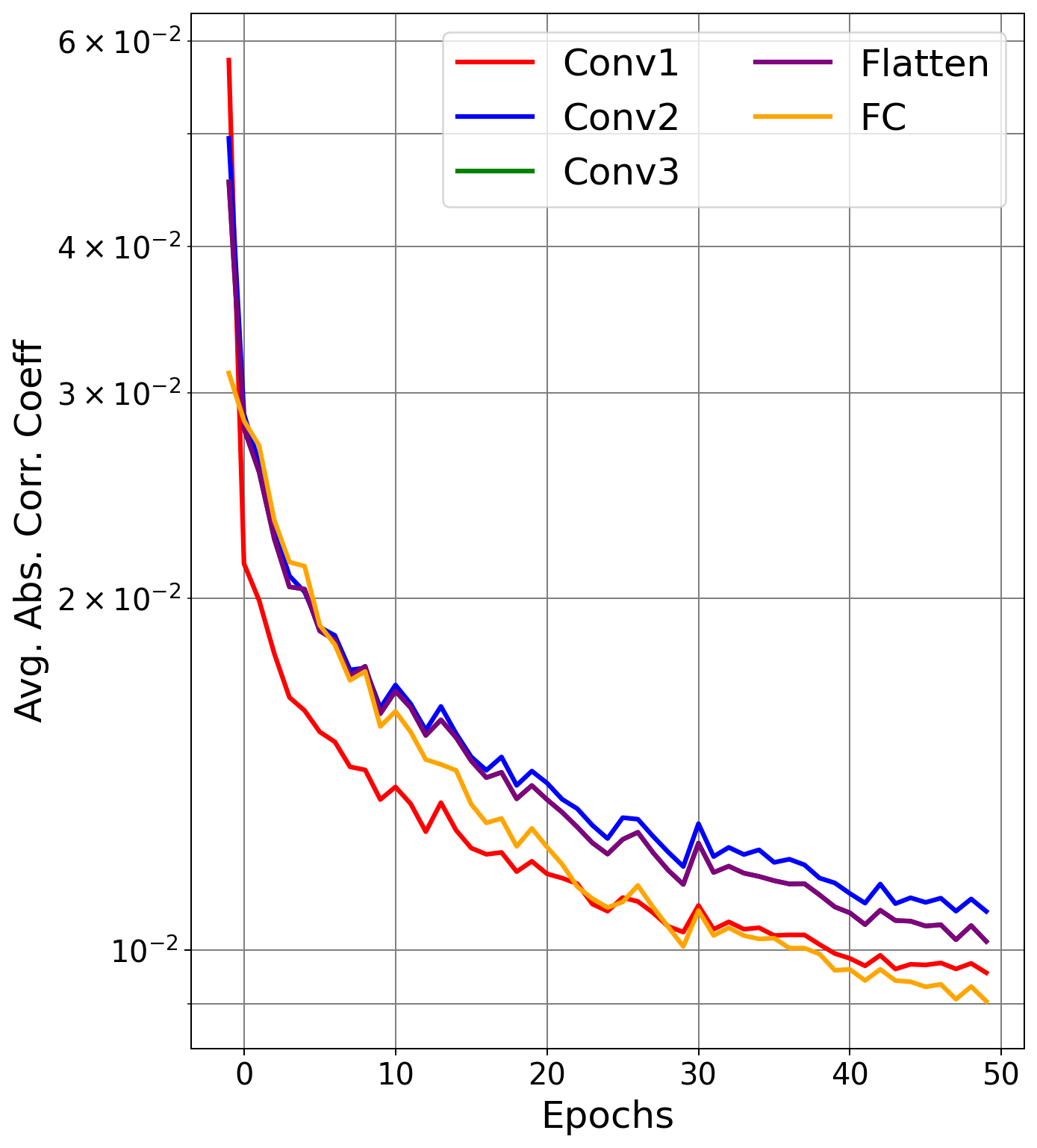}
        \caption{Empirical correlations between the cross entropy output score and hidden-layer activations in a backpropagation-trained CIFAR-10 CNN. These correlations decrease during optimization, consistent with Theorem~\ref{thm:ce_orthogonality}.}
        \label{fig:cnn-cifar10-correlations}
    \end{subfigure}
   \caption{Illustrations of the SBD framework and the empirical score--activation orthogonality.}
    \label{fig:intro-figures}
\end{figure}

Building on this principle, we formulate \textit{Score Broadcast and Decorrelation} (SBD), a broadcast-and-decorrelate framework for general differentiable losses that supplies a unified theoretical grounding for the three-factor learning rule introduced above, with the neuromodulatory factor \textit{derived} as the broadcast loss score rather than postulated; Figure~\ref{fig:sbd-three-factor} depicts the resulting update, including the score-vector expansion of Section~\ref{sec:score_expansion}. The MSE-based EBD update is recovered as one specific instance.\looseness=-1

We also introduce score vector expansion. The  output score has dimension
equal to the number of output coordinates, which can limit the number of
independent decorrelation directions available to wide hidden layers. SBD can
expand the broadcast vector by multiplying the score with deterministic
modulators, such as functions of the predictive distribution. These modulated
scores preserve the conditional mean zero property at the population
optimum while providing a richer set of decorrelation directions.\looseness=-1

We evaluate SBD on CIFAR-10 and Tiny ImageNet using CNN architectures matched to
the original EBD setting. The experiments are intended as controlled proof of
concept tests within broadcast learning rather than as claims of superiority over
exact BP. BP is included as a reference optimizer, while the main comparisons are
to MSE based EBD and other broadcast baselines. In this setting, SBD improves over other broadcast formulations, and score vector expansion provides additional gains.\looseness=-1

\subsection{Related work}

Alternatives to BP span several major families, including \textit{contrastive methods} such as
Equilibrium Propagation \citep{scellier2017equilibrium,scellier2019equivalence}, \textit{target
propagation} \citep{bengio2014auto,lee2015difference}, \textit{forward-only
methods} \citep{hinton2022forward}, \textit{layerwise local losses}
\citep{nokland2019training}, \textit{predictive methods}
\citep{rao1999predictive,whittington2017approximation,golkar2022constrained}, \textit{similarity
matching} \citep{qin2021contrastive}, \textit{feedback alignment}
\citep{lillicrap2016random}, and \textit{learned feedback pathways}
\citep{akrout2019deep}.

Most closely related to the present work are \textit{error broadcast} methods,
which send global output information directly to hidden layers. Direct
Feedback Alignment (DFA) projects the output loss derivative through fixed random feedback
matrices \citep{Nokland:16,bartunov2018assessing,han2019efficient,launay2019principled,launay2020direct,bordelon2022influence}. Although this removes weight transport, DFA does not derive the broadcast signal from a loss-dependent population-optimality principle or provide a mechanism for expanding the broadcast beyond the output dimension. Clark et al.'s Global Error Vector Broadcasting (GEVB) method \citep{Clark:21} uses a global error vector to modulate local plasticity in vectorized networks with nonnegative weights beyond the first layer, but does not establish a general score-orthogonality principle or conditions under which decorrelation is sufficient for population optimality. Most directly, Error Broadcast and Decorrelation (EBD) \citep{erdoganerror} gives broadcast learning a normative foundation by deriving layerwise decorrelation objectives from orthogonality of the MMSE estimation error, but this foundation is restricted to squared loss. EBD also reports an empirical decrease in output-error--activation correlations under cross-entropy backpropagation and discusses additional nonlinear feature constraints, without developing the general-loss framework here.\looseness=-1

\subsection{Contributions}
\begin{itemize}
\item \textbf{The output score as the broadcast signal.}
We identify the loss score
$\boldsymbol{\delta}=\nabla_{\mathbf{a}}\mathcal{L}(\mathbf{y},\mathbf{a})$,
taken with respect to an unconstrained output parameter $\mathbf{a}$, as the
principled quantity to broadcast. This choice recovers the EBD output error
under MSE and yields $\mathbf{p}-\mathbf{y}$ under softmax cross entropy.

\item \textbf{A characterization that makes decorrelation a training
criterion.}
We characterize when the population-optimal score has zero conditional mean
and is therefore orthogonal to every input-measurable feature. Conversely, when
the conditional risk is convex in the unconstrained output parameter, a
conditionally mean-zero score identifies a pointwise population minimizer;
score decorrelation is thus an optimality criterion rather than merely a
diagnostic.

\item \textbf{A broad loss family and a derived three-factor rule.}
The characterization covers Bregman divergences (including MSE), proper
scoring rules through unconstrained links (including cross entropy), and
exponential-family negative log-likelihoods. Across these losses, SBD derives
the neuromodulatory factor in the local three-factor learning rule from the
task loss itself rather than postulating it.

\item \textbf{A remedy for the broadcast rank bottleneck.}
We introduce score vector expansion (SVE), which stacks deterministically
modulated score blocks. SVE preserves conditional-mean-zero orthogonality at
the same population optimum while increasing the broadcast dimension and
allowing additional decorrelation directions.

\item \textbf{Empirical validation.}
On CIFAR-10 and Tiny ImageNet, SBD substantially improves over broadcast
baselines, and SVE provides further gains. Update-alignment diagnostics also
show positive cosine similarity with BP gradients throughout training.
\looseness=-1
\end{itemize}

\section{A review of error broadcast and decorrelation method}
\label{sec:review_ebd}
\label{sec:problem_statement}

We study a supervised learning problem using a multilayer perceptron with $L$ layers, including the output layer. Let $\mathbf{h}^{(k)} \in \mathbb{R}^{N^{(k)}}$ denote the activation vector at layer $k$, where $\mathbf{h}^{(0)} = \mathbf{x}$ is the input and $\mathbf{h}^{(L)}$ is the output. For each layer $k=1,\ldots,L$, the pre-activation and activation are written as
\begin{equation*}
\mathbf{u}^{(k)} = \mathbf{W}^{(k)}\mathbf{h}^{(k-1)} + \mathbf{b}^{(k)},
\qquad
\mathbf{h}^{(k)} = f^{(k)}\!\left(\mathbf{u}^{(k)}\right),
\end{equation*}
where $\mathbf{W}^{(k)}$ and $\mathbf{b}^{(k)}$ are the weights and biases of layer $k$, and $f^{(k)}(\cdot)$ is its activation function.

EBD \citep{erdoganerror} is a broadcast-based alternative to backpropagation
grounded in the stochastic orthogonality property of nonlinear MMSE
estimators. If
$\widehat{\mathbf{y}}^{\star}(X)=\mathbb{E}[\mathbf{Y}\mid X]$ is the optimal
MMSE estimator, then its estimation error
$\boldsymbol{\epsilon}^{\star}=\mathbf{Y}-\widehat{\mathbf{y}}^{\star}(X)$
is orthogonal to every measurable function $g(X)$ for which
$\mathbb{E}[\|g(X)\boldsymbol{\epsilon}^{\star T}\|_F]<\infty$. EBD chooses
these functions as hidden-layer activations or their nonlinear transforms and,
using the network convention of \citet{erdoganerror}, defines the current MSE
output error as
$\boldsymbol{\epsilon}:=\mathbf{h}^{(L)}-\mathbf{y}$. At layer $k$, EBD seeks to enforce\looseness=-1
\begin{equation*}
\mathbf{R}^{(k)}_{g\epsilon} = \mathbb{E}\bigl[g^{(k)}(\mathbf{h}^{(k)})\boldsymbol{\epsilon}^T\bigr] = \mathbf{0},
\end{equation*}
where $g^{(k)}(\mathbf{h}^{(k)})$ is a nonlinear function of the layer
activations whose outer product with $\boldsymbol{\epsilon}$ is integrable.
Because $\mathbf{h}^{(k)}$ is itself a deterministic function of the input,
this is a valid choice of input-measurable function; typically,
$g^{(k)}(\mathbf{h}^{(k)})=\mathbf{h}^{(k)}$.
The autoregressive estimate $\widehat{\mathbf{R}}_{g\epsilon}^{(k)}[m]$
for minibatch $m$ of size $B$ is defined in
Appendix~\ref{app:ebd-mmse-background}. EBD minimizes the layerwise-defined decorrelation objective
\begin{equation*}
\mathcal{J}^{(k)}_{\mathrm{EBD}}[m] = \frac{1}{2}\left\|\widehat{\mathbf{R}}_{g\epsilon}^{(k)}[m]\right\|_F^2.
\end{equation*}
Thus the stochastic orthogonality condition is turned into layer-specific loss
functions that decorrelate transformed layer activations from the broadcast
output error \citep{erdoganerror}.

Differentiating this loss gives two gradient components: one contains
derivatives of the layer activations with respect to the layer parameters,
whereas the other contains derivatives of the output error and requires
signals propagated from deeper layers. EBD retains the activation-derivative
component and defines the projected output error
\begin{equation*}
\mathbf{q}^{(k)}[n] = \widehat{\mathbf{R}}_{g\epsilon}^{(k)}[m]\,\boldsymbol{\epsilon}[n], \hspace{0.2in} n=mB+1, \ldots, (m+1)B,
\end{equation*}
where $n$ is the sample index. The resulting single-sample hidden-layer update, for $B=1$, is
\begin{equation*}
\Delta W_{ij}^{(k)}[n] = \zeta\, g_i^{\prime (k)}\!\left(h_i^{(k)}[n]\right) f^{\prime (k)}\!\left(u_i^{(k)}[n]\right) q_i^{(k)}[n] h_j^{(k-1)}[n].
\end{equation*}
Here \(\Delta W\) is a parameter increment and
\(\zeta=-\eta(1-\lambda)\) for gradient descent with learning rate \(\eta>0\).
This update has the three-factor structure in neuroscience: a presynaptic
activity, a postsynaptic factor formed by the derivatives of $g^{(k)}$ and
$f^{(k)}$, and the projected global error as the modulatory factor. EBD's
derivation, however, is tied to the MSE loss only. The proposed SBD framework
generalizes the derivation of the three-factor rule to a wider family of
losses, where MSE is a special instance.

\section{Extension to cross entropy loss}
\label{sec:cross_entropy_extension}

We now extend the EBD framework from MSE to multiclass classification under
cross entropy loss. Let $(X,\mathbf{Y})\sim\mathbb{P}_{X,\mathbf{Y}}$ be
jointly distributed input--label pairs taking values in
$\mathcal{X}\times\{\mathbf{e}_1,\ldots,\mathbf{e}_D\}$, where
$\mathbf{e}_d\in\mathbb{R}^D$ is the $d$-th canonical basis vector. As before,
the network produces hidden activations $\mathbf{h}^{(k)}$ for
$k=1,\ldots,L-1$, and the output layer produces logits
\begin{equation}
\label{eq:net_out}
\mathbf{a}=\mathbf{W}^{(L)}\mathbf{h}^{(L-1)}+\mathbf{b}^{(L)},
\end{equation}
with corresponding predicted probabilities $\mathbf{p}=\softmax(\mathbf{a})$, and cross entropy loss
\begin{equation*}
\mathcal{L}_{\mathrm{CE}}(\mathbf{y},\mathbf{a})=-\sum_{d=1}^{D} y_d\log p_d.
\end{equation*}
We refer to the gradient with respect to the logits $\mathbf{a}$ as the
\textit{output score}, which here takes the form
\begin{equation*}
\boldsymbol{\delta}_{\mathrm{CE}}=\nabla_{\mathbf{a}}\mathcal{L}_{\mathrm{CE}}(\mathbf{y},\mathbf{a})=\mathbf{p}-\mathbf{y}.
\end{equation*}
We will show that this cross entropy score plays the role of the output error
$\boldsymbol{\epsilon}$ in the MSE version of EBD, satisfying the analogous
orthogonality property at population optimality.

To obtain the cross entropy analog of MMSE orthogonality, define 
\begin{equation*}
\mathbf{q}(\mathbf{x})=\left[\begin{array}{ccc}\mathbb{P}(\mathbf{Y}=\mathbf{e}_1\mid X=\mathbf{x}),\ldots,\mathbb{P}(\mathbf{Y}=\mathbf{e}_D\mid X=\mathbf{x})\end{array}\right]^T,
\end{equation*}
as the true
conditional class-probability vector. For any predictor $\mathbf{p}(\mathbf{x})$, the 
cross entropy risk is
\begin{equation*}
\mathcal{R}_{\mathrm{CE}}(\mathbf{p})=\mathbb{E}\Bigl[-\sum_{d=1}^{D} Y_d\log p_d(X)\Bigr]=\mathbb{E}_{X}\bigl[H(\mathbf{q}(X),\mathbf{p}(X))\bigr],
\end{equation*}
where $H(\mathbf{q},\mathbf{p})$ is the cross entropy between the true
and predicted conditional distributions. Using the standard decomposition
$H(\mathbf{q},\mathbf{p})=H(\mathbf{q})+\mathrm{KL}(\mathbf{q}\|\mathbf{p})$,
the unique population minimizer over probability predictors is
$\mathbf{p}^{\star}(\mathbf{x})=\mathbf{q}(\mathbf{x})$ almost surely, which
yields the following identity replacing the MMSE estimation-error
property.\looseness=-1

\begin{theorem}[Conditional mean-zero property for cross entropy]
\label{thm:ce_conditional_mean_zero}
Let $\mathbf{p}^{\star}(X)$ denote the population minimizer of the
cross entropy risk. Then
\begin{equation*}
\mathbb{E}\bigl[\mathbf{p}^{\star}(X)-\mathbf{Y}\mid X\bigr]=\mathbf{0}.
\end{equation*}
\end{theorem}
\noindent\emph{The proof is deferred to Appendix~\ref{app:proof_ce_conditional_mean_zero}.}

The conditional mean-zero identity becomes unconditional orthogonality with
respect to arbitrary measurable functions of $X$ via the tower property of
conditional expectation.

\begin{lemma}[Tower-property orthogonality]
\label{lem:tower_orthogonality}
Let $\mathbf{U}$ be an integrable random vector and let $g(X)$ be any
measurable function of $X$ such that $\mathbf{U}\,g(X)^{T}$ is
integrable, i.e., $\mathbb{E}\bigl[\|\mathbf{U}\,g(X)^{T}\|\bigr]<\infty$.
If $\mathbb{E}[\mathbf{U}\mid X]=\mathbf{0}$, then
$\mathbb{E}\bigl[\mathbf{U}\,g(X)^{T}\bigr]=\mathbf{0}$.\looseness=-1
\end{lemma}

Applying the lemma with $\mathbf{U}=\mathbf{p}^{\star}(X)-\mathbf{Y}$ yields
the main cross entropy orthogonality theorem.

\begin{theorem}[Cross entropy orthogonality theorem]
\label{thm:ce_orthogonality}
Assume that the predictor class contains the true conditional distribution,
so that $\mathbf{p}^{\star}(X)=\mathbf{q}(X)$. Then for any measurable
function $g(X)$ such that $(\mathbf{p}^{\star}(X)-\mathbf{Y})\,g(X)^{T}$
is integrable, in particular, for any $g$ with $\mathbb{E}\|g(X)\|<\infty$,
since $\mathbf{p}^{\star}(X)-\mathbf{Y}$ is bounded,
\begin{equation*}
\mathbb{E}\bigl[(\mathbf{p}^{\star}(X)-\mathbf{Y})g(X)^T\bigr]=\mathbf{0}.
\end{equation*}
In particular, choosing $g(X)=g^{(k)}(\mathbf{h}^{(k)}(X))$ gives
\begin{equation*}
\mathbb{E}\bigl[(\mathbf{p}^{\star}(X)-\mathbf{Y})g^{(k)}(\mathbf{h}^{(k)}(X))^T\bigr]=\mathbf{0}.
\end{equation*}
\end{theorem}
\noindent\emph{The proof is deferred to Appendix~\ref{app:proof_ce_orthogonality}.}

Figure~\ref{fig:cnn-cifar10-correlations} illustrates this empirically:
although the theorem is a population statement, the measured score--hidden
correlations in a CNN trained with cross entropy steadily decline during
backpropagation, indicating that optimization drives the network toward the
predicted orthogonality regime.

Using this cross entropy counterpart of orthogonality, we define
the layerwise decorrelation condition
\begin{equation*}
\mathbf{R}_{\mathrm{CE}}^{(k)}
=\mathbb{E}\bigl[g^{(k)}(\mathbf{h}^{(k)})(\mathbf{p}-\mathbf{y})^T\bigr]
\in\mathbb{R}^{N^{(k)}\times D},
\qquad
\mathbf{R}_{\mathrm{CE}}^{(k)}=\mathbf{0},
\end{equation*}
with objective
$\mathcal{J}_{\mathrm{CE}}^{(k)}=\tfrac{1}{2}\|\mathbf{R}_{\mathrm{CE}}^{(k)}\|_F^2$.
The corresponding local three-factor update is obtained by substituting
$\boldsymbol{\delta}=\mathbf{p}-\mathbf{y}$ into the general SBD rule in
Section~\ref{sec:general_losses_sbd}; the cross-entropy implementation is
detailed in Appendix~\ref{app:ce_sbd_algorithm}.\looseness=-1

\paragraph{Sufficiency under a dense-feature assumption.}
Theorem~\ref{thm:ce_orthogonality} shows that score--feature
orthogonality is \emph{necessary} at the population cross entropy optimum.
Sufficiency requires a sufficiently rich feature family, as discussed for
general losses in Section~\ref{sec:general_losses_sbd}.
The finite-width algorithm should be viewed as an approximation to this
idealized case. The detailed cross-entropy argument is  in
Appendix~\ref{app:ce_dense_sufficiency}.\looseness=-1

\section{Extension to more general losses: score broadcast and decorrelation}
\label{sec:general_losses_sbd}

The cross entropy derivation reveals the central idea behind the generalized framework: what matters is not the MSE output error or the particular form of the cross entropy score, but the score of the loss with respect to the network output. For a general differentiable loss $\mathcal{L}(\mathbf{y},\mathbf{a})$ with $\mathbf{a}\in\mathbb{R}^{D_{\mathrm{out}}}$, we define the score by
\begin{equation*}
\boldsymbol{\delta}(\mathbf{y},\mathbf{a}) := \nabla_{\mathbf{a}}\mathcal{L}(\mathbf{y},\mathbf{a}).
\end{equation*}
This provides a universal broadcast signal. In particular,
$\boldsymbol{\delta}_{\mathrm{MSE}}=\mathbf{a}-\mathbf{y}$,
and, 
$\boldsymbol{\delta}_{\mathrm{CE}}=\mathbf{p}-\mathbf{y}$, so both MSE and cross entropy fit naturally into the same score-based formulation.
For arbitrary differentiable losses, the first universal orthogonality statement arises from stationarity of the last-layer weights. Let $\mathbf{a}$  denote network output as in Eq.~(\ref{eq:net_out}), define the population risk
$\mathcal{R}=\mathbb{E}\bigl[\mathcal{L}(\mathbf{Y},\mathbf{a}(\mathbf{X}))\bigr]$.
Then differentiation with respect to the last-layer weight matrix gives
\begin{equation*}
\nabla_{\mathbf{W}^{(L)}}\mathcal{R}=\mathbb{E}\bigl[\boldsymbol{\delta}\,\mathbf{h}^{(L-1)T}\bigr].
\end{equation*}
Consequently, every stationary point of the population risk satisfies the following proposition.

\begin{proposition}[Stationary score-feature orthogonality]
\label{prop:stationary_score_feature_orthogonality}
Assume that $\mathcal{L}(\mathbf{y},\mathbf{a})$ is differentiable in $\mathbf{a}$ and that gradient and expectation may be interchanged in the expression for $\mathcal{R}$. If $\mathbf{W}^{(L)}$ is stationary while all earlier parameters are fixed, then
$\mathbb{E}\bigl[\boldsymbol{\delta}\,\mathbf{h}^{(L-1)T}\bigr]=\mathbf{0}.$
\end{proposition}
\noindent\emph{The proof is deferred to Appendix~\ref{app:proof_stationary_score_feature_orthogonality}.}

The proposition only covers the last hidden layer. Orthogonality with arbitrary measurable input functions requires the conditional mean-zero score property of the following theorem:\looseness=-1

\begin{theorem}[General score orthogonality theorem]
\label{thm:general_score_orthogonality}
Let $\mathbf{a}^{\star}(X)$ be an optimal predictor at the population level for
a differentiable loss $\mathcal{L}(\mathbf{y},\mathbf{a})$, and let
$\boldsymbol{\delta}^{\star}=\nabla_{\mathbf{a}}\mathcal{L}(\mathbf{Y},\mathbf{a}^{\star}(X))$
be the corresponding optimal score. Assume that the optimal score satisfies
the conditional mean-zero property
$\mathbb{E}[\boldsymbol{\delta}^{\star}\mid X]=\mathbf{0}$. Then for every measurable function $g(X)$ such that
$\boldsymbol{\delta}^{\star}\,g(X)^{T}$ is integrable (equivalently,
$\mathbb{E}\bigl[\|\boldsymbol{\delta}^{\star}\,g(X)^{T}\|\bigr]<\infty$),
\begin{equation*}
\mathbb{E}\bigl[\boldsymbol{\delta}^{\star}g(X)^T\bigr]=\mathbf{0}.
\end{equation*}
In particular, taking $g(X)=g^{(k)}(\mathbf{h}^{(k)}(X))$ yields layerwise
score orthogonality.
\end{theorem}
\noindent\emph{The proof is deferred to
Appendix~\ref{app:proof_general_score_orthogonality}.}

\noindent\emph{Remark.} The conditional mean-zero hypothesis assumed in this
theorem is broadly satisfied: in
Section~\ref{sec:conditional_mean_zero_score} and
Appendix~\ref{app:sec6_appendix} we show that it holds whenever the
conditional risk has an interior minimizer in an open parameter domain. This
characterization covers the standard differentiable-loss families used in
supervised learning, including Bregman divergences, proper scoring rules
through unconstrained links, and exponential-family negative log-likelihoods.
MSE and cross-entropy are canonical instances of the first and second
families, respectively.\looseness=-1

\paragraph{Sufficiency under a dense-feature assumption.}
The dense-feature argument depends on the loss only through its score. Let $\mathbf{a}(X)$ denote a \emph{candidate} predictor (not assumed optimal) and let
$\boldsymbol{\delta} := \nabla_{\mathbf{a}}\mathcal{L}(\mathbf{Y},\mathbf{a}(X))$
be its score, assumed square-integrable. By the tower property, decorrelating
$\boldsymbol{\delta}$ against a feature family whose linear span is dense in $L^2(P_X)$ forces
$\mathbb{E}[\boldsymbol{\delta}\mid X]=\mathbf{0}$ almost surely. Under the regularity assumptions of
Section~\ref{sec:conditional_mean_zero_score}, this identifies a pointwise minimizer whenever the conditional risk and its domain are convex. The cross-entropy case of Section~\ref{sec:cross_entropy_extension} and the MSE
case of EBD~\citep[Appendix B.2]{erdoganerror} are two instances; under SBD
the same argument applies uniformly across the differentiable-loss families
covered by the conditional-risk characterization.\looseness=-1

Let $d_k$ denote the dimension of
$g^{(k)}(\mathbf{h}^{(k)})$. The resulting generalized feature-by-score
correlation and score-broadcast objective are
\begin{equation*}
\mathbf{R}_{\mathrm{Score}}^{(k)}
:=\mathbb{E}\bigl[g^{(k)}(\mathbf{h}^{(k)})\boldsymbol{\delta}^T\bigr]
\in\mathbb{R}^{d_k\times D_{\mathrm{out}}},
\qquad
\mathcal{J}_{\mathrm{Score}}^{(k)}
=\frac{1}{2}\left\|\mathbf{R}_{\mathrm{Score}}^{(k)}\right\|_F^2,
\end{equation*}
with minibatch estimator
\begin{equation*}
\widehat{\mathbf{R}}_{\mathrm{Score}}^{(k)}[m]=\lambda \widehat{\mathbf{R}}_{\mathrm{Score}}^{(k)}[m-1]+(1-\lambda)B^{-1}\,\mathbf{G}^{(k)}[m]\mathbf{\Delta}[m]^T,
\end{equation*}
where $\mathbf{G}^{(k)}[m]\in\mathbb{R}^{d_k\times B}$ collects the hidden
features and $\mathbf{\Delta}[m]\in\mathbb{R}^{D_{\mathrm{out}}\times B}$
collects the score vectors
$\boldsymbol{\delta}[mB+1],\ldots,\boldsymbol{\delta}[mB+B]$.

As shown in Appendix~\ref{app:ce_sbd_derivation}, differentiating
$\mathcal{J}_{\mathrm{Score}}^{(k)}[m]$ gives one gradient component containing
derivatives of the layer activations and another containing derivatives of the
output score. The latter would require propagating signals through deeper
layers, so the practical broadcast rule retains the activation-derivative
component.
The layerwise broadcast is
\begin{equation*}
\mathbf{q}_{\mathrm{Score}}^{(k)}[n]=\widehat{\mathbf{R}}_{\mathrm{Score}}^{(k)}[m]\,\boldsymbol{\delta}[n],  \hspace{0.2in} n=mB+1, \ldots, (m+1)B,
\end{equation*}
and for $B=1$ and elementwise $g^{(k)}$, the weight update keeps the generic local form
\begin{equation*}
\Delta W_{ij}^{(k)}[n]=\zeta\, g_i^{\prime(k)}\!\left(h_i^{(k)}[n]\right) f^{\prime(k)}\!\left(u_i^{(k)}[n]\right) q_{\mathrm{Score},i}^{(k)}[n] h_j^{(k-1)}[n].
\end{equation*}
The update rule above provides a derivation of the three-factor learning
rule, a long-standing model of synaptic plasticity in computational
neuroscience \citep{fremaux2016neuromodulated, kusmierz2017learning}, under a
broad family of differentiable losses. In the neuroscience literature the
neuromodulatory factor is typically postulated; under SBD it is derived
\emph{from the loss itself} as the broadcast loss score. Because the
derivation depends on the loss only through its score, the same three-factor
update covers every loss to which the conditional-mean-zero characterization
applies, including Bregman divergences, exponential-family negative
log-likelihoods, and proper scoring rules through unconstrained links. SBD is
thus a unified derivation of the three-factor rule across the standard
differentiable losses of supervised learning, with MSE-based EBD recovered as
one instance. The CNN experiments additionally use activation-covariance
regularization \citep{erdoganerror,ozsoy2022self,bozkurt2023correlative}
and layer-specific regularizers described in Appendices~\ref{app:cemsen-sbd}
and~\ref{app:tinet}.

\section{When does an unconstrained loss parametrization yield a conditionally mean-zero score?}
\label{sec:conditional_mean_zero_score}

Theorem~\ref{thm:general_score_orthogonality} shows that full layerwise
orthogonality follows once the population-optimal score satisfies
$\mathbb{E}[\boldsymbol{\delta}^{\star}\mid X]=\mathbf{0}$. We now ask when
this hypothesis holds. The key step is to forget the network parameterization
and study the loss as a function of the prediction variable alone.

For each input value $x$, define the conditional risk
\begin{equation*}
C_x(\mathbf{a}) = \mathbb{E}\bigl[\mathcal{L}(\mathbf{Y},\mathbf{a}) \mid X = x\bigr],
\end{equation*}
where $\mathbf{a}$ is the output parameter with respect to which the SBD score
$\boldsymbol{\delta} = \nabla_{\mathbf{a}}\mathcal{L}(\mathbf{Y},\mathbf{a})$ is taken,
and assume that $\mathbf{a}$ ranges over an open set
$\mathcal{A}\subseteq\mathbb{R}^m$. Losses with constrained natural prediction
variables (e.g., probability losses on $\mathbf{p}\in\Delta_D$) are brought into
this setting by composing with an unconstrained link
$\mathbf{p}=\psi(\mathbf{a})$ and writing
$\mathcal{L}(\mathbf{Y},\mathbf{a})=\ell(\mathbf{Y},\psi(\mathbf{a}))$. The
following theorem then characterizes when the conditional-mean-zero score
property holds.

\begin{theorem}[Conditional-risk characterization, informal]
\label{thm:cond_risk_informal}
Suppose $C_x$ has an interior minimizer $\mathbf{a}^\star(x)\in\mathcal{A}$ for
almost every $x$, with mild regularity allowing differentiation under the
conditional expectation. Then the optimal score satisfies
\begin{equation*}
\mathbb{E}\bigl[\nabla_{\mathbf{a}}\mathcal{L}(\mathbf{Y},\mathbf{a}^\star(X))\mid X\bigr]
= \mathbf{0}.
\end{equation*}
If, in addition, $\mathcal A$ and $C_x$ are convex for almost every $x$, the converse also holds:
any predictor whose score has zero conditional mean almost surely minimizes
$C_x$ pointwise.
\end{theorem}
The full statement of the theorem and its proof are deferred to
Appendix~\ref{app:sec6_appendix} (Theorem~\ref{thm:conditional_risk_score}).
Appendix~\ref{app:sec6_appendix} applies the theorem to three canonical
loss families: Bregman divergences (with MSE as the special
case $\phi(\mathbf{u})=\tfrac12\|\mathbf{u}\|_2^2$), proper scoring rules
through unconstrained links (with softmax cross-entropy as the canonical
instance), and exponential-family negative log-likelihoods (where the
conditional-mean-zero property reduces to moment matching). The main takeaway
is that SBD subsumes the standard differentiable losses of supervised learning
under a single principle, with MSE-based EBD recovered as one instance.\looseness=-1

\section{Score vector expansion}
\label{sec:score_expansion}

The output score $\boldsymbol{\delta}(\mathbf{Y},\mathbf{a})$ has dimension $D_{\mathrm{out}}$, which is typically much less than the hidden layer width $d_k$. We use $D_{\varepsilon}$ for the dimension of the signal broadcast to hidden layers: $D_{\varepsilon}=D_{\mathrm{out}}$ for plain SBD and $D_{\varepsilon}=M D_{\mathrm{out}}$ when $M$ score blocks are stacked. Hence the unexpanded layerwise correlation matrix has rank at most $D_{\mathrm{out}}$, and this bottleneck limits the number of independent decorrelation constraints imposed by SBD on the hidden representation. This limitation can be relaxed by enlarging the broadcast signal with deterministic modulations that preserve the conditional mean-zero property. The modulator need only be a function of $X$. Given any modulator $\boldsymbol{\phi}:\mathcal{X}\to\mathbb{R}^{D_{\mathrm{out}}}$, define the modulated score at the population optimum by \looseness=-1
\begin{equation*}
\boldsymbol{\eta}^{\boldsymbol{\phi},\star}(X)
\;:=\;\boldsymbol{\phi}(X)\odot
\boldsymbol{\delta}\bigl(\mathbf{Y},\mathbf{a}^{\star}(X)\bigr).
\end{equation*}
When the modulated score is integrable, it also has zero conditional mean because $\boldsymbol{\phi}(X)$ pulls out of the conditional expectation. We stack $M$ such modulated scores into the expanded broadcast vector
\begin{equation*}
\tilde{\boldsymbol{\delta}}(X)
\;:=\;\bigl[\boldsymbol{\phi}_{1}(X)\odot\boldsymbol{\delta}(\mathbf{Y},\mathbf{a}(X));\;\ldots;\;\boldsymbol{\phi}_{M}(X)\odot\boldsymbol{\delta}(\mathbf{Y},\mathbf{a}(X))\bigr]
\in\mathbb{R}^{D_{\varepsilon}},
\qquad D_{\varepsilon}=M D_{\mathrm{out}}.
\end{equation*}
At the population optimum, this preserves layerwise orthogonality whenever the score--feature products are integrable, allowing additional decorrelation directions at linear broadcast cost (Appendix~\ref{app:ce_sbd_complexity}). Output-dependent choices $\boldsymbol{\phi}(X)=\psi(\mathbf{a}(X))$ are a special case.\looseness=-1

The general construction only requires the modulators
$\boldsymbol{\phi}_{\ell}$ to be deterministic functions of $X$, enabling substantial freedom in their design; a natural and
computationally inexpensive choice is to build them from quantities
already produced by the forward pass, such as the predictive
distribution $\mathbf{p}(X)$, which is itself a nonlinear deterministic
function of $X$. Following this guideline, for the CNN experiments in Section~\ref{sec:numerical_experiments} we adopt the $3D$-dimensional
instance
$\tilde{\boldsymbol{\delta}}
=\bigl[\boldsymbol{\delta};\;\mathbf{p}(X)\odot\boldsymbol{\delta};\;\mathrm{roll}_{k}(\mathbf{p}(X))\odot\boldsymbol{\delta}\bigr]\in\mathbb{R}^{3D}$,
which augments the raw score with a confidence-weighted score block and a cyclic permutation of the probability coordinates. We use $k=5$ in the reported experiments. Because class indices are arbitrary, $k$ has no semantic interpretation. These blocks are one admissible choice among many; Appendix~\ref{app:score-expansion-ablation} reports an ablation over alternatives, and more principled, data-adaptive constructions are left as future work. The numerical experiments of Section~\ref{sec:numerical_experiments}, most clearly the $3.1$-point gain on Tiny ImageNet, show that this selected instance improves over unexpanded SBD. Full derivations and the expanded algorithm are deferred to Appendix~\ref{app:score_expansion_appendix}.\looseness=-1

\section{Numerical experiments}
\label{sec:numerical_experiments}

Because our main contribution is conceptual and theoretical, we designed our experiments as tests of the score broadcast principle within broadcast
learning. BP with cross entropy is included as a reference optimizer; the main
comparisons are to MSE based EBD and DFA based broadcast baselines. For CIFAR-10
\citep{krizhevsky2009cifar}, we use the CNN architecture and training setup of
\citet{erdoganerror}, replacing MSE by cross entropy where appropriate.
Table~\ref{tab:cifar10-cnn-comparison} shows that among broadcast methods,
SBD with score expansion (SBD Exp) performs best: SBD improves on EBD from
$66.4\%$ to $69.2\%$, and score expansion further improves
performance to $70.0\%$. A $4\times$ width run raises BP to $83.1\%$ and
SBD Exp to $74.5\%$. Training settings are given in Appendix~\ref{app:cemsen}.
Appendix~\ref{app:cemsen-sbd-bp-cosine} also reports positive layerwise cosine
similarity between raw local SBD directions and exact BP gradients on the
same minibatches.\looseness=-1
\begin{table}[h]
\caption{CIFAR-10 test accuracy (\%). In-house CE results report mean $\pm$ standard deviation over five runs and share the CNN architecture of Appendix~\ref{app:cemsen}. MSE results are from \citet{erdoganerror}; NN-GEVB is a literature reference from \citet{Clark:21} using a different architecture.}
\label{tab:cifar10-cnn-comparison}
\centering
{\setlength{\tabcolsep}{2.9pt}
\begin{tabular}{cccccccc}
\toprule
BP & DFA & EBD & BP & DFA & NN-GEVB & SBD & SBD Exp \\
(MSE) & (MSE) & (MSE) & (CE) & (CE) & (CE) & (CE, Ours) & (CE, Ours) \\
\midrule
$75.2 \pm 0.3$ & $58.4 \pm 1.6$ & $66.4 \pm 0.4$ & $\mathbf{78.5 \pm 0.5}$ & $65.3 \pm 1.2$ & $66.26$ & $69.2 \pm 0.7$ & $\underline{70.0 \pm 0.7}$ \\
\bottomrule
\end{tabular}}
\end{table}

\begin{table}[t]
\caption{Tiny ImageNet accuracy (\%; mean $\pm$ standard deviation over five runs), evaluated on the official validation set, used as the test set here. BP, DFA, and SBD use the CNN architecture of Appendix~\ref{app:tinet}. NN-GEVB is an in-house run using the code of \citet{Clark:21}.}
\label{tab:tiny-imagenet-cnn-comparison}
\centering
{\setlength{\tabcolsep}{4pt}
\begin{tabular}{ccccccc}
\toprule
BP & EBD & BP & DFA & NN-GEVB & SBD & SBD Exp\\
(MSE) & (MSE) & (CE) & (CE) & (CE, in-house) & (CE, Ours) & (CE, Ours)\\
\midrule
$36.0 \pm 0.7$ & $18.5 \pm 0.7$ & $\mathbf{39.9 \pm 0.3}$ & $17.5 \pm 0.4$ & $18.8 \pm 0.5$ & $28.3 \pm 0.4$ & $\underline{31.4 \pm 0.4}$ \\
\bottomrule
\end{tabular}}
\end{table}

For Tiny ImageNet \citep{le2015tiny}, we use the $200$ class, $6$ layer CNN
setup described in Appendix~\ref{app:tinet}. Table~\ref{tab:tiny-imagenet-cnn-comparison}
shows the same qualitative ordering on a harder benchmark. BP remains the
strongest reference optimizer, while SBD improves substantially over DFA
and EBD (MSE). Score expansion again gives the best broadcast result.
Across both benchmarks, SBD improves on MSE-based EBD,
with further gains from the expanded variants. Beyond cross-entropy,
Appendix~\ref{app:poisson_demo} reports Poisson regression with a
prediction-bin calibration diagnostic motivated by the conditional-mean-zero
property. Appendices~\ref{app:cemsen}--\ref{app:resnet-experiments} provide
the experimental details.\looseness=-1

We further evaluate SBD in two complementary settings: deep residual
architectures and deep recurrent networks with biologically plausible
plasticity. On a wide, 20-layer residual network, the extended SBD variant
described in Appendix~\ref{app:resnet-experiments} achieves
$(75.90\pm0.60)\%$ test accuracy. A BP reference on the same wide
architecture achieves $(94.50\pm0.09)\%$. This experiment
extends the empirical scope of SBD to deep residual architectures, while the
remaining gap to BP highlights an important direction for improvement.
We also consider Correlative Information Maximization (CorInfoMax)
networks, introduced by \citet{bozkurt2023correlative} and subsequently
trained with EBD by \citet{erdoganerror}. These networks derive their
recurrent representation dynamics from a correlation-based
information-maximization objective, combining interlayer prediction with
adaptive lateral interactions. Here, CE-based score broadcasts supply
task supervision through local plasticity while retaining CorInfoMax's
recurrent inference. Training is fully online
(batch size $1$), using local Hebbian, anti-Hebbian, and three-factor
updates. The LC architecture additionally uses spatially local, unshared
weights, avoiding convolutional weight tying. On CIFAR-10, the 10-layer
FC network achieves $(60.24\pm0.46)\%$, compared with the previously
reported $(54.89\pm0.58)\%$ for EBD, while the LC network with nine hidden stages reaches
$(71.98\pm0.49)\%$
(Appendix~\ref{app:corinfomax-experiments}). Together, these results
demonstrate learning at depth both in residual architectures and under
the combined constraints of recurrent lateral interactions,
single-sample training, and local synaptic plasticity.\looseness=-1

\section{Conclusions and limitations}
\label{sec:conclusion}

We introduced \textit{Score Broadcast and Decorrelation} (SBD), a loss- and
parameterization-aware framework for broadcast credit assignment. Under
the stated conditions, the population-optimal output score is
conditionally mean zero and hence orthogonal to input-measurable
features, including hidden activations. When the conditional risk is
convex in the chosen output coordinates, this conditional-mean-zero
property also characterizes optimal predictions. SBD thereby provides a
principled answer to what should be broadcast and a theoretical
foundation for local three-factor learning: the modulatory signal is
derived from the task loss rather than postulated, with MSE-based EBD
recovered as a special case. Score vector expansion provides a second
advance: deterministic input-dependent modulation preserves
population-optimal orthogonality while allowing independent
decorrelation directions beyond the raw output dimension, thereby
relaxing an intrinsic rank bottleneck of unexpanded broadcast.
Experiments on CIFAR-10 and Tiny ImageNet show improvements over the
evaluated broadcast baselines and further gains from expansion.
The extended ResNet experiment and fully online CorInfoMax experiments
additionally demonstrate learning in deep residual architectures and biologically
plausible recurrent networks with lateral connections. SBD thus offers
both a practical broadcast-learning algorithm and a theoretical lens on
biologically observed neuromodulatory plasticity.\looseness=-1

\paragraph{Limitations.}
The strongest guarantees concern population optima under the stated
assumptions; they do not establish convergence of the practical local
no-propagation updates in finite networks trained on finite data.
Although the depth experiments broaden the demonstrated applicability
of SBD, substantial performance gaps to BP remain, and larger-scale
evaluation is still needed. Improving local learning dynamics and
developing data-adaptive score expansions are promising directions
toward narrowing these gaps and enabling more effective and scalable
broadcast learning. Relating the framework
to measurements of neuromodulatory plasticity is a complementary
direction.\looseness=-1

\newpage
\subsection*{AI use statement}
We utilized generative AI tools for grammatical editing, text refinement, LaTeX and reference formatting, and literature discovery. We also used these tools to generate and edit code implementations of our methods and numerical experiments, check consistency between the code and manuscript, and obtain feedback on experimental methodology and configurations.

We used LLM-based tools to check and refine the formulation of mathematical claims derived by the authors and to assist with writing and revising their proofs. The theoretical framework, derivation of the SBD learning algorithm, and essential proof arguments were developed by the authors. We reviewed the suggested mathematical refinements and corrections before incorporating them.

We have not used AI to propose or refine research hypotheses, provide critical ingredients for the proofs, assist with translation, or clean and reformat existing datasets. Qualitative and thematic data analysis is not applicable to this work.

We have reviewed all AI-assisted work including text and code. We also utilized AI-tools for literature discovery, which is followed by manual investigation and reporting of these references. We take responsibility for the final content of this work, including text, claims or code produced with the aid of generative AI.
\subsection*{Reproducibility Statement}

Experimental settings are described in the appendices. The supplementary
repository contains the implemented SBD variants, selected baseline trainers,
and the Poisson notebook; its README identifies the supplied experiments.

\clearpage
\appendix

\clearpage
\appendix
\startcontents[appendices]

\section*{Appendix contents}
\begingroup
  \setlength{\parskip}{2pt}
  \renewcommand{\baselinestretch}{1}\selectfont
  \printcontents[appendices]{}{1}{\setcounter{tocdepth}{2}}
\endgroup

\clearpage

\section{Deferred proofs for main-text results}
\label{app:deferred_main_proofs}

\subsection{Proof of theorem~\ref{thm:ce_conditional_mean_zero}}
\label{app:proof_ce_conditional_mean_zero}
Since cross entropy is strictly proper, its population minimizer is the true conditional label distribution, i.e., $\mathbf{p}^{\star}(X)=\mathbf{q}(X)=\mathbb{E}[\mathbf{Y}\mid X]$. Therefore,
\begin{equation*}
\mathbb{E}\bigl[\mathbf{p}^{\star}(X)-\mathbf{Y}\mid X\bigr]=\mathbf{p}^{\star}(X)-\mathbb{E}[\mathbf{Y}\mid X]=\mathbf{q}(X)-\mathbf{q}(X)=\mathbf{0}.
\end{equation*}
Hence the output score of the optimal cross entropy predictor is conditionally mean zero. \hfill$\square$

This probability-space optimum is attained by finite softmax logits exactly
when every conditional class probability is positive.

\subsection{Proof of lemma~\ref{lem:tower_orthogonality}}
\label{app:proof_tower_orthogonality}
The integrability assumption $\mathbb{E}\bigl[\|\mathbf{U}\,g(X)^{T}\|\bigr]<\infty$
ensures that the (vector-valued) tower property applies. Because $g(X)$ is measurable with respect to $X$, it can be pulled out of the conditional expectation. Hence
\begin{equation*}
\mathbb{E}\bigl[\mathbf{U}g(X)^T\bigr]=\mathbb{E}\Bigl[\mathbb{E}\bigl[\mathbf{U}g(X)^T\mid X\bigr]\Bigr]=\mathbb{E}\Bigl[\mathbb{E}[\mathbf{U}\mid X]g(X)^T\Bigr].
\end{equation*}
This is exactly the mechanism by which conditional mean-zero leads to orthogonality with arbitrary functions of the input. \hfill$\square$

\subsection{Proof of theorem~\ref{thm:ce_orthogonality}}
\label{app:proof_ce_orthogonality}
By Theorem~\ref{thm:ce_conditional_mean_zero}, $\mathbf{U}=\mathbf{p}^{\star}(X)-\mathbf{Y}$ satisfies $\mathbb{E}[\mathbf{U}\mid X]=\mathbf{0}$.
Since $\mathbf{p}^{\star}(X)\in\Delta_{D}$ and $\mathbf{Y}\in\{0,1\}^{D}$,
the score vector $\mathbf{U}$ is uniformly bounded, so any measurable $g$ with
$\mathbb{E}\|g(X)\|<\infty$ automatically yields an integrable product
$\mathbf{U}\,g(X)^{T}$. Applying Lemma~\ref{lem:tower_orthogonality} to this choice of $\mathbf{U}$ immediately gives
\begin{equation*}
\mathbb{E}\bigl[(\mathbf{p}^{\star}(X)-\mathbf{Y})g(X)^T\bigr]=\mathbf{0}
\end{equation*}
for every such $g(X)$. Substituting $g(X)=g^{(k)}(\mathbf{h}^{(k)}(X))$ yields the layerwise version used by the decorrelation framework. \hfill$\square$

\subsection{Cross-entropy sufficiency under a dense-feature assumption}
\label{app:ce_dense_sufficiency}
Theorem~\ref{thm:ce_orthogonality} shows that score--feature
orthogonality is \emph{necessary} at the population cross entropy optimum.
The converse, that exact orthogonality identifies the optimum,
requires the feature family used in the decorrelation
to be sufficiently rich. Let $\mathbf{p}(X)$ denote a \emph{candidate}
predictor (not assumed to be the population minimizer), and define
$m_{\mathbf{p}}(X) := \mathbb{E}[\mathbf{p}(X)-\mathbf{Y}\mid X] = \mathbf{p}(X)-\mathbf{q}(X)$,
the gap between the candidate predictor and the true conditional class probabilities.
The decorrelation condition $\mathbb{E}[(\mathbf{p}(X)-\mathbf{Y})\mathbf{z}(X)^T]=0$ is, by the tower property, exactly the orthogonality of $m_{\mathbf{p}}$ against the feature $\mathbf{z}$ in $L^2(P_X)$. If the hidden-layer features (e.g., first-layer activations in the wide-network limit) have linear span dense in $L^2(P_X)$, then enforcing the decorrelation condition against \emph{all} such features forces $m_{\mathbf{p}}$ to be orthogonal to a dense subspace. Since $m_{\mathbf{p}}\in L^2(P_X)$, this forces $m_{\mathbf{p}}(X)=0$ almost surely, which gives $\mathbf{p}(X)=\mathbf{q}(X)=\mathbf{p}^{\star}(X)$, i.e., the candidate predictor coincides with the population minimizer of Theorem~\ref{thm:ce_conditional_mean_zero}.
The same dense-feature reasoning was used in EBD~\citep[Appendix B.2]{erdoganerror} for MSE; under SBD it operates identically with the cross entropy score. The finite-width algorithm should
be viewed as an approximation to this idealized condition.
\looseness=-1

\subsection{Proof of proposition~\ref{prop:stationary_score_feature_orthogonality}}
\label{app:proof_stationary_score_feature_orthogonality}
\noindent\emph{Proof.}
At a stationary point, by definition, the gradient of the population risk with respect to the last-layer weights must vanish. Since
\begin{equation*}
\nabla_{\mathbf{W}^{(L)}}\mathcal{R}=\mathbb{E}\bigl[\boldsymbol{\delta}\,\mathbf{h}^{(L-1)T}\bigr],
\end{equation*}
setting the gradient to zero gives the claimed orthogonality relation immediately. \hfill$\square$

\subsection{Proof of theorem~\ref{thm:general_score_orthogonality}}
\label{app:proof_general_score_orthogonality}
The proof is the same tower-property argument used in the cross entropy case. Since $g(X)$ is measurable with respect to $X$,
\begin{equation*}
\mathbb{E}\bigl[\boldsymbol{\delta}^{\star}g(X)^T\bigr]=\mathbb{E}\Bigl[\mathbb{E}\bigl[\boldsymbol{\delta}^{\star}g(X)^T\mid X\bigr]\Bigr]=\mathbb{E}\Bigl[\mathbb{E}[\boldsymbol{\delta}^{\star}\mid X]g(X)^T\Bigr].
\end{equation*}
Using the hypothesis $\mathbb{E}[\boldsymbol{\delta}^{\star}\mid X]=\mathbf{0}$ gives
\begin{equation*}
\mathbb{E}\bigl[\boldsymbol{\delta}^{\star}g(X)^T\bigr]=\mathbb{E}[\mathbf{0}\,g(X)^T]=\mathbf{0}.
\end{equation*}
Therefore the optimal score is orthogonal to every measurable function of the input. \hfill$\square$

\section{When does a loss yield a conditionally mean-zero score?}
\label{app:sec6_appendix}

This appendix section provides the full conditional-risk characterization summarized in
Section~\ref{sec:conditional_mean_zero_score}.

Theorem~\ref{thm:general_score_orthogonality} shows that the generalized SBD framework has its full layerwise orthogonality property whenever the population-optimal score satisfies $\mathbb{E}[\boldsymbol{\delta}^{\star}\mid X]=\mathbf{0}$. The key question is therefore: when does a loss produce this conditional mean-zero score? In order to perform this characterization, we look at the loss structure.

For a fixed input value $x$, define the conditional risk
\begin{equation}
\label{eq:cond_risk}
C_x(\mathbf{a})=\mathbb{E}\bigl[\mathcal{L}(\mathbf{Y},\mathbf{a})\mid X=x\bigr].
\end{equation}
Here $\mathbf{a}$ denotes the output parameter with respect to which the SBD score is taken. Throughout the theorem below, $\mathbf{a}$ lies in an open subset \(\mathcal A\subseteq\mathbb{R}^m\). For MSE loss, $\mathbf{a}$ is directly the Euclidean estimate. For logit-based classification or natural-parameter likelihood models, $\mathbf{a}$ is the corresponding unconstrained output parameter. If a loss is naturally
defined on a constrained prediction variable, such as a probability vector
\(\mathbf{p} \in\Delta_D\), we treat it here by composing it with an unconstrained link \(\mathbf{p}=\psi(\mathbf{a})\) and applying the theorem to
\begin{equation*}
\mathcal{L}(\mathbf{Y},\mathbf{a})=\ell(\mathbf{Y},\psi(\mathbf{a})).
\end{equation*}
The direct constrained formulation is not needed for the SBD results in this
paper and is left outside the main theorem.

We note that the population risk is related to the conditional risk in Eq.~(\ref{eq:cond_risk}) simply through

\begin{equation*}
  \mathcal{R}(\mathbf{a}) = \mathbb{E}\bigl[\mathcal{L}(\mathbf{Y},\mathbf{a}(X))\bigr]
  = \mathbb{E}_X\bigl[C_X(\mathbf{a}(X))\bigr].
\end{equation*}

A measurable predictor that minimizes $C_x$ for almost every $x$ therefore
minimizes population risk, provided the relevant risks are finite.

\begin{theorem}[Conditional-risk characterization in an unconstrained parameterization]
\label{thm:conditional_risk_score}
Let $\mathcal{A}\subseteq\mathbb{R}^m$ be open, and let $\mathcal{L}(\mathbf{y},\mathbf{a})$ be differentiable in $\mathbf{a}\in\mathcal{A}$. For almost every $x$, assume that the conditional risk $C_x(\mathbf{a})$ has an interior minimizer $\mathbf{a}^{\star}(x)\in\mathcal{A}$, and that differentiation may be interchanged with conditional expectation. Then the optimal score
\begin{equation*}
\boldsymbol{\delta}^{\star}=\nabla_{\mathbf{a}}\mathcal{L}(\mathbf{Y},\mathbf{a}^{\star}(X))
\end{equation*}
satisfies
\begin{equation*}
\mathbb{E}[\boldsymbol{\delta}^{\star}\mid X]=\mathbf{0}.
\end{equation*}
If, in addition, $\mathcal A$ is convex and $C_x(\mathbf{a})$ is convex in $\mathbf{a}$ for almost every $x$, then the converse also holds: any measurable predictor $\tilde{\mathbf{a}}(X)$ satisfying
\begin{equation*}
\mathbb{E}\bigl[\nabla_{\mathbf{a}}\mathcal{L}(\mathbf{Y},\tilde{\mathbf{a}}(X))\mid X\bigr]=\mathbf{0}
\end{equation*}
minimizes $C_x(\mathbf{a})$ for almost every $x$; if $C_x(\mathbf{a})$ is strictly convex, this minimizer is unique almost surely.
\end{theorem}
\noindent\emph{Proof.} Forward part of the proof simply relies on the first order condition for optimality, and being able to exchange differentiation and expectation under the presumed regularity assumption. For each fixed input value $x$, differentiation of the conditional risk gives
\begin{equation*}
\nabla_{\mathbf{a}} C_x(\mathbf{a}) = \nabla_{\mathbf{a}}\mathbb{E}\bigl[\mathcal{L}(\mathbf{Y},\mathbf{a})\mid X=x\bigr].
\end{equation*}
If $\mathbf{a}^{\star}(x)$ is an interior minimizer of the differentiable function $C_x$, then its first-order optimality condition is
\begin{equation*}
\nabla_{\mathbf{a}} C_x(\mathbf{a}^{\star}(x))=\mathbf{0}.
\end{equation*}
Exchanging differentiation with expectation yields
\begin{equation*}
\mathbb{E}\bigl[\nabla_{\mathbf{a}}\mathcal{L}(\mathbf{Y},\mathbf{a}^{\star}(x))\mid X=x\bigr]=\mathbf{0}
\end{equation*}
for almost every $x$, which is exactly $\mathbb{E}[\boldsymbol{\delta}^{\star}\mid X]=\mathbf{0}$.

For the reverse direction, suppose $C_x$ is differentiable and convex and a predictor $\tilde{\mathbf{a}}(x)$ satisfies
\begin{equation*}
\mathbb{E}\bigl[\nabla_{\mathbf{a}}\mathcal{L}(\mathbf{Y},\tilde{\mathbf{a}}(X))\mid X\bigr]=\mathbf{0},
\end{equation*}
We show that $\tilde{\mathbf{a}}(x)$ minimizes $C_x$ for almost every $x$.

First, we translate the score condition into a gradient condition on $C_x$. Applying the gradient--expectation interchange used in the forward direction, but now evaluated at $\tilde{\mathbf{a}}(X)$ rather than at $\mathbf{a}^\star(X)$, gives
\begin{equation*}
\nabla_{\mathbf{a}} C_x(\tilde{\mathbf{a}}(x)) = \mathbb{E}\bigl[\nabla_{\mathbf{a}}\mathcal{L}(\mathbf{Y},\tilde{\mathbf{a}}(x))\mid X=x\bigr].
\end{equation*}
The right-hand side is $\mathbf{0}$ for almost every $x$ by assumption. Therefore
\begin{equation*}
\nabla_{\mathbf{a}} C_x(\tilde{\mathbf{a}}(x)) = \mathbf{0} \quad \text{for almost every } x.
\end{equation*}
As the next step, we use the fact that a zero gradient of a convex function characterizes a global minimizer. Fix any $x$ for which $C_x$ is differentiable, convex, and satisfies $\nabla_{\mathbf{a}} C_x(\tilde{\mathbf{a}}(x)) = \mathbf{0}$. By the first-order characterization of convexity, for every $\mathbf{a} \in \mathcal{A}$,
\begin{equation*}
C_x(\mathbf{a}) \geq C_x(\tilde{\mathbf{a}}(x)) + \nabla_{\mathbf{a}} C_x(\tilde{\mathbf{a}}(x))^T (\mathbf{a} - \tilde{\mathbf{a}}(x)) = C_x(\tilde{\mathbf{a}}(x)) + \mathbf{0}^T (\mathbf{a} - \tilde{\mathbf{a}}(x)) = C_x(\tilde{\mathbf{a}}(x)).
\end{equation*}
Hence $\tilde{\mathbf{a}}(x)$ is a global minimizer of $C_x$. Since this holds for almost every $x$, the predictor $\tilde{\mathbf{a}}$ minimizes $C_x$ pointwise almost surely.

Finally, in case of strict convexity of $C_x$, the inequality above is strict whenever $\mathbf{a} \neq \tilde{\mathbf{a}}(x)$:
\begin{equation*}
C_x(\mathbf{a}) > C_x(\tilde{\mathbf{a}}(x)) \quad \text{for all } \mathbf{a} \in \mathcal{A} \setminus \{\tilde{\mathbf{a}}(x)\}.
\end{equation*}
Therefore $\tilde{\mathbf{a}}(x)$ is the \emph{unique} minimizer of $C_x$, and the converse predictor is unique almost surely. \hfill$\square$

The theorem clarifies why this condition is stronger than ordinary stationarity of a parameterized model $\mathbf{a}_{\theta}(X)$. When optimization is carried out only over the parameter vector $\theta$, one generally obtains the weaker feature-level orthogonality of Proposition~\ref{prop:stationary_score_feature_orthogonality}. The stronger identity $\mathbb{E}[\boldsymbol{\delta}^{\star}\mid X]=\mathbf{0}$ appears when the optimal prediction value is characterized directly from the conditional risk for each input.

We apply the theorem to Bregman divergences (covering MSE), proper scoring rules through
unconstrained links (covering cross-entropy), and exponential-family
negative log-likelihoods.

\paragraph{Bregman losses.}
Bregman divergences~\citep{bregman1967relaxation} form a broad family of losses generated by a strictly convex potential function $\phi$. Important members include MSE (from $\phi(\mathbf{u})=\tfrac{1}{2}\|\mathbf{u}\|_2^2$), the Itakura--Saito divergence~\citep{itakura1968analysis,fevotte2009nonnegative}, and the generalized KL divergence. We show that for any such loss, the population-optimal predictor is a conditional expectation, and the score is therefore conditionally mean zero.

Let $\Omega\subseteq\mathbb{R}^m$ be open and convex, and let $\phi:\Omega\to\mathbb{R}$ be twice differentiable with positive definite Hessian $\nabla^2\phi$. Let $\tau:\mathcal{Y}\to\Omega$ be a target encoding. Assume finite conditional first moments and Bregman risks, with
$\boldsymbol\mu_x=\mathbb E[\tau(\mathbf Y)\mid X=x]\in\Omega$.
The associated Bregman divergence and loss are
\begin{equation*}
D_{\phi}(\mathbf{u},\mathbf{v}) = \phi(\mathbf{u}) - \phi(\mathbf{v}) - \nabla\phi(\mathbf{v})^T(\mathbf{u}-\mathbf{v}),
\qquad
\mathcal{L}_{\phi}(\mathbf{y},\mathbf{a}) = D_{\phi}(\tau(\mathbf{y}),\mathbf{a}).
\end{equation*}
Differentiating with respect to the second argument gives
\begin{equation*}
\nabla_{\mathbf{a}} \mathcal{L}_{\phi}(\mathbf{y},\mathbf{a})
= \nabla^2\phi(\mathbf{a})\bigl(\mathbf{a} - \tau(\mathbf{y})\bigr).
\end{equation*}
Taking the conditional expectation over $\mathbf{Y}\mid X=x$ and using linearity yields the gradient of the conditional risk:
\begin{equation*}
\nabla_{\mathbf{a}} C_x(\mathbf{a})
= \nabla^2\phi(\mathbf{a})\Bigl(\mathbf{a} - \mathbb{E}[\tau(\mathbf{Y})\mid X = x]\Bigr).
\end{equation*}
Global optimality follows from the Bregman decomposition
\[
C_x(\mathbf a)=C_x(\boldsymbol\mu_x)+D_\phi(\boldsymbol\mu_x,\mathbf a).
\]
Strict convexity of $\phi$ makes the second term nonnegative, with equality
only at $\mathbf a=\boldsymbol\mu_x$. Thus the unique minimizer is
\begin{equation*}
\mathbf{a}^{\star}(X) = \mathbb{E}[\tau(\mathbf{Y})\mid X].
\end{equation*}
This is precisely the form required by Theorem~\ref{thm:conditional_risk_score}, so the optimal score $\boldsymbol{\delta}^{\star}=\nabla_{\mathbf{a}}\mathcal{L}_{\phi}(\mathbf{Y},\mathbf{a}^{\star}(X))$ satisfies $\mathbb{E}[\boldsymbol{\delta}^{\star}\mid X]=\mathbf{0}$. The MSE case $\phi(\mathbf{u})=\tfrac{1}{2}\|\mathbf{u}\|_2^2$, $\tau(\mathbf{Y})=\mathbf{Y}$ gives $\mathbf{a}^{\star}(X)=\mathbb{E}[\mathbf{Y}\mid X]$ and $\boldsymbol{\delta}^{\star}=\mathbf{a}^{\star}(X)-\mathbf{Y}$, recovering the MMSE estimation-error orthogonality used by EBD as a single Bregman instance.

\paragraph{Proper probability losses through unconstrained links.}
Probabilistic classification uses losses defined on the probability simplex: cross-entropy, Brier score, log loss, and so on. These losses share a strong optimality property~\citep{gneiting2007strictly}: their conditional risks are minimized at the true conditional class-probability vector. The complication is that the simplex $\Delta_D$ is a constrained domain, so Theorem~\ref{thm:conditional_risk_score} does not apply directly. The standard fix is to compose the simplex-domain loss with an unconstrained link function (e.g., softmax over logits), which moves the problem into the open Euclidean setting where Theorem~\ref{thm:conditional_risk_score} applies.

Let $\ell(\mathbf{y},\mathbf{p})$ be a differentiable strictly proper loss defined on the interior of the probability simplex, and let
\begin{equation*}
\mathbf{q}(X) = \mathbb{E}[\mathbf{Y}\mid X]
\end{equation*}
denote the true conditional class-probability vector, assumed to lie in the interior of the simplex almost surely. Strict propriety means that, for each fixed $x$, the simplex-domain conditional risk
\begin{equation*}
\mathbf{p}\mapsto \mathbb{E}\bigl[\ell(\mathbf{Y},\mathbf{p})\mid X = x\bigr]
\end{equation*}
is uniquely minimized at $\mathbf{p}=\mathbf{q}(x)$.

To bring this into the unconstrained setting required by Theorem~\ref{thm:conditional_risk_score}, we reparameterize through a differentiable link $\psi:\mathbb{R}^m\to\mathrm{int}(\Delta_D)$ that maps an unconstrained vector $\mathbf{a}$ to a probability vector $\mathbf{p}=\psi(\mathbf{a})$, and define the unconstrained loss
\begin{equation*}
\mathcal{L}(\mathbf{y},\mathbf{a}) = \ell(\mathbf{y},\psi(\mathbf{a})).
\end{equation*}
The SBD score is then taken with respect to the unconstrained variable:
\begin{equation*}
\boldsymbol{\delta} = \nabla_{\mathbf{a}}\mathcal{L}(\mathbf{y},\mathbf{a}).
\end{equation*}
Suppose that for almost every $x$, the unconstrained conditional risk
\begin{equation*}
C_x(\mathbf{a}) = \mathbb{E}\bigl[\ell(\mathbf{Y},\psi(\mathbf{a}))\mid X = x\bigr]
\end{equation*}
has an interior minimizer $\mathbf{a}^{\star}(x)$ satisfying $\psi(\mathbf{a}^{\star}(x)) = \mathbf{q}(x)$, and that the gradient--expectation interchange holds. Then by Theorem~\ref{thm:conditional_risk_score},
\begin{equation*}
\mathbb{E}\bigl[\nabla_{\mathbf{a}}\ell(\mathbf{Y},\psi(\mathbf{a}^{\star}(X)))\mid X\bigr] = \mathbf{0},
\end{equation*}
which is exactly the conditional-mean-zero score property used by SBD.

\textit{Softmax cross-entropy.} The canonical instance is $\psi=\softmax$, $\ell(\mathbf{y},\mathbf{p})=-\sum_{d=1}^D y_d\log p_d$. A direct computation gives
\begin{equation*}
\mathbf{p} = \softmax(\mathbf{a}),
\qquad
\nabla_{\mathbf{a}}\mathcal{L}_{\mathrm{CE}}(\mathbf{y},\mathbf{a}) = \mathbf{p} - \mathbf{y}.
\end{equation*}
At the population optimum, $\mathbf{p}^{\star}(X) = \psi(\mathbf{a}^{\star}(X)) = \mathbf{q}(X)$, so
\begin{equation*}
\mathbb{E}\bigl[\mathbf{p}^{\star}(X) - \mathbf{Y}\mid X\bigr]
= \mathbf{p}^{\star}(X) - \mathbf{q}(X) = \mathbf{0}.
\end{equation*}
This is the cross-entropy score $\mathbf{p}-\mathbf{y}$ used throughout the main text, recovered as the unconstrained-logit specialization of the proper-loss construction.

\paragraph{Negative log-likelihood losses.}
Maximum-likelihood estimation in a parametric family yields a third source of differentiable losses. We show that whenever the parameter space is open and the conditional risk has an interior minimizer, the score is conditionally mean zero. Furthermore, for exponential families, this conditional-mean-zero property reduces to the classical moment-matching identity.

Let $\{p_{\eta}:\eta\in\mathcal{H}\}$ be a differentiable parametric family of densities with open parameter space $\mathcal{H}\subseteq\mathbb{R}^m$, and define the negative log-likelihood loss
\begin{equation*}
\mathcal{L}_{\mathrm{NLL}}(\mathbf{y},\eta) = -\log p_{\eta}(\mathbf{y}).
\end{equation*}
The corresponding conditional risk is the conditional cross-entropy from the true distribution to $p_{\eta}$,
\begin{equation*}
C_x(\eta) = \mathbb{E}\bigl[-\log p_{\eta}(\mathbf{Y})\mid X = x\bigr].
\end{equation*}
Suppose $C_x$ has an interior minimizer $\eta^{\star}(x)\in\mathcal{H}$ for almost every $x$, and the gradient--expectation interchange holds. Theorem~\ref{thm:conditional_risk_score} then gives
\begin{equation*}
\mathbb{E}\bigl[-\nabla_{\eta}\log p_{\eta^{\star}(X)}(\mathbf{Y})\mid X\bigr] = \mathbf{0}.
\end{equation*}
The NLL score is the negative of the statistical likelihood score
$\nabla_\eta\log p_\eta(\mathbf Y)$; both have zero conditional mean here.

\textit{Exponential families: moment matching.} For an exponential family in canonical form,
\begin{equation*}
p_{\eta}(\mathbf{y}) = h(\mathbf{y})\exp\bigl(\eta^TT(\mathbf{y}) - A(\eta)\bigr),
\end{equation*}
the score has a particularly clean form: $\nabla_{\eta}\log p_{\eta}(\mathbf{y}) = T(\mathbf{y}) - \nabla A(\eta)$, so
\begin{equation*}
\boldsymbol{\delta}_{\mathrm{NLL}} = \nabla A(\eta) - T(\mathbf{Y}).
\end{equation*}
The conditional-mean-zero score condition $\mathbb{E}[\boldsymbol{\delta}_{\mathrm{NLL}}^{\star}\mid X]=\mathbf{0}$ then reads
\begin{equation*}
\nabla A(\eta^{\star}(X)) = \mathbb{E}[T(\mathbf{Y})\mid X],
\end{equation*}
which is exactly the classical moment-matching condition for maximum-likelihood estimation in exponential families: the model expected sufficient statistic equals the conditional expected sufficient statistic. Important instances include Gaussian NLL (with $T(\mathbf{Y})=\mathbf{Y}$, recovering the conditional-mean predictor), Bernoulli logistic loss, Poisson NLL, and multinomial softmax cross-entropy when written in natural-parameter coordinates.

\subsection{Proof-of-concept demonstration: Poisson NLL regression}
\label{app:poisson_demo}

The main paper's experiments target cross-entropy classification, which together
with the MSE setting of \citet{erdoganerror} covers two of the loss families
admitted by SBD. To illustrate the framework on an exponential-family negative
log-likelihood outside these two, we run a small synthetic Poisson
regression experiment within the scope of
Theorem~\ref{thm:conditional_risk_score} and Appendix~\ref{app:sec6_appendix}.
We compare BP and SBD through
test NLL, binned residual calibration, and score--activation correlations,
using the known population-optimal predictor as an oracle reference.

\subsubsection{Data-generating process.}
We sample $X \sim \mathcal{N}(\mathbf{0}, I_8)$ and define the ground-truth
log-rate
\begin{equation*}
f^{\star}(x) = 1.0 + 0.4\sin(x_1)\cos(x_2) + 0.3\,x_3 x_4
- 0.15(x_5^2 - 1) + 0.10(x_6 + x_7) + 0.10\tanh(x_8),
\end{equation*}
as a synthetic regression function in the style of Friedman 1 \citep{friedman1991mars}, so that  $Y \mid X \sim \mathrm{Poisson}\bigl(\exp f^{\star}(X)\bigr)$. We
generate $50{,}000$ training and $10{,}000$ test samples, with $f^{\star}$
clipped to $[-1.5, 3.5]$ before sampling $Y$ to bound the generating rates.
Hereafter, $f^{\star}$ denotes this clipped log-rate. Because it is known
explicitly, we can estimate the Bayes NLL
$\mathcal{R}^{\star} = \mathbb{E}[\exp f^{\star}(X) - Y\,f^{\star}(X)]$ and the
oracle's finite-sample calibration diagnostic (defined below).

\subsubsection{Loss and score.}
The Poisson NLL with log-rate parameterization is
$\mathcal{L}_{\mathrm{NLL}}(y, a) = \exp(a) - y\,a$ (up to a $y$-only constant),
giving the SBD score
\begin{equation*}
\boldsymbol{\delta}(y, a) = \nabla_a \mathcal{L}_{\mathrm{NLL}}(y, a)
= \exp(a) - y.
\end{equation*}
This is the canonical exponential-family score $\nabla A(\eta) - T(Y)$ with
$T(Y) = Y$ and $A(\eta) = \exp(\eta)$, matching the structure of
Appendix~\ref{app:sec6_appendix}. The conditional-mean-zero property at the
optimum reduces to the moment-matching identity
$\exp(a^{\star}(X)) = \mathbb{E}[Y \mid X]$.

\subsubsection{Derivation of the oracle Bayes NLL.}
The Poisson NLL on a sample $(X,Y)$ in log-rate parameterization is, up to a
$Y$-only constant,
\begin{equation*}
\mathcal{L}_{\mathrm{NLL}}(Y, a) = \exp(a) - Y\,a.
\end{equation*}
The conditional risk for a fixed input $x$ is
\begin{equation*}
C_x(a) = \mathbb{E}\bigl[\mathcal{L}_{\mathrm{NLL}}(Y,a) \mid X = x\bigr]
= \exp(a) - \mathbb{E}[Y \mid X = x]\,a.
\end{equation*}
For the synthetic data-generating process $Y \mid X \sim \mathrm{Poisson}\bigl(\exp f^{\star}(X)\bigr)$,
the conditional mean is $\mathbb{E}[Y \mid X = x] = \exp f^{\star}(x)$.
Differentiating $C_x$ in $a$ and setting the gradient to zero gives
$\exp(a^{\star}(x)) = \exp f^{\star}(x)$, so the population-optimal predictor
is $a^{\star}(x) = f^{\star}(x)$. Substituting back yields the conditional
risk at the optimum
\begin{equation*}
C_x(a^{\star}(x)) = \exp f^{\star}(x) \cdot \bigl(1 - f^{\star}(x)\bigr).
\end{equation*}
The oracle Bayes NLL is the marginal expectation of this conditional risk:
\begin{equation*}
\mathcal{R}^{\star} = \mathbb{E}_X\!\bigl[C_X(a^{\star}(X))\bigr]
= \mathbb{E}\!\bigl[\exp f^{\star}(X) - Y\,f^{\star}(X)\bigr],
\end{equation*}
where the second equality uses the tower property and
$\mathbb{E}[Y \mid X] = \exp f^{\star}(X)$. Because $f^{\star}$ is known, we
estimate $\mathcal{R}^{\star}$ by Monte Carlo on the test set,
$\widehat{\mathcal{R}}^{\star} = N_{\mathrm{test}}^{-1}
\sum_{n=1}^{N_{\mathrm{test}}}
\bigl[\exp f^{\star}(X_n) - Y_n\,f^{\star}(X_n)\bigr]$, which converges to
$\mathcal{R}^{\star}$ at rate $N_{\mathrm{test}}^{-1/2}$. For the test set
of $10{,}000$ samples used here, this yields $\widehat{\mathcal{R}}^{\star}
\approx -0.5595$. This empirical oracle value estimates the population Bayes
risk; it is not a lower bound for every finite-sample test NLL.

\subsubsection{Network and methods.}
We use a three-layer MLP with widths $8 \to 128 \to 64 \to 1$, ReLU activations,
bias terms, and Kaiming--He initialization~\citep{he2015rectifiers}. The network output is the scalar
log-rate $a(x)$. Two methods are compared:
\begin{itemize}
\item \textbf{BP:} backpropagation of the exact Poisson NLL gradient.
\item \textbf{SBD:} the rank-$1$ score broadcast of
Section~\ref{sec:general_losses_sbd} (no score-vector expansion). The output
layer uses the standard NLL gradient; the two hidden layers use the local SBD
update with broadcast cross-correlation $\widehat{\mathbf{R}}_{\mathrm{Score}}^{(k)}$ projecting
$\boldsymbol{\delta} = \exp(a) - y$ to each layer.
\end{itemize}
Both methods use Adam with $\eta_0 = 3\!\times\!10^{-3}$, $0.99$
per-epoch learning-rate decay, batch size $64$, and $200$ epochs over five runs.
Weight decay is $5\!\times\!10^{-4}$ for BP and $10^{-5}$ for SBD.
SBD uses $\lambda = 0.99999$ for the broadcast exponential moving average (EMA),
initialized with independent $\mathcal{N}(0,0.01^2)$ entries.

\subsubsection{Evaluation metrics.}
We track test NLL and binned residual calibration per epoch, and evaluate
score--activation correlations after training.
\begin{itemize}
\item \emph{Test Poisson NLL.} The test-average loss
$\widehat{\mathcal{R}}$, compared with the empirical oracle
$\widehat{\mathcal{R}}^{\star}$ on the same test set.
\item \emph{Binned residual calibration (CMZ in the figures).} We bin the test set into
$K = 20$ quantile bins of the predicted log-rate $\hat{a}(X)$, compute
$|\widehat{\mathbb{E}}[\boldsymbol{\delta}\mid \text{bin}_k]|$ within each
bin, and average the absolute values. This tests residual calibration with
respect to binned predictions, a coarser condition than
$\mathbb{E}[\boldsymbol{\delta}\mid X]=0$. Replacing $\hat a$ with $f^{\star}$
in both the score and binning gives an oracle reference of $0.0628$;
finite-sample fluctuations make this reference nonzero, but it is not a
lower bound on the diagnostic.
\item \emph{Score--activation correlations.} Average absolute Pearson
correlation between $\boldsymbol{\delta}$ and each hidden-layer activation,
evaluated on the test set for one trained network per method. These are
empirical diagnostics of the orthogonality predicted at the population
optimum by Theorem~\ref{thm:general_score_orthogonality}.
\end{itemize}

\paragraph{Results.}
Both methods converge to near-oracle test Poisson NLL: BP attains a final
test NLL of $-0.5491 \pm 0.0008$ and SBD $-0.5452 \pm 0.0023$, compared with
$\widehat{\mathcal{R}}^{\star} = -0.5595$.
Reported uncertainties and shaded bands are standard deviations across
five runs, computed with divisor $n=5$ as in the notebook.
Figure~\ref{fig:poisson_excess_nll} shows the test NLL excess over the oracle
($\widehat{\mathcal{R}} - \widehat{\mathcal{R}}^{\star}$) on a logarithmic scale. Both BP and SBD
descend by roughly two and a half orders of magnitude during training and
plateau at small residual gaps, approximately $1\!\times\!10^{-2}$ for BP
and $1.5\!\times\!10^{-2}$ for SBD.

\begin{figure}[t]
\centering
\includegraphics[width=0.82\linewidth]{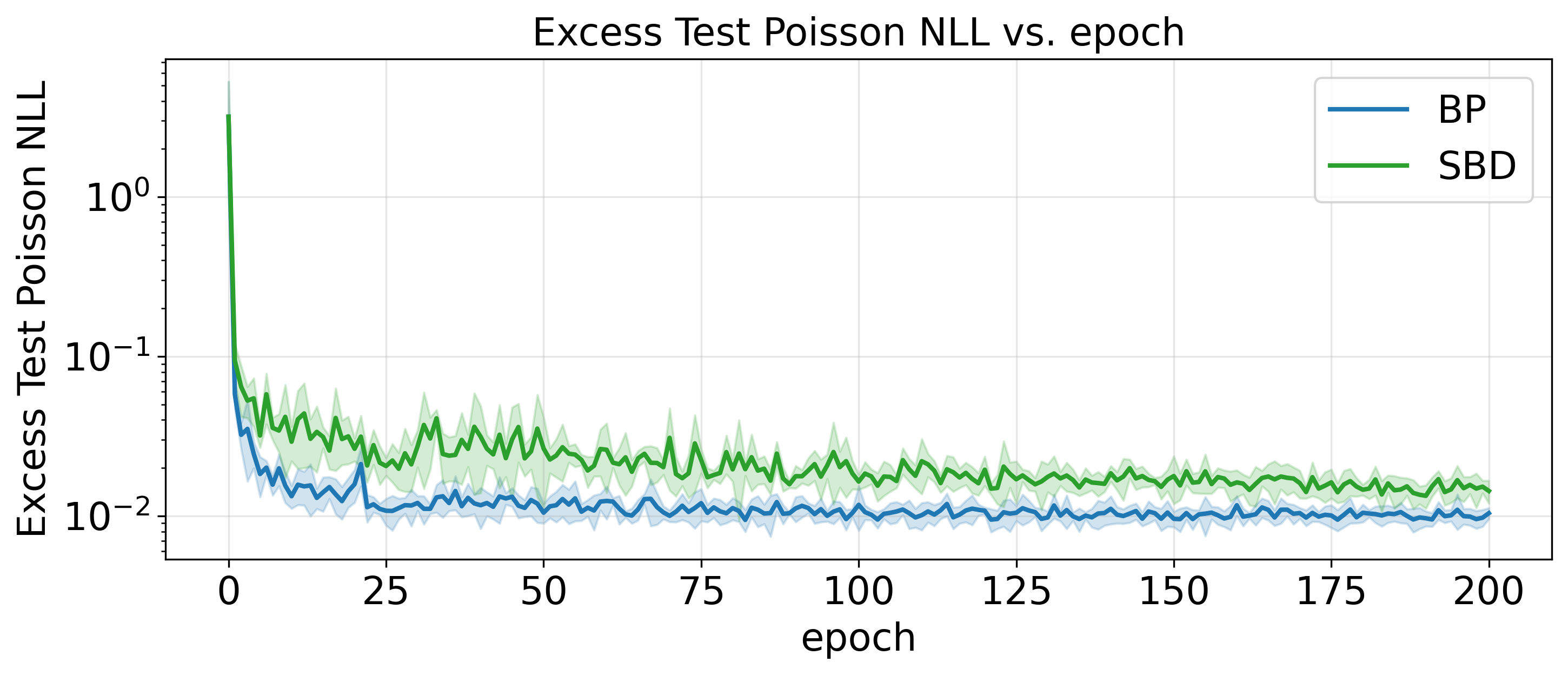}
\caption{Poisson proof-of-concept: test excess negative log-likelihood
relative to the empirical oracle, $\widehat{\mathcal{R}} - \widehat{\mathcal{R}}^{\star}$, on a
logarithmic scale. Both BP and SBD converge close to the oracle, with SBD
plateauing slightly above BP. Shading denotes one standard deviation across five runs.}
\label{fig:poisson_excess_nll}
\end{figure}

Figure~\ref{fig:poisson_cmz} shows binned residual calibration versus epoch.
Both methods reduce the diagnostic by more than an order of magnitude:
BP reaches $0.0726 \pm 0.0146$ and SBD $0.0872 \pm 0.0156$, compared with the
oracle reference of $0.0628$. Its decrease accompanies the decline in test
NLL. The post-training score--activation correlations range from $0.006$
to $0.019$ across the two hidden layers and methods.

\begin{figure}[t]
\centering
\includegraphics[width=0.82\linewidth]{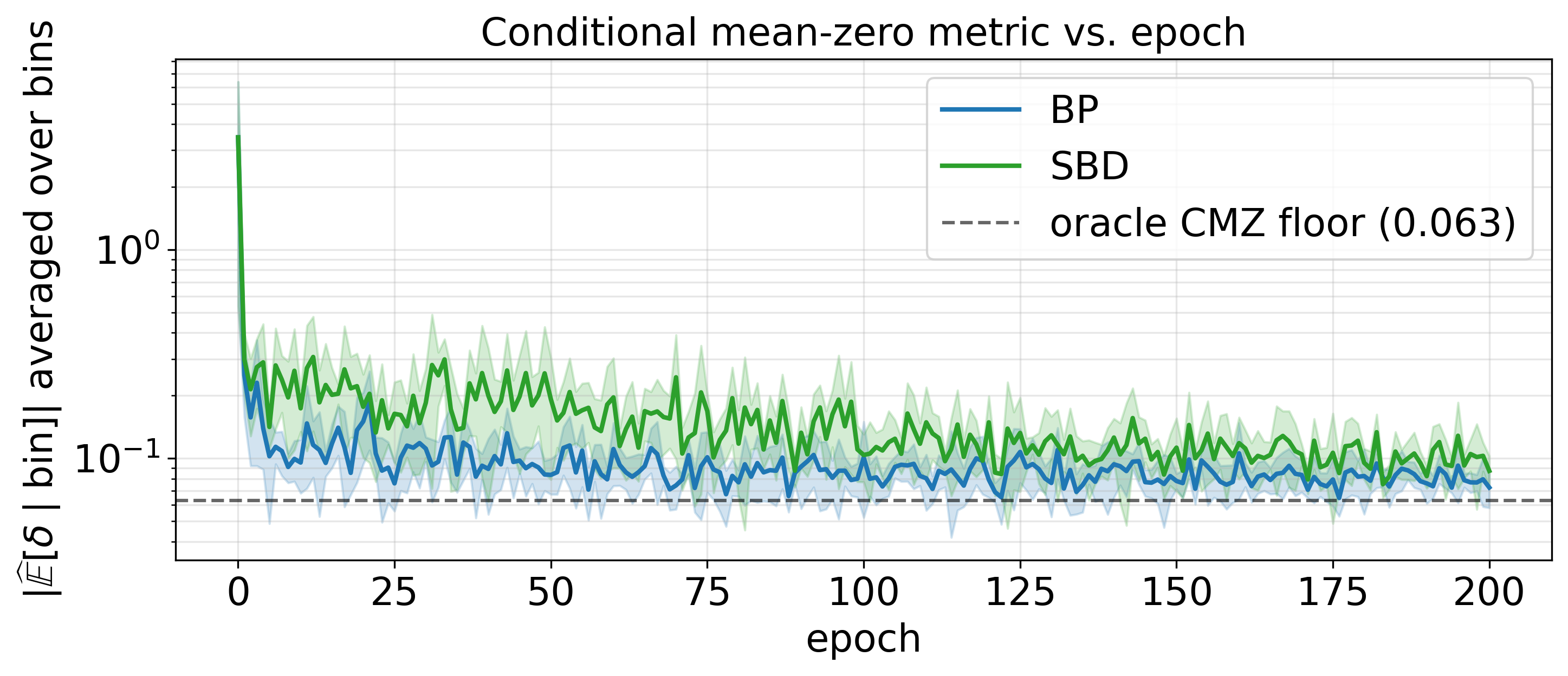}
\caption{Poisson proof-of-concept: binned residual calibration versus epoch
on a logarithmic scale. Shading denotes one standard deviation across five
runs. The dashed line, labeled ``oracle CMZ floor,'' is the oracle's
finite-sample diagnostic value, not a lower bound.}
\label{fig:poisson_cmz}
\end{figure}

Together, near-oracle test NLL, improved residual calibration, and small
score--activation correlations support the use of SBD on this Poisson
task. These finite-sample diagnostics are consistent with the population
theory, but do not establish full conditional mean zero or attainment of
the population optimum.

\section{Cross entropy SBD algorithm}
\label{app:ce_sbd_algorithm}

This appendix collects the cross entropy SBD procedure used in
Section~\ref{sec:cross_entropy_extension} into a single self-contained
reference. Table~\ref{tab:ce_sbd_algo} lists the per-minibatch operations
performed at every hidden layer $k$, where $\mathbf{h}^{(k)}\!\in\!\mathbb{R}^{N^{(k)}}$
is the layer-$k$ activation, $\mathbf{a}\!\in\!\mathbb{R}^{D_{\mathrm{out}}}$ are the
output logits, $\mathbf{p}=\softmax(\mathbf{a})$, and
$\boldsymbol{\delta}_{\mathrm{CE}}=\mathbf{p}-\mathbf{y}$ is the cross entropy
output score. The layerwise correlation estimate
$\widehat{\mathbf{R}}_{\mathrm{CE}}^{(k)}\!\in\!\mathbb{R}^{N^{(k)}\times D_{\mathrm{out}}}$
tracks the running cross-correlation between the hidden feature map
$g^{(k)}(\mathbf{h}^{(k)})$ and the score, and supplies the broadcast modulator
$\mathbf{q}_{\mathrm{CE}}^{(k)}$ used in the local three-factor weight update.
The output layer is trained with the standard cross entropy gradient.

\begin{table}[h]
\caption{Cross entropy SBD algorithm.}
\label{tab:ce_sbd_algo}
\centering
\begin{tabular}{p{0.10\linewidth}p{0.80\linewidth}}
\toprule
Step & Operation \\
\midrule
1 & Perform a forward pass to compute the hidden activations $\mathbf{h}^{(k)}$, the logits $\mathbf{a}$, the probabilities $\mathbf{p}=\softmax(\mathbf{a})$, and the score $\boldsymbol{\delta}_{\mathrm{CE}}=\mathbf{p}-\mathbf{y}$. \\
2 & Form the batch score matrix $\mathbf{\Delta}_{\mathrm{CE}}[m]=\mathbf{P}[m]-\mathbf{Y}[m]$ and update the layerwise correlation estimate $\widehat{\mathbf{R}}_{\mathrm{CE}}^{(k)}[m]=\lambda\,\widehat{\mathbf{R}}_{\mathrm{CE}}^{(k)}[m-1]+\tfrac{1-\lambda}{B}\,\mathbf{G}^{(k)}[m]\,\mathbf{\Delta}_{\mathrm{CE}}[m]^{T}$. \\
3 & For $n=mB+1,\ldots,(m+1)B$, project the broadcast score to hidden layer $k$ by computing $\mathbf{q}_{\mathrm{CE}}^{(k)}[n]=\widehat{\mathbf{R}}_{\mathrm{CE}}^{(k)}[m]\,(\mathbf{p}[n]-\mathbf{y}[n])$. \\
4 & For $B=1$, update the hidden-layer parameters with the local three-factor rule $\Delta W_{ij}^{(k)}[n]=\zeta\, g_i^{\prime(k)}\!\bigl(h_i^{(k)}[n]\bigr)\, f^{\prime(k)}\!\bigl(u_i^{(k)}[n]\bigr)\, q_{\mathrm{CE},i}^{(k)}[n]\, h_j^{(k-1)}[n]$, and update the biases analogously; for $B>1$, use Eq.~\eqref{eq:plain_batch_update_W}. \\
5 & Update the output layer with the standard cross entropy gradient. \\
\bottomrule
\end{tabular}
\end{table}

\subsection{Background on MMSE estimation theory and EBD}
\label{app:ebd-mmse-background}

Error Broadcast and Decorrelation (EBD) begins from a standard property of
nonlinear minimum mean square error (MMSE) estimation
\citep{erdoganerror}. For random vectors $\mathbf{X}$ and $\mathbf{Y}$, the
optimal MMSE estimator is the conditional mean
\begin{equation*}
\widehat{\mathbf{y}}^{\star}(\mathbf{X})
=\mathbb{E}[\mathbf{Y}\mid\mathbf{X}].
\end{equation*}
Its estimation error
$\boldsymbol{\epsilon}^{\star}
=\mathbf{Y}-\widehat{\mathbf{y}}^{\star}(\mathbf{X})$
satisfies the stochastic orthogonality condition: it is orthogonal to every
measurable function $\mathbf{g}(\mathbf{X})$ of the input for which the
expectation exists,
\begin{equation}
\mathbb{E}[\boldsymbol{\epsilon}^{\star}\mid\mathbf{X}]=\mathbf{0}
\quad\Longrightarrow\quad
\mathbb{E}\!\left[
\mathbf{g}(\mathbf{X})\boldsymbol{\epsilon}^{\star T}
\right]=\mathbf{0}.
\label{eq:app-mmse-orthogonality}
\end{equation}
EBD further uses the sufficiency of this condition for nonlinear MMSE
estimation: when the available functions are sufficiently rich, enforcing
their orthogonality with the estimation error identifies the conditional-mean
estimator \citep{erdoganerror}.

EBD applies this principle to a neural network trained with MSE. Writing its
output error as
$\boldsymbol{\epsilon}=\mathbf{h}^{(L)}-\mathbf{y}$, the hidden-layer
activations and their nonlinear transforms
$\mathbf{g}^{(k)}(\mathbf{h}^{(k)})$ provide measurable functions of the
network input. At each hidden layer, EBD therefore seeks to enforce zero
cross-correlation between the transformed layer activations and the output
error:
\begin{equation*}
\mathbf{R}_{\mathbf{g}\boldsymbol{\epsilon}}^{(k)}
=\mathbb{E}\!\left[
\mathbf{g}^{(k)}(\mathbf{h}^{(k)})\boldsymbol{\epsilon}^{T}
\right],
\qquad
\mathcal{J}_{\mathrm{EBD}}^{(k)}
=\frac{1}{2}\left\|\mathbf{R}_{\mathbf{g}\boldsymbol{\epsilon}}^{(k)}\right\|_{F}^{2}.
\end{equation*}
For a minibatch of size $B$, with
$\mathbf{G}^{(k)}[m]\in\mathbb{R}^{N^{(k)}\times B}$ and
$\mathbf{E}[m]\in\mathbb{R}^{D_{\mathrm{out}}\times B}$, this correlation is
estimated autoregressively with forgetting factor $\lambda$:
\begin{equation*}
\widehat{\mathbf{R}}_{\mathbf{g}\boldsymbol{\epsilon}}^{(k)}[m]
=\lambda\widehat{\mathbf{R}}_{\mathbf{g}\boldsymbol{\epsilon}}^{(k)}[m-1]
+\frac{1-\lambda}{B}\mathbf{G}^{(k)}[m]\mathbf{E}[m]^{T},
\qquad
\widehat{\mathbf{R}}_{\mathbf{g}\boldsymbol{\epsilon}}^{(k)}[m]
\in\mathbb{R}^{N^{(k)}\times D_{\mathrm{out}}}.
\end{equation*}

The derivative of the layer-specific loss with respect to a layer parameter
decomposes into two gradient components. The first contains derivatives of the
transformed layer activations; the second contains derivatives of the output
error. The second component requires signals propagated backward from the
output layer, whereas the first depends only on the layer activations and the
broadcast output error. EBD retains the first component to obtain its localized
update. For a single sample, it defines
\begin{equation*}
\mathbf{q}^{(k)}
=\widehat{\mathbf{R}}_{\mathbf{g}\boldsymbol{\epsilon}}^{(k)}
\boldsymbol{\epsilon},
\qquad
\Delta W_{ij}^{(k)}
=\zeta\,
g_i^{\prime(k)}(h_i^{(k)})
f^{\prime(k)}(u_i^{(k)})
q_i^{(k)}h_j^{(k-1)}.
\end{equation*}
Here $\mathbf{q}^{(k)}$ is the projection of the output error onto the layer
activations through the estimated cross-correlation matrix. The update has the
three-factor form emphasized in EBD: the presynaptic activity
$h_j^{(k-1)}$, the postsynaptic factor
$g_i^{\prime(k)}f^{\prime(k)}$, and the projected global error
$q_i^{(k)}$ as the modulatory factor. EBD is tied to the MSE output error and
its MMSE orthogonality property; SBD retains this broadcast-and-decorrelate
construction while replacing the output error with the appropriate loss
score. The derivation below specializes that construction to cross entropy.

\subsection{Derivation of the SBD update rule}
\label{app:ce_sbd_derivation}

This subsection derives the practical hidden-layer SBD update directly from
the minibatch correlation objective. For each sample $n$, the layer-$k$
preactivation and activation satisfy
\begin{equation*}
\mathbf{u}^{(k)}[n]
=\mathbf{W}^{(k)}\mathbf{h}^{(k-1)}[n]+\mathbf{b}^{(k)},
\qquad
\mathbf{h}^{(k)}[n]=f^{(k)}\!\left(\mathbf{u}^{(k)}[n]\right),
\end{equation*}
and $g^{(k)}$ denotes the elementwise nonlinear transform used in the
score--activation correlation.

The batch probability and label matrices are
\begin{equation*}
\mathbf{P}[m]=\bigl[\mathbf{p}[mB+1],\ldots,\mathbf{p}[mB+B]\bigr],
\qquad
\mathbf{Y}[m]=\bigl[\mathbf{y}[mB+1],\ldots,\mathbf{y}[mB+B]\bigr]
\end{equation*}
with $\mathbf{\Delta}_{\mathrm{CE}}[m]=\mathbf{P}[m]-\mathbf{Y}[m]$.

For minibatch $m$, let
\begin{align*}
\mathbf{H}^{(k-1)}[m]
&=\bigl[\mathbf{h}^{(k-1)}[mB+1],\ldots,\mathbf{h}^{(k-1)}[(m+1)B]\bigr],\\
\mathbf{G}^{(k)}[m]
&=\bigl[g^{(k)}(\mathbf{h}^{(k)}[mB+1]),\ldots,g^{(k)}(\mathbf{h}^{(k)}[(m+1)B])\bigr],\\
\mathbf{\Delta}_{\mathrm{CE}}[m]
&=\bigl[\boldsymbol{\delta}_{\mathrm{CE}}[mB+1],\ldots,\boldsymbol{\delta}_{\mathrm{CE}}[(m+1)B]\bigr],
\end{align*}
where $\boldsymbol{\delta}_{\mathrm{CE}}[n]=\mathbf{p}[n]-\mathbf{y}[n]$.
Treating $\widehat{\mathbf{R}}_{\mathrm{CE}}^{(k)}[m-1]$ as a fixed state when
differentiating the current minibatch objective, define
\begin{equation*}
\alpha:=\frac{1-\lambda}{B},
\qquad
\widehat{\mathbf{R}}_{\mathrm{CE}}^{(k)}[m]=\lambda\widehat{\mathbf{R}}_{\mathrm{CE}}^{(k)}[m-1]+\alpha\,\mathbf{G}^{(k)}[m]\mathbf{\Delta}_{\mathrm{CE}}[m]^T,
\end{equation*}
and define
\begin{equation*}
\mathcal{J}_{\mathrm{CE}}^{(k)}[m]=\frac{1}{2}\left\|\widehat{\mathbf{R}}_{\mathrm{CE}}^{(k)}[m]\right\|_F^2
=\frac{1}{2}\,\mathrm{Tr}\!\left(\widehat{\mathbf{R}}_{\mathrm{CE}}^{(k)}[m]^T\widehat{\mathbf{R}}_{\mathrm{CE}}^{(k)}[m]\right).
\end{equation*}
Hence, for a hidden-layer weight $W_{ij}^{(k)}$,
\begin{align*}
\frac{\partial \mathcal{J}_{\mathrm{CE}}^{(k)}[m]}{\partial W_{ij}^{(k)}}
&=\mathrm{Tr}\!\left(\widehat{\mathbf{R}}_{\mathrm{CE}}^{(k)}[m]^T\frac{\partial \widehat{\mathbf{R}}_{\mathrm{CE}}^{(k)}[m]}{\partial W_{ij}^{(k)}}\right)\\
&=\alpha\,\mathrm{Tr}\!\left(\widehat{\mathbf{R}}_{\mathrm{CE}}^{(k)}[m]^T\frac{\partial \mathbf{G}^{(k)}[m]}{\partial W_{ij}^{(k)}}\mathbf{\Delta}_{\mathrm{CE}}[m]^T\right)
+\alpha\,\mathrm{Tr}\!\left(\widehat{\mathbf{R}}_{\mathrm{CE}}^{(k)}[m]^T\mathbf{G}^{(k)}[m]\frac{\partial \mathbf{\Delta}_{\mathrm{CE}}[m]^T}{\partial W_{ij}^{(k)}}\right).
\end{align*}
The first gradient component contains derivatives of the transformed
layer-$k$ activations with respect to $W_{ij}^{(k)}$, whereas the second
contains
$\partial\mathbf{\Delta}_{\mathrm{CE}}[m]/\partial W_{ij}^{(k)}$ and
therefore depends on how the logits change through all deeper layers. Indeed,
because the label vector is constant with respect to the hidden-layer
parameters,
\begin{equation*}
\frac{\partial \boldsymbol{\delta}_{\mathrm{CE}}[n]}{\partial W_{ij}^{(k)}}
=\frac{\partial \mathbf{p}[n]}{\partial W_{ij}^{(k)}}
=\left(\operatorname{diag}(\mathbf{p}[n])-\mathbf{p}[n]\mathbf{p}[n]^T\right)
\frac{\partial \mathbf{a}[n]}{\partial W_{ij}^{(k)}},
\end{equation*}
and $\partial \mathbf{a}[n]/\partial W_{ij}^{(k)}$ expands into a chain of
derivatives through layers $k+1,\ldots,L$. Retaining this component would
therefore reintroduce layer-by-layer propagation from the output. The
practical SBD rule retains only the activation-derivative component, the
score-based counterpart of the localized EBD update
\citep{erdoganerror}.

Define the layerwise projected-score matrix
\begin{equation*}
\mathbf{Q}_{\mathrm{CE}}^{(k)}[m]
:=\widehat{\mathbf{R}}_{\mathrm{CE}}^{(k)}[m]\mathbf{\Delta}_{\mathrm{CE}}[m]
=\bigl[\mathbf{q}_{\mathrm{CE}}^{(k)}[mB+1],\ldots,\mathbf{q}_{\mathrm{CE}}^{(k)}[(m+1)B]\bigr].
\end{equation*}
Using cyclicity of the trace, the retained activation-derivative component
becomes
\begin{equation*}
\left.\frac{\partial \mathcal{J}_{\mathrm{CE}}^{(k)}[m]}{\partial W_{ij}^{(k)}}\right|_{\mathrm{local}}
=\alpha\,\mathrm{Tr}\!\left(\mathbf{Q}_{\mathrm{CE}}^{(k)}[m]\frac{\partial \mathbf{G}^{(k)}[m]^T}{\partial W_{ij}^{(k)}}\right).
\end{equation*}
Let $\mathbf{e}_i\in\mathbb{R}^{N^{(k)}}$ denote the $i$th canonical basis
vector. Then
\begin{equation*}
\frac{\partial \mathbf{G}^{(k)}[m]}{\partial W_{ij}^{(k)}}
=\mathbf{e}_i
\begin{bmatrix}
g_i^{\prime(k)}\!\left(h_i^{(k)}[mB+1]\right)f^{\prime(k)}\!\left(u_i^{(k)}[mB+1]\right)h_j^{(k-1)}[mB+1]\\
\vdots\\
g_i^{\prime(k)}\!\left(h_i^{(k)}[(m+1)B]\right)f^{\prime(k)}\!\left(u_i^{(k)}[(m+1)B]\right)h_j^{(k-1)}[(m+1)B]
\end{bmatrix}^{T}.
\end{equation*}
Substituting this derivative yields
\begin{equation*}
\left.\frac{\partial \mathcal{J}_{\mathrm{CE}}^{(k)}[m]}{\partial W_{ij}^{(k)}}\right|_{\mathrm{local}}
=\alpha\sum_{n=mB+1}^{(m+1)B}
g_i^{\prime(k)}\!\left(h_i^{(k)}[n]\right)
f^{\prime(k)}\!\left(u_i^{(k)}[n]\right)
q_{\mathrm{CE},i}^{(k)}[n]
h_j^{(k-1)}[n].
\end{equation*}
An identical calculation for the bias gives
\begin{equation*}
\left.\frac{\partial \mathcal{J}_{\mathrm{CE}}^{(k)}[m]}{\partial b_i^{(k)}}\right|_{\mathrm{local}}
=\alpha\sum_{n=mB+1}^{(m+1)B}
g_i^{\prime(k)}\!\left(h_i^{(k)}[n]\right)
f^{\prime(k)}\!\left(u_i^{(k)}[n]\right)
q_{\mathrm{CE},i}^{(k)}[n].
\end{equation*}

To write the update compactly, define the derivative matrices
\begin{align*}
\mathbf{G}_{d}^{(k)}[m]
&=\bigl[\mathbf{g}^{\prime(k)}(\mathbf{h}^{(k)}[mB+1]),\ldots,\mathbf{g}^{\prime(k)}(\mathbf{h}^{(k)}[(m+1)B])\bigr],\\
\mathbf{F}_{d}^{(k)}[m]
&=\bigl[\mathbf{f}^{\prime(k)}(\mathbf{u}^{(k)}[mB+1]),\ldots,\mathbf{f}^{\prime(k)}(\mathbf{u}^{(k)}[(m+1)B])\bigr],
\end{align*}
where the derivatives are applied elementwise, and set
\begin{equation*}
\mathbf{Z}_{\mathrm{CE}}^{(k)}[m]
=\mathbf{G}_{d}^{(k)}[m]\odot \mathbf{F}_{d}^{(k)}[m]\odot \mathbf{Q}_{\mathrm{CE}}^{(k)}[m].
\end{equation*}
Then the retained minibatch-local gradients are
\begin{equation*}
\left.\nabla_{\mathbf{W}^{(k)}}\mathcal{J}_{\mathrm{CE}}^{(k)}[m]\right|_{\mathrm{local}}
=\alpha\,\mathbf{Z}_{\mathrm{CE}}^{(k)}[m]\mathbf{H}^{(k-1)}[m]^T,
\qquad
\left.\nabla_{\mathbf{b}^{(k)}}\mathcal{J}_{\mathrm{CE}}^{(k)}[m]\right|_{\mathrm{local}}
=\alpha\,\mathbf{Z}_{\mathrm{CE}}^{(k)}[m]\mathbf{1}_{B\times 1}.
\end{equation*}
For gradient descent with learning rate $\eta>0$, let
$\Delta\mathbf W^{(k)}$ denote the parameter increment and set
$\zeta=-\eta(1-\lambda)$. Adaptive optimizers act on local directions,
with implementation-specific rescaling described in the experimental
appendices. The single-sample gradient-descent rule ($B=1$) is
\begin{equation*}
\Delta W_{ij}^{(k)}[n]
=\zeta\,g_i^{\prime(k)}\!\left(h_i^{(k)}[n]\right)
f^{\prime(k)}\!\left(u_i^{(k)}[n]\right)
q_{\mathrm{CE},i}^{(k)}[n]
h_j^{(k-1)}[n],
\end{equation*}
with the bias update obtained by omitting the presynaptic factor
$h_j^{(k-1)}[n]$. This is exactly the cross entropy SBD rule reported in
Section~\ref{sec:cross_entropy_extension}.

For general $B$, the expression for the update can be written as 

\begin{equation}
\label{eq:plain_batch_update_W}
\Delta W_{ij}^{(k)}[m]
=\frac{\zeta}{B}\sum_{n=mB+1}^{mB+B}\,g_i^{\prime(k)}\!\left(h_i^{(k)}[n]\right)
f^{\prime(k)}\!\left(u_i^{(k)}[n]\right)
q_{\mathrm{CE},i}^{(k)}[n]
h_j^{(k-1)}[n].
\end{equation}

\subsection{Extension to convolutional layers}
\label{app:cnn_sbd_derivation}

The MLP derivation of the previous subsection extends to convolutional
architectures by using every channel--location activation as a function
entering the stochastic orthogonality condition and retaining only its
dependence on the corresponding filter coefficients. This is the score-based
counterpart of the convolutional EBD construction under MSE
\citep[Appendix~D]{erdoganerror}.

Let $\mathbf{H}^{(k)} \in \mathbb{R}^{P^{(k)} \times M^{(k)} \times N^{(k)}}$
denote the output of the $k$th convolutional layer, where $P^{(k)}$ is the
number of channels and $M^{(k)} \times N^{(k)}$ is the spatial map size.
For channel $p$, the filter tensor is
$\mathbf{W}^{(k,p)} \in \mathbb{R}^{P^{(k-1)} \times \Omega^{(k)} \times \Omega^{(k)}}$
with bias $b^{(k,p)} \in \mathbb{R}$, and the convolutional layer
output is
\begin{equation*}
\mathbf{H}^{(k,p)} = f\!\bigl(\mathcal{U}^{(k,p)}\bigr),
\qquad
\mathcal{U}^{(k,p)} = \bigl(\mathbf{H}^{(k-1)} \ast \mathbf{W}^{(k,p)}\bigr) + b^{(k,p)},
\end{equation*}
where $\ast$ denotes the channel-summed spatial cross-correlation used by
standard CNN implementations, with the kernel left unflipped.

\paragraph{Score-broadcast objective.}
Let $\boldsymbol{\varepsilon}\in\mathbb{R}^{D_{\varepsilon}}$ denote the
signal broadcast to the hidden convolutional layers. For plain SBD,
$\boldsymbol{\varepsilon}=\boldsymbol{\delta}$ and
$D_{\varepsilon}=D_{\mathrm{out}}$; with $M$-block score expansion,
$\boldsymbol{\varepsilon}=\widetilde{\boldsymbol{\delta}}$ and
$D_{\varepsilon}=M D_{\mathrm{out}}$ (Appendix~\ref{app:score_expansion_appendix}).
For each channel $p$ and each spatial location $(r,s)$, the feature-by-broadcast
cross-correlation is
\begin{equation*}
\mathbf{R}_{\mathbf{g}\boldsymbol{\varepsilon}}^{(k,p)}[r,s,q]
= \mathbb{E}\bigl[\,g^{(k)}\!\bigl(\mathbf{H}^{(k,p)}[r,s]\bigr)\,\varepsilon_q\,\bigr],
\qquad q = 1,\ldots,D_{\varepsilon}.
\end{equation*}
The conditional-mean-zero property
$\mathbb{E}[\boldsymbol{\varepsilon}^\star \mid X] = \mathbf{0}$ established
for both the raw and expanded score signals implies that this
cross-correlation vanishes at the population optimum. The per-channel
score-broadcast objective at minibatch $m$ is
\begin{equation*}
\mathcal{J}^{(k,p)}_{\mathrm{Score}}[m]
= \frac{1}{2} \sum_{q=1}^{D_{\varepsilon}}
\bigl\|\widehat{\mathbf{R}}_{\mathbf{g}\boldsymbol{\varepsilon}}^{(k,p)}[m,:,:,q]\bigr\|_F^2,
\end{equation*}
and the layer objective is the sum over $p=1,\ldots,P^{(k)}$. The
exponentially weighted correlation estimate is explicitly
\begin{align*}
&\widehat{\mathbf{R}}_{\mathbf{g}\boldsymbol{\varepsilon}}^{(k,p)}
[m,r,s,q]\\
&\quad=
\lambda\widehat{\mathbf{R}}_{\mathbf{g}\boldsymbol{\varepsilon}}^{(k,p)}
[m-1,r,s,q]
+\frac{1-\lambda}{B}
\sum_{n=mB+1}^{(m+1)B}
g^{(k)}\!\left(\mathbf{H}^{(k,p)}[n,r,s]\right)\varepsilon_q[n].
\end{align*}

\paragraph{Weight update.}
As in the MLP case, we omit the derivative of the broadcast score through
deeper layers and retain only the local dependence of
$g^{(k)}(\mathbf{H}^{(k,p)})$ on the filter coefficients. Define the
per-sample postsynaptic modulator
\begin{equation*}
\boldsymbol{\phi}[n,p,:,:]
= \sum_{q=1}^{D_{\varepsilon}} \varepsilon_q[n] \,\cdot\,
\Bigl(
\widehat{\mathbf{R}}_{\mathbf{g}\boldsymbol{\varepsilon}}^{(k,p)}[m,:,:,q]
\odot {g^{(k)}}'\!\bigl(\mathbf{H}^{(k,p)}[n,:,:]\bigr)
\odot f'\!\bigl(\mathcal{U}^{(k,p)}[n,:,:]\bigr)
\Bigr).
\end{equation*}
Let $s_k$, $d_k$, and $\rho_k$ denote the stride, dilation, and padding of
layer $k$, respectively. Under the cross-correlation convention above and
using zero-based spatial and kernel indices, the required local derivative is
\begin{align*}
&\frac{\partial \mathbf{H}^{(k,p)}[n,r,c]}
{\partial \mathbf{W}^{(k,p)}_{h,i,j}}\\
&\quad=
f'\!\left(\mathcal{U}^{(k,p)}[n,r,c]\right)
\mathbf{H}^{(k-1,h)}
\bigl[n,\,s_k r+d_k i-\rho_k,\,s_k c+d_k j-\rho_k\bigr].
\end{align*}
Substituting this expression into the correlation-objective derivative gives
the filter increment
\begin{equation*}
\Delta \mathbf{W}^{(k,p)}_{h,i,j}
= \frac{\zeta}{B}
\sum_{n=mB+1}^{(m+1)B}\sum_{r,c}
\boldsymbol{\phi}[n,p,r,c]\,
\mathbf{H}^{(k-1,h)}
\bigl[n,\,s_k r+d_k i-\rho_k,\,s_k c+d_k j-\rho_k\bigr],
\end{equation*}
where the signed coefficient $\zeta$ absorbs the descent sign and the factor
$1-\lambda$, as in the previous subsection. Locations outside the input map
are interpreted according to the layer's
padding rule (and are zero under zero padding). Thus, for unit stride and
dilation with same padding, the presynaptic factor paired with output location
$(r,c)$ and kernel offset $(i,j)$ is precisely
$\mathbf{H}^{(k-1,h)}[n,r+i-\rho_k,c+j-\rho_k]$.
The corresponding bias increment is
\begin{equation*}
\Delta b^{(k,p)}
= \frac{\zeta}{B}
\sum_{n=mB+1}^{(m+1)B} \sum_{r,s} \boldsymbol{\phi}[n,p,r,s].
\end{equation*}
Setting $\boldsymbol{\varepsilon}$ to the MSE output error recovers the
convolutional EBD update \citep[Appendix~D.1]{erdoganerror}.

\paragraph{Three-factor structure.}
The convolutional update preserves the three-factor structure of the MLP
case: for each filter coefficient, the presynaptic factor is
$\mathbf{H}^{(k-1,h)}$ evaluated at the stride-, dilation-, and
padding-adjusted input coordinate defined above;
$\bigl(g^{(k)}\bigr)'\!\odot\,f'$ is the postsynaptic sensitivity, and
the score-projected term
$\sum_q \varepsilon_q \cdot \widehat{\mathbf{R}}_{\mathbf{g}\boldsymbol{\varepsilon}}^{(k,p)}$
is the broadcast modulator. The same three-factor structure depicted in
Figure~\ref{fig:sbd-three-factor} therefore applies to convolutional layers,
with the convolution operator replacing the matrix product of the MLP case.

\paragraph{Auxiliary regularization.}
The convolutional weight-entropy regularizer of
\citet[Appendix~D.1.2]{erdoganerror} is disabled in the reported CNN runs
($c^{(k)}_{\mathrm{cov}}=0$). Active fully connected regularizers are
specified in Appendices~\ref{app:cemsen-sbd} and~\ref{app:tinet}.

\subsection{Update complexity}
\label{app:ce_sbd_complexity}

We summarize the per-minibatch update cost of the rule in
Table~\ref{tab:ce_sbd_algo} and compare it to backpropagation (BP) and Direct
Feedback Alignment (DFA) under a fully-connected layer model. Let $L$ denote
the number of layers, $N^{(k)}$ the width of layer $k$ (so $N^{(0)}$ is the
input dimension and $N^{(L)}=D$ the number of classes), $B$ the minibatch
size, and $D$ the output-score dimension. The hidden-layer broadcast
dimension is $D_{\varepsilon}=D$ for plain SBD and
$D_{\varepsilon}=MD$ for the $M$-block score-expansion variant of
Section~\ref{sec:score_expansion}.

The forward pass is shared by all three methods and costs
$\mathcal{O}\!\bigl(B\sum_{k=1}^{L}N^{(k)}N^{(k-1)}\bigr)$. The
local outer-product weight update at layer $k$, of cost
$\mathcal{O}\!\bigl(B\,N^{(k)}N^{(k-1)}\bigr)$, is also identical in form
across the three methods. We compare the cost of producing the postsynaptic
modulator, excluding shared forward/weight-gradient operations and
experiment-specific auxiliary regularization.

\paragraph{BP.}
The error signal $\boldsymbol{\delta}^{(k)}$ at hidden layer $k$ is obtained
by propagating $\boldsymbol{\delta}^{(k+1)}$ through $\mathbf{W}^{(k+1)\,T}$,
costing $\mathcal{O}\!\bigl(B\,N^{(k)}N^{(k+1)}\bigr)$ per layer. The total
hidden-signal propagation cost is therefore
$\mathcal{O}\!\bigl(B\sum_{k=1}^{L-1}N^{(k)}N^{(k+1)}\bigr)$, comparable to
the forward pass, and this step requires weight transport.

\paragraph{DFA.}
The output error of dimension $D$ is broadcast to layer $k$ through a
\emph{fixed random} feedback matrix $\mathbf{B}^{(k)}\!\in\!\mathbb{R}^{N^{(k)}\times D}$,
costing $\mathcal{O}\!\bigl(B\,N^{(k)}D\bigr)$ per layer with $\mathcal{O}\!\bigl(N^{(k)}D\bigr)$
storage. No backward chain is computed.

\paragraph{SBD.}
The output score, also of dimension $D$, is broadcast through the
\emph{adaptive} correlation matrix $\widehat{\mathbf{R}}_{\mathrm{CE}}^{(k)}\!\in\!\mathbb{R}^{N^{(k)}\times D}$.
Two operations contribute at layer $k$: the EMA correlation update
(Step~2), which is an outer product of cost
$\mathcal{O}\!\bigl(B\,N^{(k)}D\bigr)$, and the projection
$\mathbf{q}_{\mathrm{CE}}^{(k)}=\widehat{\mathbf{R}}_{\mathrm{CE}}^{(k)}\boldsymbol{\delta}_{\mathrm{CE}}$
(Step~3), of cost $\mathcal{O}\!\bigl(B\,N^{(k)}D\bigr)$. Storage of
$\widehat{\mathbf{R}}_{\mathrm{CE}}^{(k)}$ is $\mathcal{O}\!\bigl(N^{(k)}D\bigr)$,
matching DFA.

\paragraph{Comparison.}
Per layer, the modulator cost is
$\mathcal{O}\!\bigl(B\,N^{(k)}N^{(k+1)}\bigr)$ for BP versus
$\mathcal{O}\!\bigl(B\,N^{(k)}D\bigr)$ for DFA and SBD. When
$D\ll N^{(k+1)}$, both broadcast methods have lower modulator cost.
SBD and DFA share the same asymptotic
complexity, but SBD performs roughly twice the work of DFA per layer (one
EMA outer product plus one projection, versus DFA's single projection through
a fixed matrix); in exchange, $\widehat{\mathbf{R}}_{\mathrm{CE}}^{(k)}$ is
\emph{learned}. Replacing $D$ by $D_{\varepsilon}=MD$ in the SBD
expressions above gives the cost of the score-expansion variant of
Section~\ref{sec:score_expansion}, which scales linearly in the expansion
factor $M$ while preserving the orthogonality framework.

\paragraph{Convolutional layers.}
The same analysis extends to the convolutional layers of
Appendix~\ref{app:cnn_sbd_derivation}. Let $P^{(k)}$ be the number of
output channels at layer $k$, $\Omega^{(k)}$ the kernel size, and
$M^{(k)} \times N^{(k)}$ the spatial output size. The shared forward
convolution and local filter-gradient contraction each cost
$\mathcal{O}\!\bigl(B\,P^{(k)}P^{(k-1)}(\Omega^{(k)})^2 M^{(k)} N^{(k)}\bigr)$
per layer. The SBD
modulator at convolutional layer $k$ requires maintaining a per-channel
feature-by-broadcast correlation tensor
$\widehat{\mathbf{R}}_{\mathbf{g}\boldsymbol{\varepsilon}}^{(k,p)} \in \mathbb{R}^{M^{(k)} \times N^{(k)} \times D_{\varepsilon}}$,
where $D_{\varepsilon}=D$ for plain SBD and $D_{\varepsilon}=MD$ for
$M$-block score expansion. Its total size is
$\mathcal{O}\!\bigl(P^{(k)} D_{\varepsilon} M^{(k)} N^{(k)}\bigr)$, with the
EMA update and the projection each costing
$\mathcal{O}\!\bigl(B\,P^{(k)} D_{\varepsilon} M^{(k)} N^{(k)}\bigr)$.
For a following convolution, BP instead propagates its signal through
layer $k+1$, with cost
\[
\mathcal O\!\left(BP^{(k)}P^{(k+1)}(\Omega^{(k+1)})^2
M^{(k+1)}N^{(k+1)}\right).
\]
The broadcast cost scales with $D_{\varepsilon}$ at the current spatial
resolution, whereas BP scales with the next layer's fan-out, kernel area,
and output resolution. Score expansion increases correlation storage and
update/projection work linearly in $M$; modulator construction is a separate
cost.

These asymptotic operation counts are independent of a particular software
stack and do not directly predict realized runtime or energy consumption.
Empirical convergence and GPU-energy measurements for the Tiny ImageNet
implementation are reported with the accuracy curves in
Appendix~\ref{app:tinet-accuracy-curves}.

\section{Score vector expansion: general theory and cross entropy example}
\label{app:score_expansion_appendix}

This appendix provides the full details behind the main-text summary in
Section~\ref{sec:score_expansion}. We first develop the score vector expansion
for a general differentiable loss whose population-optimal score has
conditional mean zero, and then specialize the construction to cross entropy.

\subsection{General formulation}
\label{app:score_expansion_general}

Let $\mathcal{L}(\mathbf{y},\mathbf{a})$ be a differentiable loss with output
score $\boldsymbol{\delta}(\mathbf{y},\mathbf{a})=\nabla_{\mathbf{a}}\mathcal{L}(\mathbf{y},\mathbf{a})\in\mathbb{R}^{D_{\mathrm{out}}}$,
and assume that the population-optimal predictor $\mathbf{a}^{\star}(X)$
satisfies the conditional mean-zero property
\begin{equation*}
\mathbb{E}\bigl[\boldsymbol{\delta}(\mathbf{Y},\mathbf{a}^{\star}(X))\,\big|\,X\bigr]=\mathbf{0},
\end{equation*}
so that the orthogonality machinery of
Theorem~\ref{thm:general_score_orthogonality} applies. The layerwise
correlation matrix
\begin{equation*}
\mathbf{R}^{(k)}=\mathbb{E}\bigl[g^{(k)}(\mathbf{h}^{(k)})\boldsymbol{\delta}^T\bigr]\in\mathbb{R}^{d_k\times D_{\mathrm{out}}}
\end{equation*}
has column rank at most $D_{\mathrm{out}}$, regardless of the hidden-layer
width $d_k$. When $d_k$ greatly exceeds $D_{\mathrm{out}}$, this caps the
number of independent decorrelation constraints and motivates enlarging the
broadcast signal.

\paragraph{Modulated scores preserve conditional mean-zero.}
Let $\boldsymbol{\phi}:\mathcal{X}\to\mathbb{R}^{D_{\mathrm{out}}}$ be any measurable vector-valued modulator of the input, and define the \emph{modulated score} at the population optimum by
\begin{equation*}
\boldsymbol{\eta}^{\boldsymbol{\phi},\star}(X)
\;:=\;
\boldsymbol{\phi}(X)\odot
\boldsymbol{\delta}\bigl(\mathbf{Y},\mathbf{a}^{\star}(X)\bigr),
\end{equation*}
where $\odot$ denotes the Hadamard product. Assume the modulated score is
integrable. Because $\boldsymbol{\phi}(X)$ is measurable with respect to $X$,
it can be pulled out of the conditional expectation. Hence
\begin{equation*}
\mathbb{E}\!\left[\boldsymbol{\eta}^{\boldsymbol{\phi},\star}(X)\,\big|\,X\right]
=\boldsymbol{\phi}(X)\odot
\mathbb{E}\!\left[\boldsymbol{\delta}(\mathbf{Y},\mathbf{a}^{\star}(X))\,\big|\,X\right]
=\mathbf{0}.
\end{equation*}
Applying the tower-property orthogonality lemma
(Lemma~\ref{lem:tower_orthogonality}), we conclude
\begin{equation*}
\mathbb{E}\!\left[\boldsymbol{\eta}^{\boldsymbol{\phi},\star}(X)g(X)^T\right]=\mathbf{0}
\quad\text{whenever }\mathbb E\|\boldsymbol\eta^{\boldsymbol\phi,\star}(X)g(X)^T\|<\infty.
\end{equation*}
Under these integrability conditions, any $X$-measurable Hadamard modulator
preserves population-optimal orthogonality. Modulators can use the raw input,
intermediate activations, or output predictions. Forward randomness independent
of $Y$ conditional on $X$ can be included by conditioning jointly on $X$ and
that randomness.

\paragraph{Broadcast expansion by stacking modulators.}
Given $M$ $X$-measurable modulators $\boldsymbol{\phi}_{1},\ldots,\boldsymbol{\phi}_{M}$, let $D_{\varepsilon}:=M D_{\mathrm{out}}$ and define the \emph{expanded score} during training by
\begin{equation*}
\tilde{\boldsymbol{\delta}}(X)\;:=\;
\begin{bmatrix}
\boldsymbol{\phi}_{1}(X)\odot\boldsymbol{\delta}(\mathbf{Y},\mathbf{a}(X))\\[2pt]
\boldsymbol{\phi}_{2}(X)\odot\boldsymbol{\delta}(\mathbf{Y},\mathbf{a}(X))\\[-2pt]
\vdots\\[-1pt]
\boldsymbol{\phi}_{M}(X)\odot\boldsymbol{\delta}(\mathbf{Y},\mathbf{a}(X))
\end{bmatrix}\in\mathbb{R}^{D_{\varepsilon}},
\end{equation*}
obtained by stacking the $M$ modulated scores. Replacing $\mathbf{a}(X)$ by the population optimum $\mathbf{a}^{\star}(X)$, the argument above shows that each block of length $D_{\mathrm{out}}$ satisfies the conditional mean-zero property, and hence
\begin{equation}
\mathbb{E}\!\left[\tilde{\boldsymbol{\delta}}^{\star}(X)\,\big|\,X\right]=\mathbf{0},
\qquad
\mathbb{E}\!\left[\tilde{\boldsymbol{\delta}}^{\star}(X)g^{(k)}(\mathbf{h}^{(k)}(X))^T\right]=\mathbf{0}.
\label{eq:expanded_orthogonality_general}
\end{equation}
The expanded correlation matrix
\begin{equation*}
\tilde{\mathbf{R}}^{(k)}\;=\;
\mathbb{E}\!\left[g^{(k)}(\mathbf{h}^{(k)})\,\tilde{\boldsymbol{\delta}}^{T}\right]\in\mathbb{R}^{d_{k}\times D_{\varepsilon}}
\end{equation*}
now has column rank at most $D_{\varepsilon}=M D_{\mathrm{out}}$, providing the layerwise
decorrelation objective
$\tilde{\mathcal{J}}^{(k)}=\tfrac12\|\tilde{\mathbf{R}}^{(k)}\|_{F}^{2}$ with
up to $M$ times more independent directions along which to suppress
dependence between the hidden representation and the broadcast score, at the
cost of a linear increase in stored correlation state.

\paragraph{Algorithmic formulation.}
The expanded-score procedure is identical to plain SBD except that the
$D_{\mathrm{out}}$-dimensional score and its batch matrix are replaced by
their $D_{\varepsilon}$-dimensional counterparts:
\begin{align*}
\tilde{\boldsymbol{\delta}}[n]
&=\bigl[\boldsymbol{\phi}_{1}(\mathbf{x}[n])\odot\boldsymbol{\delta}(\mathbf{y}[n],\mathbf{a}(\mathbf{x}[n]));\;\ldots;\;\boldsymbol{\phi}_{M}(\mathbf{x}[n])\odot\boldsymbol{\delta}(\mathbf{y}[n],\mathbf{a}(\mathbf{x}[n]))\bigr],\\
\tilde{\mathbf{\Delta}}[m]
&=\begin{bmatrix}\tilde{\boldsymbol{\delta}}[mB+1] & \cdots & \tilde{\boldsymbol{\delta}}[mB+B]\end{bmatrix}\in\mathbb{R}^{D_{\varepsilon}\times B},\\
\widehat{\widetilde{\mathbf{R}}}^{(k)}[m]
&=\lambda\,\widehat{\widetilde{\mathbf{R}}}^{(k)}[m{-}1]+\frac{1-\lambda}{B}\,\mathbf{G}^{(k)}[m]\,\tilde{\mathbf{\Delta}}[m]^{T},\\
\tilde{\mathbf{q}}^{(k)}[mB+n]
&=\widehat{\widetilde{\mathbf{R}}}^{(k)}[m]\,\tilde{\boldsymbol{\delta}}[mB+n], \text{ for } n=1, \ldots, B,
\end{align*}
\begin{equation}
\label{eq:batch_update_W}
\Delta W_{ij}^{(k)}[m]=\frac{\zeta}{B}\sum_{l=mB+1}^{mB+B}\,g_{i}^{\prime(k)}\!\bigl(h_{i}^{(k)}[l]\bigr)\,f^{\prime(k)}\!\bigl(u_{i}^{(k)}[l]\bigr)\,\tilde{q}_{i}^{(k)}[l]\,h_{j}^{(k-1)}[l].
\end{equation}
The output layer continues to use the raw $D_{\mathrm{out}}$-dimensional
score for its standard gradient update.

\subsection{Cross entropy instantiation}
\label{app:ce_score_expansion}

For cross entropy classification, the output dimension is $D_{\mathrm{out}}=D$
(the number of classes) and the output score is
$\boldsymbol{\delta}_{\mathrm{CE}}=\mathbf{p}-\mathbf{y}$, with
$\mathbf{p}^{\star}(X)=\mathbf{q}(X)$ as the population minimizer. The
conditional mean-zero property required by the general construction is
established in Theorem~\ref{thm:ce_conditional_mean_zero}, so the modulated and
expanded score machinery above applies to any deterministic $X$-measurable
modulator. In our experiments, the modulators are built from the predictive
distribution $\mathbf{p}(X)$, which is itself a deterministic function of $X$
for fixed network parameters.

For classification networks in which the hidden width $d_k$ greatly exceeds
$D$; for example, CIFAR-10 with $D=10$ and $d_k\in\{10^{3},10^{5}\}$ in the
fully-connected and convolutional layers, respectively, the rank bottleneck
in $\mathbf{R}_{\mathrm{CE}}^{(k)}\in\mathbb{R}^{d_k\times D}$ is severe, and
score expansion provides a natural remedy.

\paragraph{A concrete $3D$-dimensional instance.}
For the CNN experiments in Section~\ref{sec:numerical_experiments} we adopt a
minimal $M=3$ instance built from two natural $X$-measurable modulators
obtained from $\mathbf{p}(X)$, together with the identity. Let
$\mathbf{1}_{D}\in\mathbb{R}^{D}$ denote the all-ones vector, and let
$\mathrm{roll}_{k}(\mathbf{v})$ denote the cyclic shift of $\mathbf{v}$ by $k$
positions, that is
$[\mathrm{roll}_{k}(\mathbf{v})]_{d}=v_{((d-k-1)\bmod D)+1}$. The three
modulators are
\begin{equation*}
\boldsymbol{\phi}_{1}(X)=\mathbf{1}_{D},\qquad
\boldsymbol{\phi}_{2}(X)=\mathbf{p}(X),\qquad
\boldsymbol{\phi}_{3}(X)=\mathrm{roll}_{k}(\mathbf{p}(X)),
\end{equation*}
where $k=5$ in the reported CNN configuration. Class indices are arbitrary, so the
shift does not represent semantic proximity. The
resulting expanded score is
\begin{equation*}
\tilde{\boldsymbol{\delta}}_{\mathrm{CE}}\;=\;
\begin{bmatrix}
\boldsymbol{\delta}_{\mathrm{CE}}\\[2pt]
\mathbf{p}\odot\boldsymbol{\delta}_{\mathrm{CE}}\\[2pt]
\mathrm{roll}_{k}(\mathbf{p})\odot\boldsymbol{\delta}_{\mathrm{CE}}
\end{bmatrix}\in\mathbb{R}^{3D}.
\end{equation*}
Block one is the raw cross-entropy score. Block two weights each coordinate
by its predicted class probability, and block three uses a cyclically shifted
probability vector. These input-dependent weightings can supply distinct
score--feature correlations. All three modulators are deterministic
functions of $X$ through the forward pass, so
\eqref{eq:expanded_orthogonality_general} applies and the expanded
decorrelation condition remains an exact orthogonality relation at the
population optimum.

\paragraph{Algorithm.}
Specializing the general expanded-SBD procedure above to the cross entropy
case yields the algorithm in Table~\ref{tab:ce_sbd_expanded}. The only
modifications from Appendix~\ref{app:ce_sbd_algorithm}, Table~\ref{tab:ce_sbd_algo} are Step~2, in which
$\mathbf{\Delta}_{\mathrm{CE}}[m]$ is replaced by
$\tilde{\mathbf{\Delta}}_{\mathrm{CE}}[m]\in\mathbb{R}^{D_{\varepsilon}\times B}$, where $D_{\varepsilon}=3D$, and
Step~3, in which the projection uses
$\widehat{\widetilde{\mathbf{R}}}_{\mathrm{CE}}^{(k)}$ and
$\tilde{\boldsymbol{\delta}}_{\mathrm{CE}}$.

\begin{table}[h]
\caption{Cross entropy SBD with $3D$-dimensional score expansion. Differences from
Appendix~\ref{app:ce_sbd_algorithm}, Table~\ref{tab:ce_sbd_algo} are shown in italics.}
\label{tab:ce_sbd_expanded}
\centering
\begin{tabular}{p{0.12\linewidth}p{0.78\linewidth}}
\toprule
Step & Operation \\
\midrule
1 & Forward pass to obtain $\mathbf{h}^{(k)}$, $\mathbf{a}$, $\mathbf{p}=\softmax(\mathbf{a})$, and the score $\boldsymbol{\delta}_{\mathrm{CE}}=\mathbf{p}-\mathbf{y}$.\\
2 & \emph{Form the expanded score $\tilde{\boldsymbol{\delta}}_{\mathrm{CE}}=[\boldsymbol{\delta}_{\mathrm{CE}};\,\mathbf{p}\odot\boldsymbol{\delta}_{\mathrm{CE}};\,\mathrm{roll}_{k}(\mathbf{p})\odot\boldsymbol{\delta}_{\mathrm{CE}}]\in\mathbb{R}^{D_{\varepsilon}}$, where $D_{\varepsilon}=3D$, and stack it into the batch matrix $\tilde{\mathbf{\Delta}}_{\mathrm{CE}}[m]\in\mathbb{R}^{D_{\varepsilon}\times B}$. Update the expanded correlation estimate $\widehat{\widetilde{\mathbf{R}}}_{\mathrm{CE}}^{(k)}[m]$.}\\
3 & \emph{Project the broadcast score to hidden layer $k$ by $\tilde{\mathbf{q}}_{\mathrm{CE}}^{(k)}[n]=\widehat{\widetilde{\mathbf{R}}}_{\mathrm{CE}}^{(k)}[m]\,\tilde{\boldsymbol{\delta}}_{\mathrm{CE}}[n]$ for $n=mB+1, \ldots, (m+1)B$.}\\
4 & For $B=1$, the weight increment is\\[-2pt]
  & $\Delta W_{ij}^{(k)}[n]=\zeta\,g_i^{\prime(k)}(h_i^{(k)}[n])\,f^{\prime(k)}(u_i^{(k)}[n])\,\tilde q_{\mathrm{CE},i}^{(k)}[n]\,h_j^{(k-1)}[n]$,\\
  & with analogous bias updates. For $B>1$, use Eq.~\eqref{eq:batch_update_W}.\\
5 & Update the output layer with the ordinary $D$-dimensional cross-entropy gradient. \\
\bottomrule
\end{tabular}
\end{table}
\section{CIFAR-10 experimental setup}
\label{app:cemsen}

This appendix documents the codebase used for the CIFAR-10 broadcast-learning
experiments reported in the main text. The expanded SBD configuration is
implemented by the per-seed launch scripts \texttt{run\_cemseN\_s*.sh}, which
call \texttt{CIFAR10\_CEMSEN.py} with cross-entropy loss, \texttt{method=sbd},
and \texttt{err\_expand=2}. This setting uses the $3D$-dimensional score expansion, so that the broadcast dimension is $3D=30$ for CIFAR-10. The main-text result
is averaged over five independent seed runs using the 5-layer CNN described in
Section~\ref{app:cemsen-architecture}. All CIFAR-10 experiments run in \texttt{float32} on a single NVIDIA V100 (32GB) GPU using PyTorch. For the reported baseline-width CIFAR-10 experiments, the approximate times to complete $200$ epochs were $75$, $91$, $96$, and $101$ minutes for BP, DFA, SBD, and SBD with score expansion, respectively. For the reported $4\times$-width CIFAR-10 experiments, the approximate times to complete $200$ epochs were $170$ and $420$ minutes for BP and SBD with score expansion, respectively.

For CIFAR-10, we use the standard dataset partition of $50{,}000$ training
images and $10{,}000$ test images, and all reported CIFAR-10 accuracies are top-1
accuracies on the official test split. The architecture follows the original
EBD CNN setup, with cross entropy replacing MSE; method-specific optimization
settings are described below. For the expanded-SBD variant, Appendix~\ref{app:score-expansion-ablation}
documents the chosen score-expansion modulator set together with the
multi-seed ablation comparing it with alternative expansions.

PyTorch~\citep{paszke2019pytorch} is distributed under the BSD-3-Clause
license, while the CIFAR-10 dataset~\citep{krizhevsky2009cifar} is publicly
released by the University of Toronto for research use without a formal
license declaration on its official distribution page; both are used here
in accordance with their stated terms.

\subsection{Score broadcast and decorrelation}
\label{app:cemsen-sbd}

SBD replaces the global backpropagation chain with local, per-layer learning
rules driven by a broadcast score signal. At each mini-batch, the forward pass
produces pre-softmax logits
$\mathbf{a} \in \mathbb{R}^{D}$ and probabilities
$
\mathbf{p} = \softmax(\mathbf{a}/T),
$
where $D=10$ for CIFAR-10 and $T$ is the implementation variable
$\mathrm{\texttt{SOFTMAX\_TEMP}}$. In cross-entropy mode, the implementation
uses the broadcast score
\begin{equation*}
\boldsymbol{\delta}
=
\alpha(\mathbf{p}-\mathbf{y})
\in \mathbb{R}^{D},
\qquad
\alpha = \mathrm{\texttt{ERR\_SCALE}},
\end{equation*}
where $\alpha$ is the score scaling and $\mathbf{y}$ is the one-hot label vector.
The reported runs use $T=\alpha=1$, so this is the cross-entropy derivative
with respect to the logits.

For hidden-layer SBD updates, this $D$-dimensional signal is expanded into a
higher-dimensional broadcast vector
$\boldsymbol{\varepsilon}\in\mathbb{R}^{D_{\varepsilon}}$ by concatenating
blocks of the form
$\boldsymbol{\phi}_{\ell}(\mathbf{x})\odot\boldsymbol{\delta}$, where each
$\boldsymbol{\phi}_{\ell}$ is an $X$-measurable deterministic function of the input $\mathbf{x}$ (in the reported setting, through $\mathbf{p}$ and $\mathbf{a}$). In the main expanded-SBD
configuration, implemented with \texttt{err\_expand=2}, the hidden-layer
broadcast vector is
\begin{equation}
\boldsymbol{\varepsilon}(\mathbf{x},\mathbf{y})
=
\bigl[
\boldsymbol{\delta};\;
\mathbf{p}\odot\boldsymbol{\delta};\;
\mathrm{roll}_{k}(\mathbf{p})\odot\boldsymbol{\delta}
\bigr]
\in \mathbb{R}^{D_{\varepsilon}},
\qquad D_{\varepsilon}=3D=30.
\label{eq:err-exp}
\end{equation}
Thus the output layer receives the raw 10-dimensional score
$\boldsymbol{\delta}$, while the hidden layers receive the expanded
$3D$-dimensional broadcast signal, with $k=5$ in the reported configuration. Since each modulator is deterministic conditional
on $X$, the conditional mean-zero property is preserved at the Bayes optimum:
\[
\mathbb{E}\!\left[
\boldsymbol{\phi}_{\ell}(X)\odot\boldsymbol{\delta}^{\star}
\,\middle|\, X
\right]
=
\boldsymbol{\phi}_{\ell}(X)\odot
\mathbb{E}\!\left[\boldsymbol{\delta}^{\star}\mid X\right]
=
0.
\]
Thus score expansion preserves the population orthogonality property used by
SBD.

For each hidden trainable layer $k$ with pre-activation
$\mathbf{u}^{(k)} = W^{(k)}\mathbf{h}^{(k-1)}$, post-activation
$\mathbf{h}^{(k)} = \mathrm{ReLU}(\mathbf{u}^{(k)})$, and incoming activation
$\mathbf{h}^{(k-1)}$, we use the feature-by-broadcast orientation adopted in
the main text,
$\widehat{\mathbf{R}}_{g\varepsilon}^{(k)}
\in\mathbb{R}^{N^{(k)}\times D_{\varepsilon}}$. If an implementation stores
the transpose for memory-layout convenience, it represents the same object as
$\widehat{\mathbf{R}}_{\varepsilon g}^{(k)}
=\widehat{\mathbf{R}}_{g\varepsilon}^{(k)T}$; all equations below use the
feature-by-broadcast orientation. The CNN uses the identity feature map
$g(\mathbf{h})=\mathbf{h}$. SBD maintains two mini-batch EMAs:
\begin{align*}
\widehat{\mathbf{R}}_{g\varepsilon}^{(k)}[m]
 &\leftarrow \lambda\,\widehat{\mathbf{R}}_{g\varepsilon}^{(k)}[m-1]
   + (1-\lambda)\cdot\tfrac{1}{B}\sum_{l=mB+1}^{mB+B}
   \mathbf{h}^{(k)}[l]\,\boldsymbol{\varepsilon}[l]^T
 \qquad \text{(cross-correlation)}
  \\
\widehat{\mathbf{R}}_{hh}^{(k)}[m]
 &\leftarrow \lambda_2\,\widehat{\mathbf{R}}_{hh}^{(k)}[m-1]
   + (1-\lambda_2)\cdot\tfrac{1}{B}\sum_{l=mB+1}^{mB+B}
   \mathbf{h}^{(k)}[l]\mathbf{h}^{(k)}[l]^T
 \quad \text{(auto-correlation; FC only)}
\end{align*}
with decay rates
$\lambda=\mathrm{\texttt{Reh\_lambda}}$ and
$\lambda_2=\mathrm{\texttt{Reh\_lambda2}}$. Both decays are annealed toward
$1$ at the start of each epoch from zero-based epoch $2$ onward according to
$\lambda\leftarrow\lambda+\rho(1-\lambda)$ with
$\rho=\mathrm{\texttt{Reh\_lambda\_drop}}$.

The tensor supplied to Adam at layer $k$ is a weighted sum of three local terms:
\begin{equation*}
G^{(k)}
\;=\;
c_{\mathrm{score}}^{(k)}\,G^{\mathrm{score}}_{W^{(k)}}
\;+\; c_{\mathrm{cov}}^{(k)}\,G^{\mathrm{cov}}_{W^{(k)}}
\;+\; c_{\ell_1}^{(k)}\,G^{\ell_1}_{W^{(k)}}.
\end{equation*}
The coefficient $c_{\mathrm{score}}^{(k)}$ denotes the weight of the
score-broadcast decorrelation term. In the implementation, this coefficient
corresponds to the \texttt{CMSE\_*} hyperparameters, retaining the naming
convention of the original MSE/EBD codebase. Adam applies its moment estimates,
coordinatewise scaling, and weight decay to this tensor before subtracting the
resulting parameter step. The three terms are:
\begin{itemize}
\item \textbf{Three-factor score-broadcast update}
\begin{equation*}
G^{\mathrm{score}}_{W^{(k)}}
\;=\;
\tfrac{1}{B}\sum_{l\in\mathcal{B}}
\bigl[\,f'(\mathbf{u}^{(k)}[l])\odot
      \bigl(\widehat{\mathbf{R}}_{g\varepsilon}^{(k)}\,\boldsymbol{\varepsilon}[l]\bigr)\,\bigr]\,
\mathbf{h}^{(k-1)}[l]^T.
\end{equation*}
This is the product of three factors: the ReLU derivative
$f'(\mathbf{u}^{(k)})$, the broadcast modulator
$\widehat{\mathbf{R}}_{g\varepsilon}^{(k)}\boldsymbol{\varepsilon}$, and the
presynaptic activation $\mathbf{h}^{(k-1)}$, accumulated across the batch
$\mathcal{B}$. The local direction uses the updated EMA, with scale set by
$c_{\mathrm{score}}^{(k)}$. Convolutions use pre-pooling activations and sum
the corresponding receptive-field products over spatial locations.
At the output layer, $\boldsymbol{\varepsilon}$ is replaced by the raw
$\boldsymbol{\delta}$, $f'$ is set to one, and
$\widehat{\mathbf{R}}_{g\varepsilon}^{(k)}$ is the identity so that the update reduces
to the standard single-layer cross entropy gradient.

\item \textbf{Activation covariance direction} (FC layers only)
\citep{erdoganerror,ozsoy2022self,bozkurt2023correlative,bozkurt2023correlatedsourcesep}
\begin{equation*}
G^{\mathrm{cov}}_{W^{(k)}}
\;=\;
-\tfrac{2}{B\,D_k}\sum_{l\in\mathcal{B}}
\bigl[\,f'(\mathbf{u}^{(k)}[l])\odot
     (\widehat{\mathbf{R}}_{hh}^{(k)} + \eta I)^{-1}
     (\mathbf{h}^{(k)}[l]+10^{-5}\mathbf{1})\,\bigr]\,
\mathbf{h}^{(k-1)}[l]^T,
\end{equation*}
motivated by the log-determinant activation objective. Here
$\eta=10^{-8}$ (the \texttt{Reh\_ini} setting) regularizes the inverse and $D_k$ is
the layer width. At the classifier, this local rule uses
$\mathbf{h}^{(k)}=\mathbf{p}$ and $f'=1$; convolutional layers use $c_{\mathrm{cov}}^{(k)}=0$
(\texttt{CCOV\_HIDDEN}).

\item \textbf{Activation $\ell_1$ sparsity}
\begin{equation*}
G^{\ell_1}_{W^{(k)}}
\;=\;
\tfrac{1}{B\,D_k}\sum_{l\in\mathcal{B}}
\mathrm{sign}(\mathbf{h}^{(k)}[l])\,\mathbf{h}^{(k-1)}[l]^T,
\end{equation*}
applied only to the hidden fully connected layer fc3 in the reported
CIFAR-10 expanded-SBD implementation. It is disabled for the convolutional
layers and is not applied to the final classifier fc4.
\end{itemize}

The scalar coefficients
$(c_{\mathrm{score}}^{(k)},c_{\mathrm{cov}}^{(k)},c_{\ell_1}^{(k)})$ are
given in Table~\ref{tab:cemsen-hparams}. The local SBD terms are
computed inside a \texttt{torch.no\_grad()} context, manually written into
\texttt{param.grad}, and applied via a standard Adam optimizer. In the SBD
runs, hidden-layer updates are supplied by these local broadcast-decorrelation
rules rather than by a global backpropagation chain.
\subsection{Network architecture}
\label{app:cemsen-architecture}

The network is a minimalist 5-layer CNN: three convolutional blocks followed
by two fully connected layers. It uses no bias terms, no BatchNorm, no
dropout, and no residual connections. Average pooling is used throughout.
This minimal choice isolates the SBD update rule from confounds introduced
by normalization layers.

\begin{table}[h]
\centering
\caption{Baseline CNN architecture. Input: CIFAR-10 images of shape
$(3, 32, 32)$. Total trainable parameters: $1{,}289{,}600$.}
\label{tab:cemsen-arch}
\small
\begin{tabular}{lllr}
\toprule
Layer & Operation & Output shape & Parameters \\
\midrule
conv1   & $5\!\times\!5$ conv, stride 1, pad 2 & $(128,\,32,\,32)$ & $9{,}600$ \\
act1    & ReLU                                  & $(128,\,32,\,32)$ & -- \\
pool1   & AvgPool $2\!\times\!2$                & $(128,\,16,\,16)$ & -- \\
conv2   & $5\!\times\!5$ conv, stride 1, pad 2 & $(64,\,16,\,16)$  & $204{,}800$ \\
act2    & ReLU                                  & $(64,\,16,\,16)$  & -- \\
pool2   & AvgPool $2\!\times\!2$                & $(64,\,8,\,8)$    & -- \\
conv3   & $2\!\times\!2$ conv, stride 2, pad 0 & $(64,\,4,\,4)$    & $16{,}384$ \\
act3    & ReLU                                  & $(64,\,4,\,4)$    & -- \\
flatten & reshape                               & $(1024,)$         & -- \\
fc3     & linear                                & $(1024,)$         & $1{,}048{,}576$ \\
act4    & ReLU                                  & $(1024,)$         & -- \\
fc4     & linear                                & $(10,)$           & $10{,}240$ \\
softmax & softmax at $T\!=\!1$                  & $(10,)$           & -- \\
\bottomrule
\end{tabular}
\end{table}

Layer widths are parameterized via the CLI flags
\texttt{--P0}, \texttt{--P1}, \texttt{--P2}, and \texttt{--fc\_hidden}
(default $128,64,64,1024$), which in turn set the channel counts of conv1,
conv2, conv3, and fc3 respectively. The fc3 input dimension
$P_2\!\cdot\!4\!\cdot\!4$ is derived automatically. Weights are initialized
from a Gaussian distribution with standard deviation
$\sqrt{2}\sqrt{1/\mathrm{\texttt{scale\_factor}}}/\sqrt{\mathrm{fan\_in}}$,
giving a scaled variant of He et al.'s fan-in initialization~\citep{he2015rectifiers} with \texttt{scale\_factor}$=6$. In the reported CIFAR-10 runs, data augmentation consists of per-channel
normalization, random cropping after 2-pixel reflection padding, and horizontal
flipping. The loader caches an initial random flip for each image and reverses
all flip choices on alternating epochs.

For the additional width-scaling experiments discussed in the main text, we
also evaluated a $4\times$-width variant of the same CNN.

\begin{table}[h]
\centering
\caption{CNN architecture at $4\times$ width. Input:
CIFAR-10 images of shape $(3, 32, 32)$. Total trainable parameters:
$20{,}395{,}520$ ($\approx 16\times$ the baseline model count).}
\label{tab:cemsen-arch-quad}
\small
\begin{tabular}{lllr}
\toprule
Layer & Operation & Output shape & Parameters \\
\midrule
conv1   & $5\!\times\!5$ conv, stride 1, pad 2 & $(512,\,32,\,32)$ & $38{,}400$ \\
act1    & ReLU                                  & $(512,\,32,\,32)$ & -- \\
pool1   & AvgPool $2\!\times\!2$                & $(512,\,16,\,16)$ & -- \\
conv2   & $5\!\times\!5$ conv, stride 1, pad 2 & $(256,\,16,\,16)$ & $3{,}276{,}800$ \\
act2    & ReLU                                  & $(256,\,16,\,16)$ & -- \\
pool2   & AvgPool $2\!\times\!2$                & $(256,\,8,\,8)$   & -- \\
conv3   & $2\!\times\!2$ conv, stride 2, pad 0 & $(256,\,4,\,4)$   & $262{,}144$ \\
act3    & ReLU                                  & $(256,\,4,\,4)$   & -- \\
flatten & reshape                               & $(4096,)$          & -- \\
fc3     & linear                                & $(4096,)$          & $16{,}777{,}216$ \\
act4    & ReLU                                  & $(4096,)$          & -- \\
fc4     & linear                                & $(10,)$            & $40{,}960$ \\
softmax & softmax at $T\!=\!1$                  & $(10,)$            & -- \\
\bottomrule
\end{tabular}
\end{table}

The wider model multiplies all three convolutional channel counts and the
hidden FC width by four, preserving spatial sizes and the ten-class output.
Most weight matrices therefore grow by a factor of sixteen, accounting for
the parameter counts in Tables~\ref{tab:cemsen-arch} and~\ref{tab:cemsen-arch-quad}.

\subsection{Training hyperparameters}
\label{app:cemsen-hparams}

Table~\ref{tab:cemsen-hparams} lists the hyperparameters used by the
expanded-SBD CIFAR-10 launch scripts \texttt{run\_cemseN\_s*.sh}. Unless
otherwise noted, expanded-SBD runs inherit these values; control runs use the
corresponding values specified in their own launch scripts.
The training loop performs $201$ passes, indexed $0$--$200$, and skips the
last full training batch as well as any incomplete batch. Learning-rate decay
begins at the start of zero-based epoch $2$. Evaluation uses unaugmented
training and test images after every pass.

At baseline width, plain SBD and DFA use $B=16$ and initial learning rate
$10^{-4}$, while expanded SBD uses $B=64$ and $4\times10^{-4}$.
BP uses $(B,\eta_0)=(16,10^{-4})$ with decay $0.96$ at baseline width and
$(64,4\times10^{-4})$ with decay $0.97$ at $4\times$ width.
The wider expanded-SBD run retains the baseline expanded-SBD recipe.

\begin{table}[h]
\centering
\caption{CEMSEN training hyperparameters corresponding to the expanded-SBD
per-seed launch scripts \texttt{run\_cemseN\_s*.sh}. Column \emph{flag} is
the CLI name in the launch script; column \emph{Symbol} names the variable
used in the equations of Section~\ref{app:cemsen-sbd}.}
\label{tab:cemsen-hparams}
\small
\begin{tabular}{llll}
\toprule
Group & Flag & Symbol & Value \\
\midrule
\multirow{5}{*}{\emph{General}}
 & \texttt{--logger\_name}  & seed               & 5 independent seeds \\
 & \texttt{--n\_epochs}     & --                 & $201$ \\
 & \texttt{--batch\_size}   & $B$                & $64$ \\
 & \texttt{--loss\_type}    & --                 & cross entropy \\
 & \texttt{--augmentation}  & --                 & $1$ (crop $+$ flip) \\
\midrule
\multirow{5}{*}{\emph{Method}}
 & \texttt{--method}        & --                 & \texttt{sbd} \\
 & \texttt{--err\_expand}   & --                 & $2$ (30 coordinates) \\
 & \texttt{--ERR\_SCALE}    & $\alpha$           & $1.0$ \\
 & \texttt{--SOFTMAX\_TEMP} & $T$                & $1.0$ \\
 & effective $D_{\varepsilon}$ & $D_{\varepsilon}$ & $30$ \\
\midrule
\multirow{6}{*}{\emph{Optimizer}}
 & \texttt{--lr}               & $\eta_0$        & $4\!\times\!10^{-4}$ \\
 & \texttt{--lr\_drop\_rate}   & --              & $0.97$ \\
 & \texttt{--lr\_drop\_every}  & --              & $1$ epoch \\
 & \texttt{--weight\_decay}    & --              & $1\!\times\!10^{-5}$ \\
 & Adam $\beta_1$, $\beta_2$   & --              & $0.9$, $0.999$ \\
 & precision                   & --              & \texttt{float32} \\
\midrule
\multirow{8}{*}{\emph{SBD coefficients}}
 & \texttt{--CMSE\_OUT}    & $c_{\mathrm{score}}^{(\mathrm{fc4})}$       & $10$ \\
 & \texttt{--CMSE\_OUT2}   & $c_{\mathrm{score}}^{(\mathrm{fc3})}$       & $0.1$ \\
 & \texttt{--CMSE\_HIDDEN} & $c_{\mathrm{score}}^{(\mathrm{conv})}$      & $0.1$ \\
 & \texttt{--CCOV\_OUT}    & $c_{\mathrm{cov}}^{(\mathrm{fc4})}$       & $1\!\times\!10^{-7}$ \\
 & \texttt{--CCOV\_OUT2}   & $c_{\mathrm{cov}}^{(\mathrm{fc3})}$       & $1\!\times\!10^{-7}$ \\
 & \texttt{--CCOV\_HIDDEN} & $c_{\mathrm{cov}}^{(\mathrm{conv})}$      & $0$ \\
 & \texttt{--CL1\_OUT}     & $c_{\ell_1}^{(\mathrm{fc3})}$             & $1\!\times\!10^{-11}$ \\
 & \texttt{--CL1\_HIDDEN}  & $c_{\ell_1}^{(\mathrm{hidden})}$          & $0$ \\
\midrule
\multirow{7}{*}{\emph{Covariance EMA}}
 & \texttt{--Reh\_lambda}            & $\lambda$          & $0.99999$ \\
 & \texttt{--Reh\_lambda2}           & $\lambda_2$        & $0.99999$ \\
 & \texttt{--Reh\_lambda\_drop}      & $\rho$             & $0.04$ \\
 & \texttt{--Reh\_lambda\_drop\_every} & --               & $1$ epoch \\
 & \texttt{--Reh\_gain}              & $\mathrm{std}(\widehat{\mathbf{R}}_{g\varepsilon}^{(\mathrm{conv})})$ & $0.01$ \\
 & \texttt{--Reh\_gain\_lin}         & $\mathrm{std}(\widehat{\mathbf{R}}_{g\varepsilon}^{(\mathrm{FC})})$   & $0.01$ \\
 & \texttt{--Reh\_ini}               & init diagonal of $\widehat{\mathbf{R}}_{hh}$ & $1\!\times\!10^{-8}$ \\
\midrule
\multirow{4}{*}{\emph{Architecture}}
 & \texttt{--P0}, \texttt{--P1}, \texttt{--P2} & channel counts & $128,\,64,\,64$ \\
 & \texttt{--fc\_hidden}   & fc3 width $F_h$   & $1024$ \\
 & \texttt{--use\_bias}    & --                & $0$ (disabled) \\
 & \texttt{--scale\_factor}, \texttt{--init\_dist} & init scale, dist. & $6$, Gaussian \\
\bottomrule
\end{tabular}
\end{table}

\subsection{Accuracy curves}
\label{app:cemsen-accuracy-curves}

Figure~\ref{fig:cemsen-accuracy-curves} shows the CIFAR-10 training and test
accuracy curves for BP, DFA, SBD, and the score-expansion variant of SBD on
the CNN architecture used in \citet{erdoganerror} and described in
Section~\ref{app:cemsen-architecture}. The expansion variant uses the
confidence-weighted and \texttt{roll(5)} blocks defined in
Eq.~\eqref{eq:err-exp}.

\begin{figure}[t]
\centering
\includegraphics[width=0.98\linewidth]{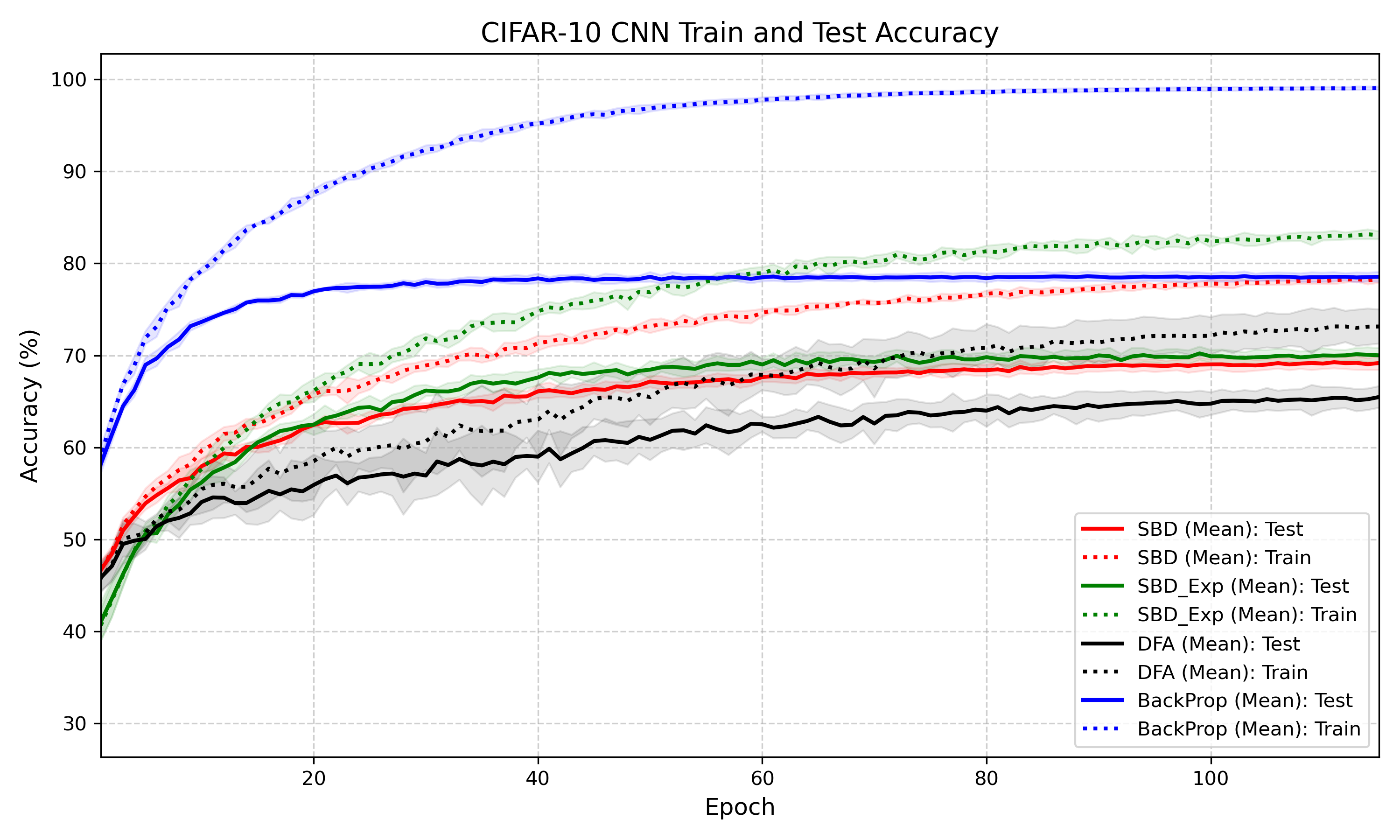}
\caption{Training and test accuracy curves on CIFAR-10 for BP, DFA, SBD, and
SBD with score expansion on the CNN architecture of
Section~\ref{app:cemsen-architecture}. The expanded-broadcast variant uses
the confidence-weighted $\mathbf{p}\odot\boldsymbol{\delta}$ block together
with the $\mathrm{roll}_{k}(\mathbf{p})\odot\boldsymbol{\delta}$ block, using $k=5$, from
Eq.~\eqref{eq:err-exp}. The displayed interval is shorter than the training
budget in Table~\ref{tab:cemsen-hparams}.}
\label{fig:cemsen-accuracy-curves}
\end{figure}

The curves are consistent with the final accuracies reported in
Table~\ref{tab:cifar10-cnn-comparison}. In particular, SBD trained with
cross entropy attains higher test accuracy than conventional DFA trained with
the same loss, indicating that the score-broadcast update is more effective
than the standard direct-feedback alternative in this setting. The
score-expansion configuration closely tracks plain SBD early in training and
achieves higher accuracy later. This comparison includes the batch-size and
learning-rate changes listed above.

Figure~\ref{fig:cemsen-accuracy-curves-quad} shows the corresponding CIFAR-10
training and test accuracy curves for the $4\times$-width CNN of
Table~\ref{tab:cemsen-arch-quad}. In this higher-capacity setting, we report
BP and SBD with expanded broadcast vector, where the expansion again uses the
confidence-weighted and \texttt{roll5} blocks from Eq.~\eqref{eq:err-exp}.

\begin{figure}[t]
\centering
\includegraphics[width=0.98\linewidth]{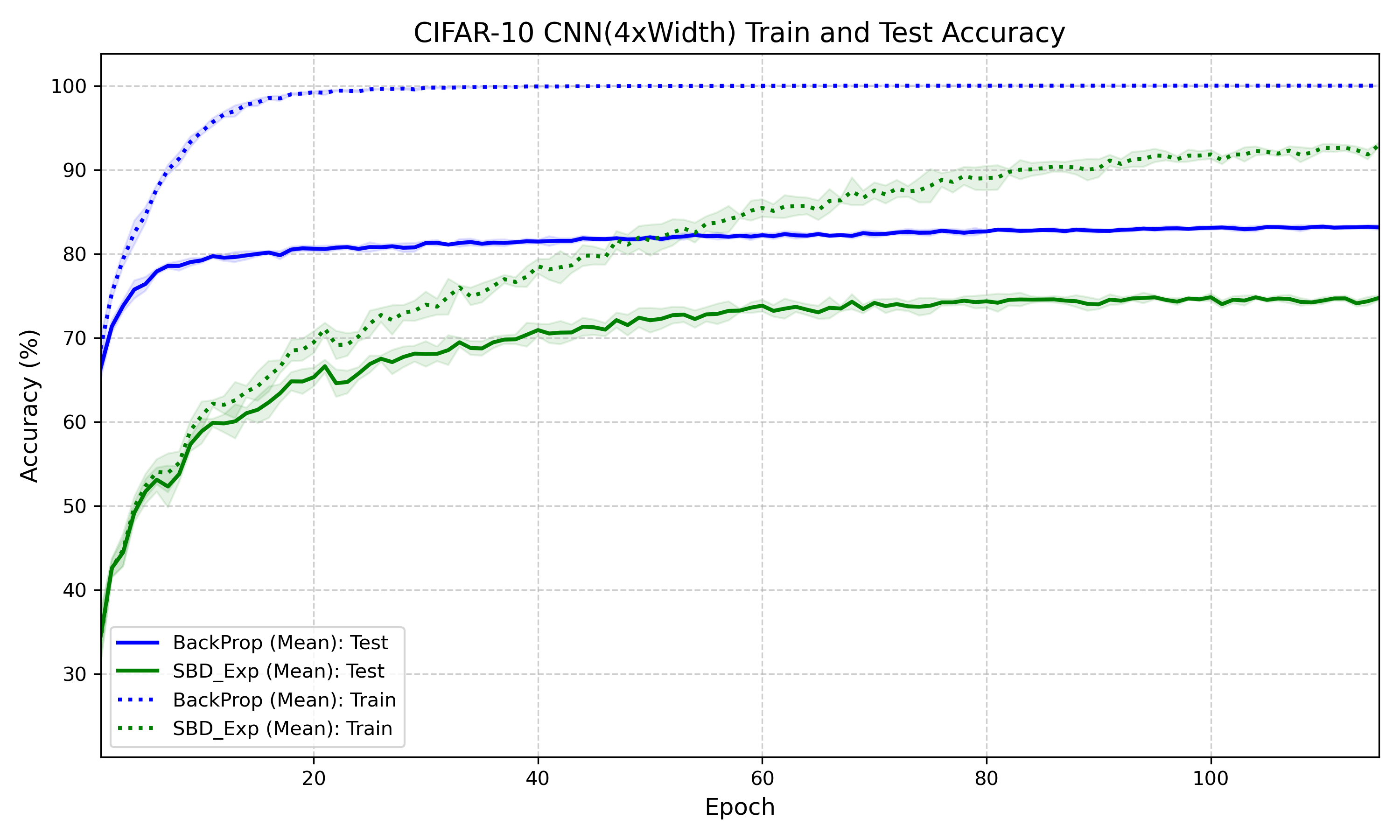}
\caption{Training and test accuracy curves on CIFAR-10 for BP and SBD with
score expansion on the $4\times$-width CNN architecture in
Table~\ref{tab:cemsen-arch-quad}. The SBD variant uses the same expanded
broadcast construction as in the baseline setting, namely the
confidence-weighted $\mathbf{p}\odot\boldsymbol{\delta}$ block together with
the $\mathrm{roll}_{k}(\mathbf{p})\odot\boldsymbol{\delta}$ block with $k=5$.
The displayed interval is shorter than the full training budget.}
\label{fig:cemsen-accuracy-curves-quad}
\end{figure}

Compared with the baseline-width model, both methods benefit substantially
from the $4\times$-width scaling. Using the final accuracies reported in
Section~\ref{sec:numerical_experiments}, BP increases from $78.51\%$ to
$83.1\%$, a gain of $4.59$ percentage points, while SBD with expanded
broadcast increases from $70.03\%$ to $74.46\%$, a gain of $4.43$ points.
The wider configurations preserve the same qualitative ordering. The BP
comparison also changes the optimization recipe as described above.

\subsection{Multi-seed ablation of score-expansion modulators}
\label{app:score-expansion-ablation}

The theory of Section~\ref{sec:score_expansion} and
Appendix~\ref{app:score_expansion_appendix} permits deterministic
$X$-measurable modulators, provided the required modulated moments exist.
We compare several such choices on CIFAR-10, including the $3D$-dimensional
expansion used in the main experiments.

\paragraph{Expansion methods.}
In the formulas below, $\vt{a}=\vt{a}(X)$ and $\vt{p}=\vt{p}(X)$ are deterministic functions of the input for fixed network parameters, so each candidate is an $X$-measurable special case of the general construction.
\begin{itemize}
\item[i.] {\texttt{confidence-weighted}}
\begin{equation*}
  \vt{\eta}_{\text{conf}}(X,Y) \;=\; \vt{p}(X)\odot\boldsymbol{\delta}
  \;\in\;\mathbb{R}^{D}.
\end{equation*}
The class-$d$ score coordinate is multiplied by the model's predicted probability
for class $d$. Coordinates corresponding to confidently predicted
classes contribute proportionally more to the broadcast; coordinates
where $p_{d}\!\approx\!0$ are suppressed.

\item[ii.] {\texttt{logit-weighted}}
\begin{equation*}
  \vt{\eta}_{\text{logit}}(X,Y) \;=\; \vt{a}(X)\odot\boldsymbol{\delta}
  \;\in\;\mathbb{R}^{D}.
\end{equation*}
The same idea as \texttt{confidence-weighted} but using the raw logits
$\vt{a}(X)$ instead of the softmax probabilities. Logit magnitudes are
unbounded, so this block can produce comparatively large modulators
once the network has converged.

\item[iii.] {\texttt{rollk}}

\begin{equation*}
  \vt{\eta}_{\text{rollk}}(X,Y;k)
  \;=\; \mathrm{roll}_{k}(\vt{p}(X))\odot\boldsymbol{\delta}
  \;\in\;\mathbb{R}^{D}.
\end{equation*}
A single cyclic shift of $\vt{p}(X)$ by $k$, $k\in\{0,1,\dots,D-1\}$, positions provides cross-class
coupling through a permutation of the output coordinates. Since class indices
are arbitrary, $k$ has no semantic interpretation and does not measure class
similarity or distance. The reported experiments simply fix $k=5$.
For $k=0$ the block is identical to \texttt{confidence-weighted}; for
$k=5$ it is identical to \texttt{roll5}.

\item[iv.] {\texttt{full-cyclic}}
\begin{equation*}
  \vt{\eta}_{\text{full}}(X,Y)
  \;=\;
  \Bigl[\,
    \mathrm{roll}_{1}(\vt{p}(X))\odot\boldsymbol{\delta}\,;\
    \dots\,;\
    \mathrm{roll}_{D-1}(\vt{p}(X))\odot\boldsymbol{\delta}
  \Bigr]
  \;\in\;\mathbb{R}^{D(D-1)}.
\end{equation*}
Concatenates all $D-1$ non-trivial cyclic shifts of $\vt{p}(X)$ into one
block. For $D=10$ this contributes $90$ extra coordinates and
captures all binary class-pairings $(d, (d+k)\bmod D)$ for $k=1,\dots,9$
simultaneously. The raw score is also retained, giving $100$ broadcast coordinates.

\item[v.] \texttt{boundary\_weight}
A band-pass modulator built from two softmax distributions at different
temperatures:
\begin{equation*}
  \vt{\phi}_{\text{bd}}(X)
  \;=\; \softmax(\vt{a}(X)/T_{1})\;-\;\softmax(\vt{a}(X)/T_{2})
  \quad\text{with } T_{1}=0.5,\, T_{2}=2.0,
\end{equation*}
\begin{equation*}
  \vt{\eta}_{\text{bd}}(X,Y) \;=\; \vt{\phi}_{\text{bd}}(X)\odot\boldsymbol{\delta}
  \;\in\;\mathbb{R}^{D}.
\end{equation*}
This block weights the score by the change in predicted probabilities under
temperature scaling. It vanishes for equal logits and in the limit of a
single dominant logit; its name does not imply a calibrated measure of distance
to a decision boundary.

\end{itemize}

\paragraph{Evaluation protocol.} We evaluated the candidate expansion
combinations across multiple independent seeds on the CIFAR-10 CNN of
Appendix~\ref{app:cemsen}. Table~\ref{tab:modulator-ablation} reports the mean
test accuracy and standard deviation across runs. The selected $3D$-dimensional
combination is the same expanded-SBD configuration reported in
Table~\ref{tab:cifar10-cnn-comparison}, whose result is repeated here for
direct comparison.

\paragraph{Modulator-combination ablation.} Table~\ref{tab:modulator-ablation} reports CIFAR-10 test accuracy for the candidate combinations evaluated. The selected combination $[\boldsymbol{\delta}; \mathbf{p} \odot \boldsymbol{\delta}; \mathrm{roll}_{k}(\mathbf{p}) \odot \boldsymbol{\delta}]$ with $k=5$, used in the main text, is highlighted.

\begin{table}[h]
    \centering
    \caption{5-seed CIFAR-10 top-1 test accuracy for candidate score-expansion modulator combinations. The selected configuration,
    whose result is also reported in
    Table~\ref{tab:cifar10-cnn-comparison}, is shown in bold.}
    \label{tab:modulator-ablation}
    \begin{tabular}{lc}
        \toprule
        Expansion & CIFAR-10 (\%) \\
        \midrule
        score + boundary-weighted & $69.40 \pm 0.31$ \\
        score + confidence-weighted & $69.67 \pm 0.42$ \\
        score + full-cyclic & $68.96 \pm 0.74$ \\
        score + logit-weighted & $61.03 \pm 1.89$ \\
        score + roll5 + confidence-weighted + boundary-weighted & $69.44 \pm 0.32$ \\
        \textbf{score + roll5 + confidence-weighted} &
        $\mathbf{70.0 \pm 0.7}$  \\
        score + roll5 & $68.35 \pm 1.46$ \\
        \bottomrule
    \end{tabular}
\end{table}

\subsection{Cosine alignment between SBD updates and backpropagation gradients}
\label{app:cemsen-sbd-bp-cosine}

To quantify the directional relationship between the local SBD direction and the
standard backpropagation (BP) gradient, we compute a signed cosine similarity
for each hidden layer of the CIFAR-10 CNN during SBD training. Let
$S^{(k)}_{e,b}$ denote the unregularized three-factor local SBD direction
for layer $k$ at epoch $e$ and mini-batch $b$, evaluated at the current network
parameters before auxiliary regularization and Adam processing.
Let $G^{(k)}_{\mathrm{BP},e,b}=\nabla_{W^{(k)}}\mathcal{L}_{\mathrm{CE}}$
denote the BP gradient of the cross entropy loss computed on the same
mini-batch and at the same parameter values. The BP gradient is used only
as a diagnostic; the network weights are still updated only by the SBD rule.
For each hidden layer, both tensors are flattened and compared as
\begin{equation}
    c^{(k)}_{e,b}
    =
    \frac{\left\langle \operatorname{vec}\!\left(S^{(k)}_{e,b}\right),
    \operatorname{vec}\!\left(G^{(k)}_{\mathrm{BP},e,b}\right)\right\rangle}
    {\left\|\operatorname{vec}\!\left(S^{(k)}_{e,b}\right)\right\|_2
    \left\|\operatorname{vec}\!\left(G^{(k)}_{\mathrm{BP},e,b}\right)\right\|_2
    + \epsilon_{\mathrm{num}}},
    \label{eq:sbd-bp-cosine}
\end{equation}
where $\epsilon_{\mathrm{num}}$ is a small numerical stabilizer. We do not take
an absolute value in Eq.~\eqref{eq:sbd-bp-cosine}; hence positive values denote
alignment between the local SBD direction and the BP gradient, while negative
values denote anti-alignment. A positive cosine makes the negative raw SBD
direction a first-order descent direction for the minibatch cross entropy.
The corresponding angle,
when reported, is $\arccos(\mathrm{clip}(c^{(k)}_{e,b},-1,1))$, but
Figure~\ref{fig:sbd-bp-cosine-cifar10} plots the cosine similarities themselves.
The mini-batch cosine values are averaged within each epoch and the experiment
is repeated over five independent seeds. Solid lines in the plot show the
across-seed mean, and the shaded envelopes show one standard deviation.

\begin{figure}[t]
\centering
\includegraphics[width=0.82\linewidth]{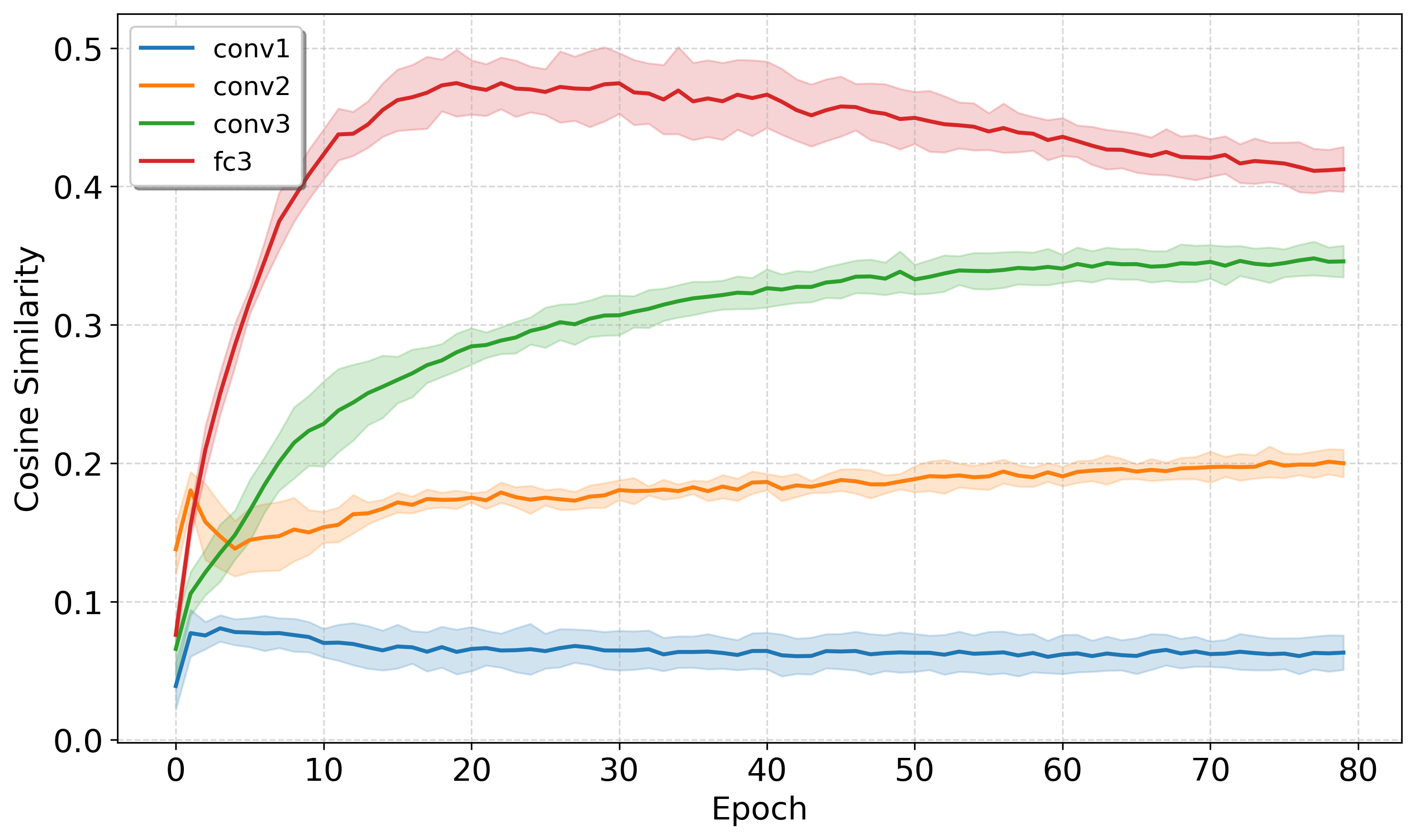}
\caption{Cosine similarity between the local SBD direction and the diagnostic
BP gradient for each hidden layer of the CIFAR-10 CNN. Curves show means over
five seeds, and shaded envelopes denote one standard deviation. Mean cosine
similarity remains positive in every hidden layer over the plotted interval.}
\label{fig:sbd-bp-cosine-cifar10}
\end{figure}

The deepest hidden fully connected layer exhibits the largest cosine similarity,
while the convolutional layers have weaker positive alignment.

\paragraph{Additional hyperparameter sensitivity.}
We additionally varied the covariance EMA coefficient $\lambda$, the
output-layer covariance coefficient $c_{\mathrm{cov}}^{(\mathrm{fc4})}$
(\texttt{--CCOV\_OUT}), and the minibatch size $B$, holding all other settings
fixed. Performance is stable around the selected EMA regime,
$\lambda\in\{0.99999,0.999995\}$, and across
$c_{\mathrm{cov}}^{(\mathrm{fc4})}\in[10^{-8},10^{-6}]$; within the tested
small-batch regime, accuracy increases with $B$.
Table~\ref{tab:cifar10-hyperparameter-sensitivity} summarizes these results.

\begin{table}[t]
\centering
\caption{CIFAR-10 hyperparameter sensitivity. Entries report test accuracy
(\%; mean $\pm$ standard deviation over five independent seed runs).}
\label{tab:cifar10-hyperparameter-sensitivity}
\scriptsize
\setlength{\tabcolsep}{2.5pt}
\begin{tabular}{lccccccc}
\toprule
$\lambda$
& $0.99993$ & $0.99996$ & $0.99997$ & $0.99998$
& $0.99999$ & $0.999995$ & $0.999999$ \\
\midrule
Test accuracy
& $10.00\pm0.00$ & $20.63\pm4.02$ & $42.70\pm1.59$
& $59.47\pm1.99$ & $69.13\pm0.81$ & $69.33\pm0.54$
& $67.52\pm0.72$ \\
\bottomrule
\end{tabular}

\vspace{1ex}

\begin{tabular}{lcccc}
\toprule
$c_{\mathrm{cov}}^{(\mathrm{fc4})}$
& $10^{-5}$ & $10^{-6}$ & $10^{-7}$ & $10^{-8}$ \\
\midrule
Test accuracy
& $66.96\pm0.91$ & $69.33\pm0.64$
& $69.34\pm0.73$ & $69.26\pm0.69$ \\
\bottomrule
\end{tabular}

\vspace{1ex}

\begin{tabular}{lccc}
\toprule
Batch size $B$ & $1$ & $4$ & $8$ \\
\midrule
Test accuracy
& $59.36\pm0.02$ & $66.26\pm0.31$ & $68.36\pm0.61$ \\
\bottomrule
\end{tabular}
\end{table}

\section{Tiny ImageNet experimental setup}
\label{app:tinet}

This appendix documents the codebase used for the Tiny ImageNet
experiments reported in the main text, including BP, DFA, 
plain SBD, and score-expanded SBD. Tiny ImageNet~\citep{le2015tiny}
was released as a Stanford CS231N course project, a downsized
200-class subset of ImageNet, without a formal open-source license
declaration on its original distribution page, and is used here for
non-commercial academic research in accordance with the underlying ImageNet
terms of use.

The Tiny ImageNet benchmark consists of $200$ classes with $500$ training
images and $50$ validation images per class, all resized to
$64\!\times\!64$ RGB. The reported SBD configurations are implemented by
per-seed launch scripts in the Tiny ImageNet codebase: plain SBD uses
\texttt{CNN\_tinyImageNet/SBD/run\_tinet\_sbd\_ce\_s*.sh}, while
score-expanded SBD uses
\texttt{CNN\_tinyImageNet/SBD\_exp/run\_tinet\_sbd\_exp\_ce\_s*.sh}.
These scripts call the corresponding Tiny ImageNet trainers,
\texttt{cnn\_tinet.py} for plain SBD and \texttt{cnn\_tinete.py} for
score-expanded SBD, applied to the 6-layer CNN described in
Section~\ref{app:tinet-architecture}.

The forward and local update rules follow the CIFAR-10 SBD formulation of
Section~\ref{app:cemsen-sbd}; this appendix therefore focuses on the
Tiny-ImageNet-specific architecture, data pipeline, and hyperparameter choices.
All experiments run in \texttt{float32} on a single NVIDIA V100 (32GB) GPU. For comparison on a common budget, the approximate times to complete $200$ epochs on Tiny ImageNet were $14$, $20$, $30$, and $32$ hours for BP, DFA, SBD, and SBD with score expansion, respectively.

Training uses the official training set, and evaluation uses the official
validation set because the standard test labels are unavailable. Throughout
the Tiny ImageNet results and figures, ``test accuracy'' refers to accuracy
on this official validation set.

For plain SBD, the broadcast vector equals the raw cross-entropy score
\[
\boldsymbol{\varepsilon}=\boldsymbol{\delta}=\mathbf{p}-\mathbf{y}
\in\mathbb{R}^{200}.
\]
For score-expanded SBD, the implementation uses the $3D$-dimensional expansion
described in Section~\ref{app:cemsen-sbd},
\[
\boldsymbol{\varepsilon}
=
[\boldsymbol{\delta};\;
\mathbf{p}\odot\boldsymbol{\delta};\;
\mathrm{roll}_{k}(\mathbf{p})\odot\boldsymbol{\delta}]
\in\mathbb{R}^{600},
\]
selected by \texttt{--err\_expand=2} in the
\texttt{SBD\_exp} launch scripts, with $k=5$. The softmax temperature and broadcast scale
are fixed at $T=1$ and $\alpha=1$ in the Tiny ImageNet implementation.

\subsection{Network architecture}
\label{app:tinet-architecture}

The Tiny ImageNet network is a wider and deeper variant of the
CIFAR-10 backbone. It uses three $5\!\times\!5$ convolutional blocks
with max-pooling, followed by three fully connected layers with
ReLU and dropout. Every convolution and linear layer contains a bias, but the
implemented SBD variants leave all biases at their zero initialization;
DFA updates only FC biases, and BP updates all biases.
A single \texttt{--width\_multiplier} flag scales
all conv channels and the FC hidden width jointly.

\begin{table}[h]
\centering
\caption{Tiny ImageNet CNN architecture at \texttt{--width\_multiplier}$=1.0$.
Input: Tiny ImageNet images of shape $(3,\,64,\,64)$. Total
parameters, including biases: $14{,}130{,}888$.}
\label{tab:tinet-arch}
\small
\begin{tabular}{lllr}
\toprule
Layer & Operation & Output shape & Parameters \\
\midrule
conv1   & $5\!\times\!5$ conv, stride 2, pad 2 & $(96,\,32,\,32)$  & $7{,}296$ \\
act1    & ReLU                                  & $(96,\,32,\,32)$  & -- \\
pool1   & MaxPool $2\!\times\!2$                & $(96,\,16,\,16)$  & -- \\
conv2   & $5\!\times\!5$ conv, stride 1, pad 2 & $(128,\,16,\,16)$ & $307{,}328$ \\
act2    & ReLU                                  & $(128,\,16,\,16)$ & -- \\
pool2   & MaxPool $2\!\times\!2$                & $(128,\,8,\,8)$   & -- \\
conv3   & $5\!\times\!5$ conv, stride 1, pad 2 & $(256,\,8,\,8)$   & $819{,}456$ \\
act3    & ReLU                                  & $(256,\,8,\,8)$   & -- \\
pool3   & MaxPool $2\!\times\!2$                & $(256,\,4,\,4)$   & -- \\
flatten & reshape                               & $(4096,)$         & -- \\
fc3     & linear                                & $(2048,)$         & $8{,}390{,}656$ \\
act4    & ReLU $+$ Dropout($p$)                 & $(2048,)$         & -- \\
fc4     & linear                                & $(2048,)$         & $4{,}196{,}352$ \\
act5    & ReLU $+$ Dropout($p$)                 & $(2048,)$         & -- \\
fc5     & linear                                & $(200,)$          & $409{,}800$ \\
softmax & softmax at $T\!=\!1$                  & $(200,)$          & -- \\
\bottomrule
\end{tabular}
\end{table}

The conv channels $P_0,P_1,P_2$ and FC hidden width $F_h$ are
parameterized as
\begin{equation*}
P_0 \!=\! \mathrm{round}(96\!\cdot\!w),\quad
P_1 \!=\! \mathrm{round}(128\!\cdot\!w),\quad
P_2 \!=\! \mathrm{round}(256\!\cdot\!w),\quad
F_h \!=\! \mathrm{round}(2048\!\cdot\!w),
\end{equation*}
where $w=\mathrm{\texttt{width\_multiplier}}$. The fc3 input dimension $P_2\!\cdot\!4\!\cdot\!4$ is derived automatically
from the conv stack, giving $4096$ at \texttt{--width\_multiplier}$=1.0$. Weights are initialized from a Gaussian
distribution with standard deviation
$\sqrt{2 / \mathrm{(\texttt{scale\_factor} \cdot \text{fan\_in})}}$,
giving a scaled variant of He et al.'s fan-in initialization~\citep{he2015rectifiers} with $\texttt{scale\_factor}=6$;
biases are initialized to zero.

\subsection{Data pipeline and augmentation}
\label{app:tinet-data}

All images are normalized per-channel using the standard ImageNet
statistics
\begin{equation*}
\boldsymbol{\mu} = (0.485,\,0.456,\,0.406),\qquad
\boldsymbol{\sigma} = (0.229,\,0.224,\,0.225).
\end{equation*}
Training augmentation consists of (i) a random crop to
$64\!\times\!64$ from a $4$-pixel reflection-padded image and
(ii) random horizontal flipping, applied independently per sample
and per epoch. No color jittering, mixup, or RandAugment is used in
the main-text configuration.

For CE BP, DFA, SBD, and expanded SBD, training and validation accuracies
use the same deterministic mirror-and-translation ensemble implemented in
\texttt{infer\_mirror\_translate}. For each image, the code first averages
the predicted probabilities over the
original image and its horizontal flip. It then applies two deterministic
reflection-padded translation shifts and averages the mirror-averaged
outputs of these translated views. The final prediction is
\[
0.5\,\bar{\mathbf{y}}_{\mathrm{mirror}}
+
0.5\,\bar{\mathbf{y}}_{\mathrm{trans}},
\]
where $\bar{\mathbf{y}}_{\mathrm{mirror}}$ is the mean of two forward
passes and $\bar{\mathbf{y}}_{\mathrm{trans}}$ is the mean over two
translated mirror pairs. The evaluation data loader itself applies no
stochastic augmentation beyond normalization.

Evaluation disables dropout and runs after every training pass.

\subsection{Training hyperparameters}
\label{app:tinet-hparams}

Table~\ref{tab:tinet-hparams} lists the training hyperparameters used by the
Tiny ImageNet score-expanded SBD launch scripts
\texttt{CNN\_tinyImageNet/SBD\_exp/run\_tinet\_sbd\_exp\_ce\_s*.sh}. The \textit{SBD coefficients} block follows the layer indexing of
Section~\ref{app:cemsen-sbd}, with \texttt{CMSE\_OUT} now applied to
the deeper output layer fc5, \texttt{CMSE\_OUT2} shared between the
two hidden FC layers fc3 and fc4, and \texttt{CMSE\_HIDDEN} shared
across all three convolutional layers. The covariance coefficients use
the analogous grouping, while \texttt{CL1\_OUT} is applied to the two
hidden FC layers fc3 and fc4 in the Tiny ImageNet implementation. Each of
these layers receives both the activation $\ell_1$ term of
Section~\ref{app:cemsen-sbd} and an additional weight $\ell_1$ term
$c_{\ell_1}^{(k)}\mathrm{sign}(W^{(k)})$.
The local FC rules use post-dropout activations in the moment estimates and
the ReLU gate in the score and covariance directions.

Plain SBD shares the expanded-SBD optimization settings. The CE BP and DFA
runs use initial learning rates $10^{-4}$ and $5\times10^{-5}$, respectively,
and dropout probabilities $0.10$ and $0.05$, compared with $0.08$ for both
SBD variants. All four use batch size $28$, Adam weight decay
$1.25\times10^{-5}$, and learning-rate decay $0.98$.
The SBD training budget comprises $251$ passes, indexed $0$--$250$; incomplete
training batches are dropped. Learning-rate and EMA schedules start at
zero-based epoch $2$, as in CIFAR-10.

The Tiny ImageNet expanded-SBD hyperparameters in Table~\ref{tab:tinet-hparams} were selected by a grid search over the SBD learning-rate and regularization coefficients. The search was managed using Weights \& Biases sweeps~\citep{wandb}, including automatic early termination of clearly underperforming configurations.

\begin{table}[h]
\centering
\caption{Tiny ImageNet expanded-SBD training hyperparameters. Values shown correspond
to the score-expanded SBD CE launch scripts
\texttt{CNN\_tinyImageNet/SBD\_exp/run\_tinet\_sbd\_exp\_ce\_s*.sh}. Column \emph{Flag}
is the CLI name in the launch script; column \emph{Symbol} names the variable
used in the equations of Section~\ref{app:cemsen-sbd}.}
\label{tab:tinet-hparams}
\small
\setlength{\tabcolsep}{3pt}
\renewcommand{\arraystretch}{1.15}
\begin{tabularx}{\linewidth}{l l l X}
\toprule
Group & Flag & Symbol & Value \\
\midrule
\multirow{5}{*}{\emph{General}}
 & \texttt{--logger\_name} & seed & $5$ independent seeds \\
 & \texttt{--n\_epochs}    & --   & $251$ \\
 & \texttt{--batch\_size}  & $B$  & $28$ \\
 & \texttt{--loss\_out}    & --   & cross entropy (\texttt{ce}) \\
 & augmentation            & --   & RandomCrop pad $4$ $+$ HFlip \\
\midrule
\multirow{5}{*}{\emph{Method}}
 & \texttt{--method}      & -- & \texttt{sbd}\\
 & \texttt{--err\_expand} & -- & $2$ \\
 & broadcast vector       & $\boldsymbol{\varepsilon}$ &
 $3D$-dimensional expansion of $\boldsymbol{\delta}$ \\
 & softmax temperature    & $T$ & $1.0$ fixed \\
 & effective $D_{\varepsilon}$ & $D_{\varepsilon}$ &
 $600$ \\
\midrule
\multirow{6}{*}{\emph{Optimizer}}
 & \texttt{--lr}              & $\eta_0$ & $8.5\!\times\!10^{-5}$ \\
 & \texttt{--lr\_drop\_rate}  & --       & $0.98$ \\
 & \texttt{--lr\_drop\_every} & --       & $1$ epoch \\
 & \texttt{--weight\_decay}   & --       & $1.25\!\times\!10^{-5}$ \\
 & Adam $\beta_1$, $\beta_2$  & --       & $0.9$, $0.999$ \\
 & precision                  & --       & \texttt{float32} \\
\midrule
\multirow{8}{*}{\emph{SBD coefficients}}
 & \texttt{--CMSE\_OUT}    & $c_{\mathrm{score}}^{(\mathrm{fc5})}$     & $10$ \\
 & \texttt{--CMSE\_OUT2}   & $c_{\mathrm{score}}^{(\mathrm{fc3,fc4})}$ & $0.1$ \\
 & \texttt{--CMSE\_HIDDEN} & $c_{\mathrm{score}}^{(\mathrm{conv})}$    & $0.1$ \\
 & \texttt{--CCOV\_OUT}    & $c_{\mathrm{cov}}^{(\mathrm{fc5})}$       & $1\!\times\!10^{-7}$ \\
 & \texttt{--CCOV\_OUT2}   & $c_{\mathrm{cov}}^{(\mathrm{fc3,fc4})}$   & $6\!\times\!10^{-7}$ \\
 & \texttt{--CCOV\_HIDDEN} & $c_{\mathrm{cov}}^{(\mathrm{conv})}$      & $0$ \\
 & \texttt{--CL1\_OUT}     & $c_{\ell_1}^{(\mathrm{fc3,fc4})}$         & $1\!\times\!10^{-7}$ \\
 & \texttt{--CL1\_HIDDEN}  & $c_{\ell_1}^{(\mathrm{hidden})}$          & $0$ \\
\midrule
\multirow{7}{*}{\emph{Covariance EMA}}
 & \texttt{--Reh\_lambda}              & $\lambda$         & $0.999992$ \\
 & \texttt{--Reh\_lambda2}             & $\lambda_2$       & $0.999992$ \\
 & \texttt{--Reh\_lambda\_drop}        & $\rho$            & $0.015$ \\
 & \texttt{--Reh\_lambda\_drop\_every} & --                & $1$ epoch \\
 & \texttt{--Reh\_gain}                &
 $\mathrm{std}(\widehat{\mathbf{R}}_{g\varepsilon}^{(\mathrm{conv})})$ & $0.01$ \\
 & \texttt{--Reh\_gain\_lin}           &
 $\mathrm{std}(\widehat{\mathbf{R}}_{g\varepsilon}^{(\mathrm{FC})})$ & $0.01$ \\
 & \texttt{--Reh\_ini}                 &
 init diagonal of $\widehat{\mathbf{R}}_{hh}$ & $1\!\times\!10^{-8}$ \\
\midrule
\multirow{3}{*}{\emph{Architecture}}
 & \texttt{--width\_multiplier} & $w$ & $1.0$ \\
 & \texttt{--dropout}           & $p$ & $0.08$ \\
 & \makecell[l]{\texttt{--scale\_factor}, \texttt{--init\_dist}}
   & init scale, dist. & $6$, Gaussian \\
\bottomrule
\end{tabularx}
\end{table}

\subsection{Accuracy curves}
\label{app:tinet-accuracy-curves}

\begin{figure}[h]
\centering
\includegraphics[width=0.85\linewidth]{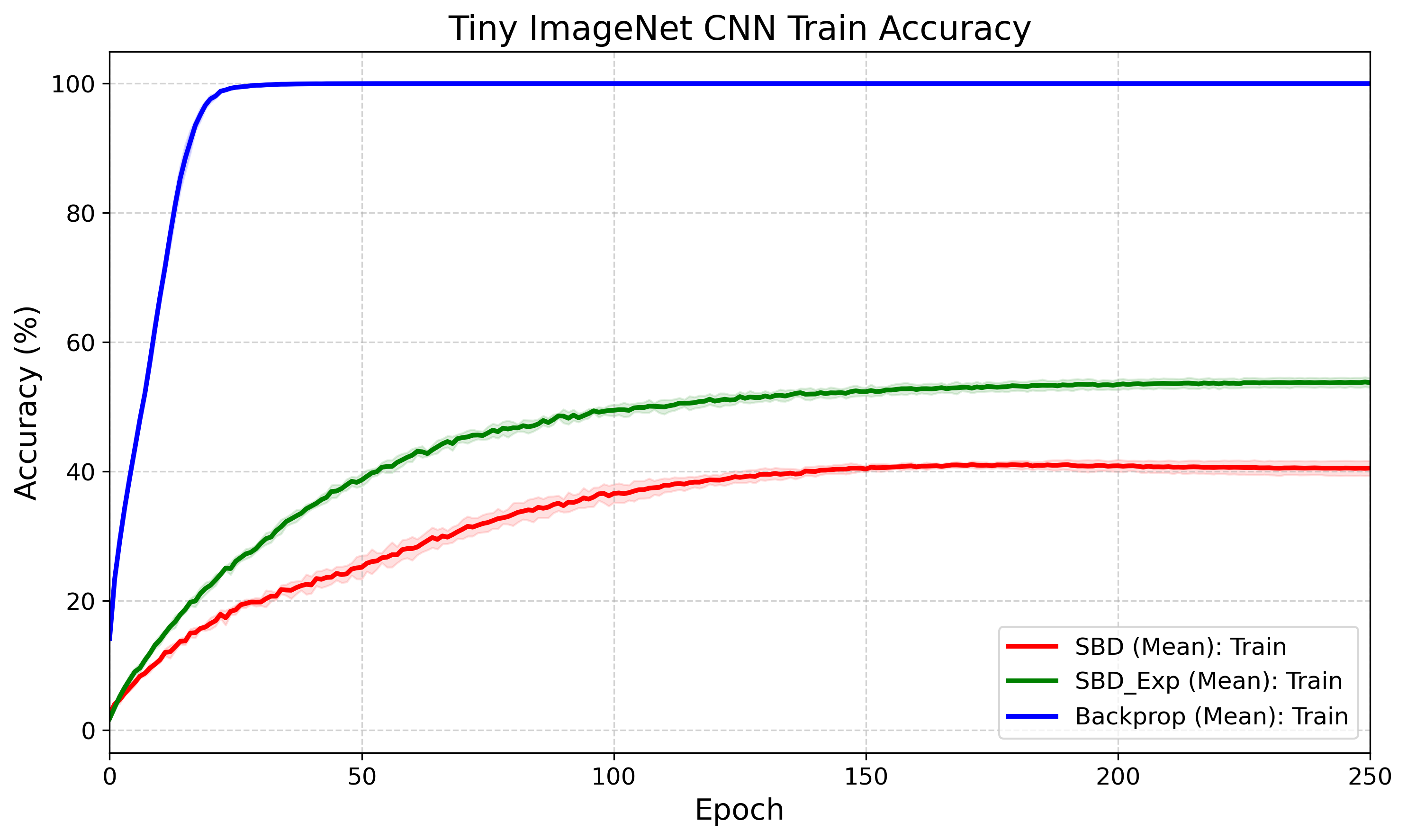}
\caption{Training accuracy on Tiny ImageNet for BP, SBD, and SBD with score expansion (SBD\_Exp), shown at epochs $0$--$250$. Each curve represents the mean over five runs, and the shaded region denotes one standard deviation.}
\label{fig:tinet-train-accuracy-curve}
\end{figure}

Figure~\ref{fig:tinet-test-accuracy-curve} and
Figure~\ref{fig:tinet-train-accuracy-curve} show the Tiny ImageNet test and
training accuracy curves for BP, SBD, and SBD with score expansion. The curves
show a clear qualitative separation between the methods. BP obtains the highest
test accuracy, reaching $39.89\%$, while SBD with score expansion reaches
$31.43\%$ and plain SBD reaches $28.29\%$. Thus, as in the CIFAR-10 experiments,
expanding the broadcast vector improves the SBD update, yielding a gain of
$3.14$ percentage points over using the score vector alone.

\begin{figure}[h]
\centering
\includegraphics[width=0.85\linewidth]{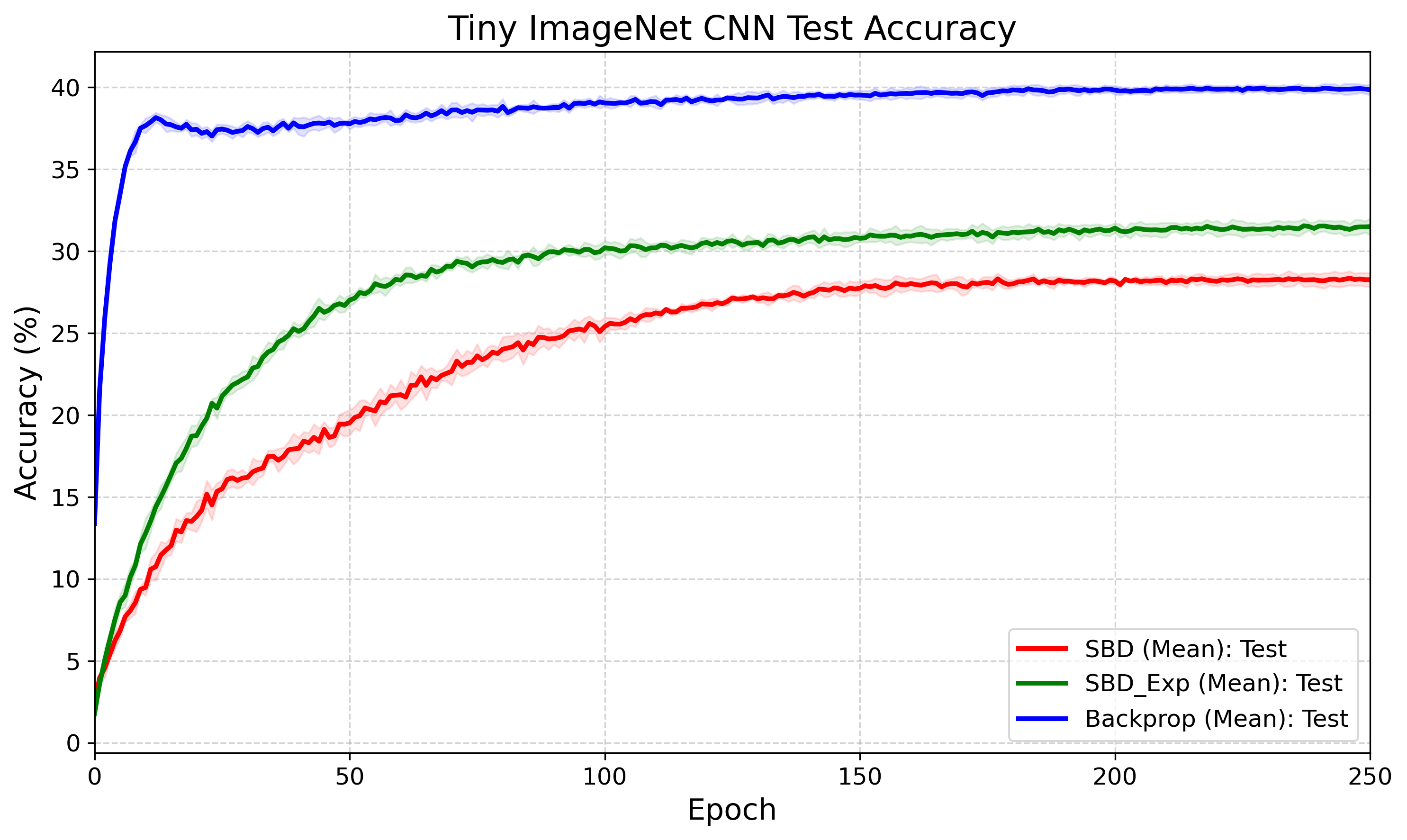}
\caption{Test accuracy on Tiny ImageNet for BP, SBD, and SBD with score expansion (SBD\_Exp), shown at epochs $0$--$250$. Each curve represents the mean over five runs, and the shaded region denotes one standard deviation.}
\label{fig:tinet-test-accuracy-curve}
\end{figure}

Figure~\ref{fig:tinet-dfa-accuracy-curve} shows the corresponding DFA training
and test accuracy curves. DFA improves slowly over the longer 500-epoch run,
reaching $17.49\%$ test accuracy and $21.02\%$ training accuracy at the end of
training. Even with this longer training horizon, DFA remains well below the SBD
variants, indicating that the score-broadcast update is more effective than the
standard direct-feedback alternative on the Tiny ImageNet CNN experiment as
well.

\begin{figure}[h]
\centering
\includegraphics[width=0.85\linewidth]{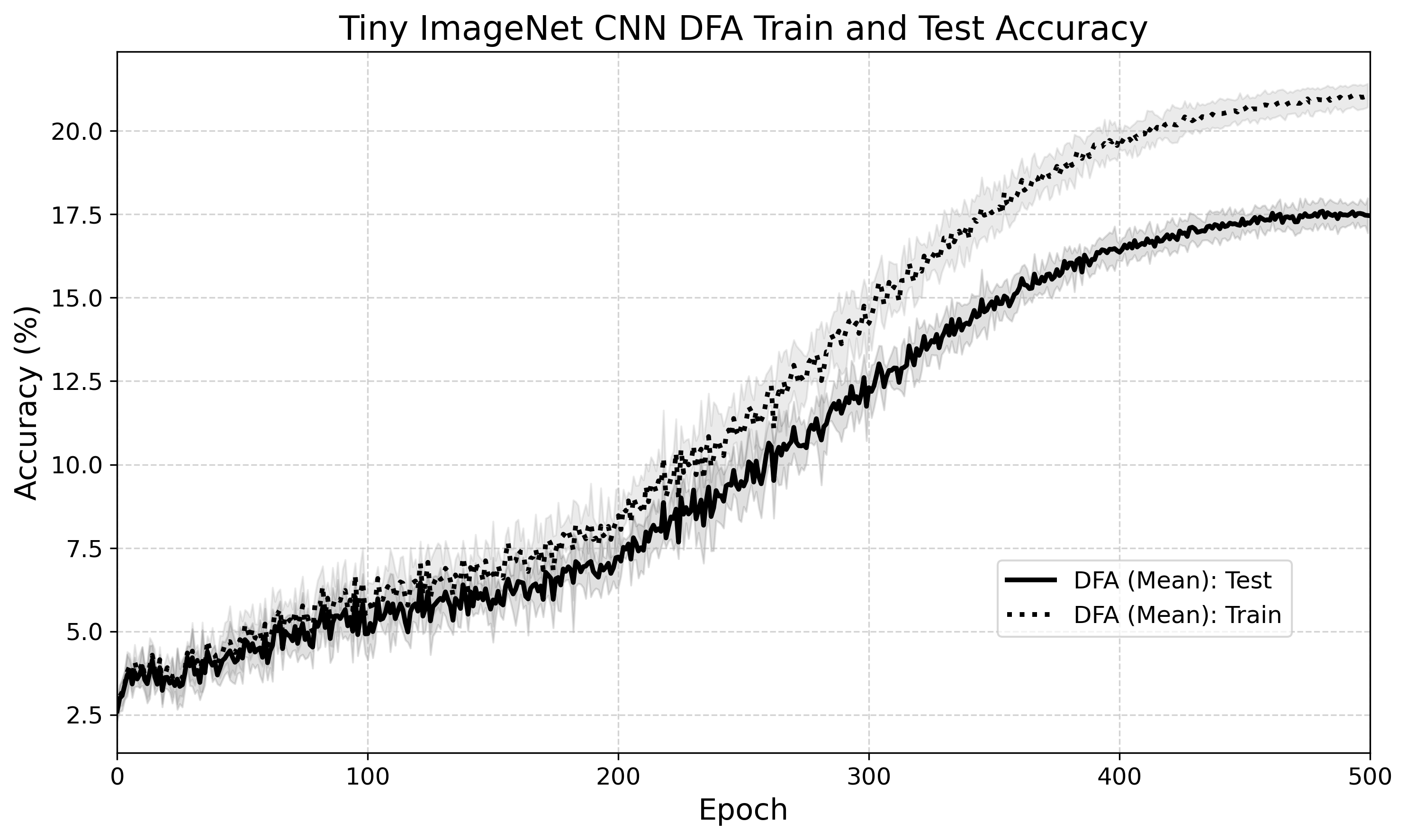}
\caption{Training and test accuracy on Tiny ImageNet for DFA across 500 epochs. Each curve represents the mean over five runs, and the shaded region denotes one standard deviation. DFA improves slowly throughout training and remains substantially below the accuracies obtained by BackProp and the SBD variants.}
\label{fig:tinet-dfa-accuracy-curve}
\end{figure}

\paragraph{Convergence speed and measured GPU energy.}
Table~\ref{tab:tinet-convergence-energy} reports the median first-passage
epoch and cumulative GPU energy required to reach fixed evaluation-accuracy
targets. Because the methods converge to different final accuracies, we also
include a method-relative target: writing
$A_{\mathcal{M}}^{\mathrm{final}}$ for the final evaluation accuracy of method
$\mathcal{M}$, this threshold is
$0.95A_{\mathcal{M}}^{\mathrm{final}}$. All runs used a single NVIDIA V100
(32GB), and the energy values are GPU-only cumulative measurements from the
system logs. Each table entry is epoch/kWh; NR denotes that the target was not
reached.

\begin{table}[h]
\centering
\caption{Tiny ImageNet convergence and measured GPU energy. Entries give the
median first-passage epoch and cumulative GPU-only energy in kWh.}
\label{tab:tinet-convergence-energy}
\small
\setlength{\tabcolsep}{5pt}
\begin{tabular}{lcccc}
\toprule
Evaluation target & BP & SBD & SBD Exp. & DFA \\
\midrule
$15\%$ & $4/0.028$ & $26/0.447$ & $16/0.358$ & $363/3.640$ \\
$25\%$ & $5/0.033$ & $95/1.497$ & $42/0.913$ & NR \\
$30\%$ & $6/0.037$ & NR & $95/1.968$ & NR \\
$95\%$ of method-specific final accuracy
& $46/0.218$ & $129/2.093$ & $93/1.991$ & $410/4.118$ \\
\bottomrule
\end{tabular}
\end{table}

BP reaches these targets fastest and with the lowest measured energy in the
present implementation. This comparison should not be interpreted as an
intrinsic hardware-efficiency advantage: BP uses highly optimized
PyTorch/CUDA kernels, whereas the SBD variants use custom layer-local update
code. Among the broadcast methods, expanded SBD reaches each attainable fixed
target before plain SBD, is the only method other than BP to reach $30\%$, and
reaches its method-relative target before both plain SBD and DFA. Both SBD
variants require substantially fewer epochs and less measured energy than DFA
at the targets they share.

\enlargethispage{3\baselineskip}

A notable difference from the SBD and DFA curves is the strong overfitting
observed for BP. In the training plot, BP rapidly approaches nearly perfect
training accuracy, reaching $99.98\%$, whereas its test accuracy saturates at
$39.89\%$. This corresponds to a train-test gap of $60.09$ percentage points.
In contrast, the SBD variants achieve lower training accuracy and also lower
test accuracy, but exhibit a much less extreme separation between training and
test performance. The expanded SBD model, for example, reaches $53.77\%$
training accuracy and $31.43\%$ test accuracy. This suggests that BP fits the
Tiny ImageNet training set very aggressively in this architecture, while the
broadcast-based learning rules act as a more constrained optimization
procedure.

\newpage

\section{Correlative Information Maximization (CorInfoMax) Network Experiments}
\label{app:corinfomax-experiments}
This appendix describes the CorInfoMax-based network used in our additional
experiments. The construction combines two previously separate ideas. The
internal representation dynamics, asymmetric predictors, and lateral
interactions follow the supervised Correlative Information Maximization
(CorInfoMax) framework of \citet{bozkurt2023correlative}. Supervised plasticity,
however, is supplied by the cross-entropy score developed in this paper rather
than by the two-state equilibrium-propagation procedure used in the original
CorInfoMax formulation. This distinction is also important relative to the
CorInfoMax-EBD model of \citet{erdoganerror}: CorInfoMax-EBD broadcasts an MSE
output error, whereas the present CorInfoMax-SBD models use cross-entropy-based
supervision. We first describe the underlying construction, then specify the
score, connectivity, and numerical controls used in each experiment.

\subsection{CorInfoMax background and score-broadcast learning}
\label{app:corinfomax-sbd}
The following background follows \citet{bozkurt2023correlative}; the
experimental variants below disable the top-down predictor pathway.
\subsubsection{Bidirectional prediction view of CorInfoMax}
\label{app:corinfomax-objective}

Let $\mathbf{h}^{(0)}=\mathbf{x}$ and let
$\mathbf{h}^{(k)}\in\mathbb{R}^{N^{(k)}}$, $k=1,\ldots,L$, denote the states
of a recurrent hierarchy. Each adjacent pair is connected by two independently
parameterized linear predictors. The feedforward matrix
$\mathbf{W}_{\mathrm{ff}}^{(k)}$ predicts layer $k+1$ from layer $k$, while
$\mathbf{W}_{\mathrm{fb}}^{(k)}$ predicts layer $k$ from layer $k+1$. Their
instantaneous residuals are
\begin{align}
\overrightarrow{\mathbf{e}}^{(k+1)}
&=\mathbf{h}^{(k+1)}-\mathbf{W}_{\mathrm{ff}}^{(k)}\mathbf{h}^{(k)},
\label{eq:corinfomax-forward-error}\\
\overleftarrow{\mathbf{e}}^{(k)}
&=\mathbf{h}^{(k)}-\mathbf{W}_{\mathrm{fb}}^{(k)}\mathbf{h}^{(k+1)}.
\label{eq:corinfomax-backward-error}
\end{align}
No transpose constraint is imposed between
$\mathbf{W}_{\mathrm{ff}}^{(k)}$ and $\mathbf{W}_{\mathrm{fb}}^{(k)}$.

For a minibatch indexed by $m$, write
$\mathbf{H}^{(k)}[m]\in\mathbb{R}^{N^{(k)}\times B}$ for the layer-state
matrix and use analogous notation for the two residual matrices. In this
background formulation, second-order statistics can be tracked by
exponential averages; for example,
\begin{align}
\widehat{\mathbf{R}}_{h}^{(k)}[m]
&=\lambda_h\widehat{\mathbf{R}}_{h}^{(k)}[m-1]
+\frac{1-\lambda_h}{B}\mathbf{H}^{(k)}[m]\mathbf{H}^{(k)}[m]^T,
\label{eq:corinfomax-state-correlation}\\
\widehat{\mathbf{R}}_{\rightarrow e}^{(k+1)}[m]
&=\lambda_h\widehat{\mathbf{R}}_{\rightarrow e}^{(k+1)}[m-1]
+\frac{1-\lambda_h}{B}
\overrightarrow{\mathbf{E}}^{(k+1)}[m]
\overrightarrow{\mathbf{E}}^{(k+1)}[m]^T,
\label{eq:corinfomax-forward-correlation}
\end{align}
with a corresponding recursion for
$\widehat{\mathbf{R}}_{\leftarrow e}^{(k)}$.

Correlative mutual information between two neighboring states admits a
forward-prediction form and a backward-prediction form:
\begin{align}
\widehat{I}_{\rightarrow}^{(\epsilon_k)}
(\mathbf{h}^{(k)},\mathbf{h}^{(k+1)})
&=\frac{1}{2}\log\det\!\left(
\widehat{\mathbf{R}}_{h}^{(k+1)}+\epsilon_k\mathbf{I}\right)
-\frac{1}{2}\log\det\!\left(
\widehat{\mathbf{R}}_{\rightarrow e}^{(k+1)}+\epsilon_k\mathbf{I}\right),
\label{eq:corinfomax-forward-cmi}\\
\widehat{I}_{\leftarrow}^{(\epsilon_k)}
(\mathbf{h}^{(k)},\mathbf{h}^{(k+1)})
&=\frac{1}{2}\log\det\!\left(
\widehat{\mathbf{R}}_{h}^{(k)}+\epsilon_k\mathbf{I}\right)
-\frac{1}{2}\log\det\!\left(
\widehat{\mathbf{R}}_{\leftarrow e}^{(k)}+\epsilon_k\mathbf{I}\right).
\label{eq:corinfomax-backward-cmi}
\end{align}
The first log-determinant in each expression rewards a representation that
uses the available dimensions, while the second penalizes variation that is
not linearly predictable from the neighboring layer. For an interior layer,
the state-dependent CorInfoMax criterion is chosen as
\begin{equation}
\mathcal{C}^{(k)}
=\widehat{I}_{\rightarrow}^{(\epsilon_{k-1})}
(\mathbf{h}^{(k-1)},\mathbf{h}^{(k)})
+\widehat{I}_{\leftarrow}^{(\epsilon_k)}
(\mathbf{h}^{(k)},\mathbf{h}^{(k+1)}).
\label{eq:corinfomax-local-objective}
\end{equation}
Using opposite prediction directions for the two adjacent terms yields
separate bottom-up and top-down pathways without requiring weight transport.

\subsubsection{Free-state recurrent inference}
\label{app:corinfomax-inference}

CorInfoMax inference updates all non-input states while the predictor and
lateral parameters are fixed. Let $\mathbf{u}^{(k)}$ be the somatic variable
before projection and let
\begin{equation}
\mathbf{C}_{\mathrm{ideal}}^{(k)}[m]
=\left(\widehat{\mathbf{R}}_{h}^{(k)}[m]
+\epsilon_k\mathbf{I}\right)^{-1}
\label{eq:corinfomax-lateral-matrix}
\end{equation}
denote the inverse-correlation operator motivating lateral adaptation.
The implemented lateral matrix $\mathbf C^{(k)}$ is updated recursively,
rather than recomputed by this inversion. A background form of the
hidden-state dynamics is
\begin{align}
\tau_u\frac{d\mathbf{u}^{(k)}}{ds}
={}&-g_{\mathrm{leak},k}\mathbf{u}^{(k)}
+2\gamma_h\mathbf{C}^{(k)}\mathbf{h}^{(k)}
-\frac{1}{\epsilon_{k-1}}
\overrightarrow{\mathbf{e}}^{(k)}
-\frac{1}{\epsilon_k}
\overleftarrow{\mathbf{e}}^{(k)},
\label{eq:corinfomax-state-dynamics}\\
\mathbf{h}^{(k)}={}&\Pi_{\mathcal{P}^{(k)}}
\!\left(\mathbf{u}^{(k)}\right),
\label{eq:corinfomax-state-projection}
\end{align}
where $s$ indexes the recurrent relaxation, $\gamma_h$ depends on the
correlation forgetting factor, and $\Pi_{\mathcal{P}^{(k)}}$ enforces the
chosen activity domain. In our experiments the hidden projection is an
elementwise ReLU clipped to $[0,50]$; the output state is unbounded.
The top-down drive is disabled, while basal and lateral interactions remain
active. These choices specialize the general CorInfoMax dynamics.

Equation~\eqref{eq:corinfomax-state-dynamics} can be read as a
multi-compartment neuron model. The drive
$\mathbf{W}_{\mathrm{ff}}^{(k-1)}\mathbf{h}^{(k-1)}$ enters through a basal
pathway, whereas
$\mathbf{W}_{\mathrm{fb}}^{(k)}\mathbf{h}^{(k+1)}$ and the lateral signal
$\mathbf{C}^{(k)}\mathbf{h}^{(k)}$ contribute through an apical pathway. The
somatic state integrates these terms subject to leakage and the activity
constraint. Under this interpretation, the two prediction errors are voltage
mismatches between the soma and its basal or apical prediction. Off-diagonal
entries of the lateral operator can be implemented through inhibitory
interneurons, while its diagonal component acts as a self-interaction. These
architectural interpretations are inherited from CorInfoMax.

Unlike CorInfoMax-EP, our training pass does not introduce a second,
label-nudged phase. The recurrent system performs a fixed number of
label-free iterations: fifteen for the fully connected model and three for
the locally connected model. The final states need not be equilibria.
The last layer is then interpreted as logits $\mathbf{a}=\mathbf{h}^{(L)}$
and used to form the supervised score.

\subsubsection{Cross-entropy score-broadcast learning}
\label{app:corinfomax-score}

For one-hot labels $\mathbf{y}$, the classifier output and its score are
\begin{equation}
\mathbf{p}=\softmax(\mathbf{a}),
\qquad
\boldsymbol{\delta}_{\mathrm{CE}}
=\nabla_{\mathbf{a}}\mathcal{L}_{\mathrm{CE}}
=\mathbf{p}-\mathbf{y}.
\label{eq:corinfomax-ce-score}
\end{equation}
This is the unregularized, unsmoothed score used in the locally connected
experiment. The fully connected experiment modifies it as specified below.
For the base construction, on minibatch $m$, let
$\mathbf{\Delta}_{\mathrm{CE}}[m]\in\mathbb{R}^{D\times B}$ collect these
scores. For every hidden layer, define a feature transform
$g^{(k)}$ and the matrix
\begin{equation}
\mathbf{G}^{(k)}[m]
=\bigl[g^{(k)}(\mathbf{h}^{(k)}[mB+1]),\ldots,
g^{(k)}(\mathbf{h}^{(k)}[mB+B])\bigr].
\label{eq:corinfomax-feature-batch}
\end{equation}
The adaptive score-projection matrix is updated by
\begin{equation}
\widehat{\mathbf{R}}_{g\delta}^{(k)}[m]
=\lambda_d\widehat{\mathbf{R}}_{g\delta}^{(k)}[m-1]
+\frac{1-\lambda_d}{B}\mathbf{G}^{(k)}[m]
\mathbf{\Delta}_{\mathrm{CE}}[m]^T,
\label{eq:corinfomax-score-correlation}
\end{equation}
and converts the global score into a layer-width modulation,
\begin{equation}
\mathbf{Q}^{(k)}[m]
=\widehat{\mathbf{R}}_{g\delta}^{(k)}[m]
\mathbf{\Delta}_{\mathrm{CE}}[m].
\label{eq:corinfomax-score-projection}
\end{equation}
If $\mathbf{F}_{d}^{(k)}$ contains the local gates for the state projection
and $\mathbf{G}_{d}^{(k)}$ contains the elementwise derivatives of
$g^{(k)}$, the local postsynaptic factor used by plasticity is
\begin{equation}
\mathbf{\Phi}^{(k)}[m]
=\mathbf{F}_{d}^{(k)}[m]\odot
\mathbf{G}_{d}^{(k)}[m]\odot
\mathbf{Q}^{(k)}[m].
\label{eq:corinfomax-three-factor-signal}
\end{equation}
This is the recurrent-network counterpart of the SBD three-factor term in the
main paper: a presynaptic state is multiplied by a postsynaptic sensitivity
and a projected, loss-derived modulatory signal. Score expansion can be used
without changing the construction by replacing
$\mathbf{\Delta}_{\mathrm{CE}}$ with the expanded score matrix of
Appendix~\ref{app:ce_score_expansion}.
Both experiments use $g^{(k)}(\mathbf h)=\mathbf h$ and uncentered
hidden states in these correlations. Local updates hold inference states
and broadcast signals fixed.

\subsubsection{Predictor and lateral plasticity}
\label{app:corinfomax-plasticity}

The CorInfoMax and SBD contributions are combined at the synapses rather than
in the state objective. For an interior layer $k$, a convenient minibatch form
of the incoming feedforward update is
\begin{align}
\Delta\mathbf{W}_{\mathrm{ff}}^{(k-1)}
={}&\frac{\mu_{\mathrm{ff},k}}{B}
\overrightarrow{\mathbf{E}}^{(k)}
\mathbf{H}^{(k-1)T}
-\frac{\mu_{\mathrm{SBD,ff},k}}{B}
\mathbf{\Phi}^{(k)}\mathbf{H}^{(k-1)T},
\label{eq:corinfomax-sbd-forward-update}
\end{align}
while the top-down predictor is updated by
\begin{align}
\Delta\mathbf{W}_{\mathrm{fb}}^{(k)}
={}&\frac{\mu_{\mathrm{fb},k}}{B}
\overleftarrow{\mathbf{E}}^{(k)}
\mathbf{H}^{(k+1)T}
-\frac{\mu_{\mathrm{SBD,fb},k}}{B}
\mathbf{\Phi}^{(k)}\mathbf{H}^{(k+1)T}.
\label{eq:corinfomax-sbd-backward-update}
\end{align}
The first term in each equation improves prediction between neighboring
layers. The second term is the cross-entropy SBD correction and has the same
three-factor structure for both pathways. Because the two matrices receive
their own residuals and have independent learning rates, these rules do not
create a hidden transpose relation between the feedforward and feedback
connections.

The lateral operator simultaneously tracks layer statistics and participates
in score-driven plasticity. With
$\mathbf{Z}^{(k)}=\mathbf{C}^{(k)}\mathbf{H}^{(k)}$, its update can be written
schematically as
\begin{equation}
\begin{aligned}
\mathbf{C}^{(k)}[m]
={}&\frac{1}{\lambda_C}\mathbf{C}^{(k)}[m-1]
-\frac{\gamma_C}{B}\mathbf{Z}^{(k)}[m]\mathbf{Z}^{(k)}[m]^T\\
&-\frac{\mu_{\mathrm{SBD,lat},k}}{B}
\left(\mathbf{\Phi}^{(k)}[m]\mathbf{H}^{(k)}[m]^T
+\mathbf{H}^{(k)}[m]\mathbf{\Phi}^{(k)}[m]^T\right).
\end{aligned}
\label{eq:corinfomax-sbd-lateral-update}
\end{equation}
The first two terms are motivated by online inverse-correlation adaptation;
the last is the symmetric score-driven correction. The latter is disabled
in the fully connected experiment and enabled in the locally connected
experiment, where all lateral products are restricted to the specified
local connections.

In the base construction, the output uses the unexpanded score. Equivalently,
one may set the score projection to the identity and $g^{(L)}$ to the identity,
so that the supervised component of the incoming-weight update reduces to
the usual outer product
$-B^{-1}\mathbf{\Delta}_{\mathrm{CE}}\mathbf{H}^{(L-1)T}$.

\subsubsection{Single-phase training cycle and relation to prior variants}
\label{app:corinfomax-training-cycle}

For each minibatch, CorInfoMax-SBD performs the following sequence:
\begin{enumerate}
\item Initialize the recurrent states and run the fixed number of free-state
iterations using the current active pathways.
\item Read the final-layer logits, evaluate
$\mathbf{p}=\softmax(\mathbf{a})$, and form the experiment's score.
\item Update each layer's score correlation
$\widehat{\mathbf{R}}_{g\delta}^{(k)}$, then compute
$\mathbf{Q}^{(k)}$ and $\mathbf{\Phi}^{(k)}$.
\item Apply the neighboring-layer prediction updates together with the SBD
terms in Equations~\eqref{eq:corinfomax-sbd-forward-update}--
\eqref{eq:corinfomax-sbd-lateral-update} on active pathways; update the
output weights from the unexpanded score.
\end{enumerate}

In both experiments the neighboring-layer prediction-plasticity coefficient
is $10^{-18}$, making its contribution negligible. Training also uses
activity sparsity and power control, momentum, and periodic bounds on
feedforward, broadcast, and lateral weight norms. An activity-power
controller reduces supervised update rates and damps momentum when activity
exceeds the layer thresholds. The power-control update includes a surrogate
gate at the upper clipping bound so that saturated units can recover.
These numerical controls are part of the experimental recipe; their exact
settings are provided in the accompanying launch scripts.

Table~\ref{tab:corinfomax-variant-comparison} summarizes the methodological
boundary between the three related models, particularly the source and
timing of supervision.

\begin{table}[h!]
\centering
\caption{Supervision mechanisms in related CorInfoMax training schemes.}
\label{tab:corinfomax-variant-comparison}
\small
\setlength{\tabcolsep}{4pt}
\begin{tabularx}{\linewidth}{l l l X}
\toprule
Method & Supervised signal & State relaxations & Synaptic credit signal \\
\midrule
CorInfoMax-EP~\citep{bozkurt2023correlative}
& MSE label nudge & Free and nudged
& Difference of local statistics between equilibria \\
CorInfoMax-EBD~\citep{erdoganerror}
& MSE output error & Free only
& Projected error--feature correlation \\
CorInfoMax-SBD (ours)
& CE-based score & Free only
& Projected score--feature correlation \\
\bottomrule
\end{tabularx}

\end{table}

The experimental models retain CorInfoMax-inspired basal and lateral
dynamics and use single-phase score-broadcast plasticity, without a
label-nudged inference pass.

\subsection{The 10-stage fully connected CorInfoMax experiment}

\subsubsection{Architecture of the 10-stage fully-connected model}
\label{app:cim-ten-stage-architecture-batch1}

The 10-stage fully connected model has layer dimensions
\begin{equation}
 [3072,\;4000,\;500,\;500,\;500,\;500,\;500,\;500,\;500,\;500,\;10].
\label{eq:cim-v6-batch1-architecture-vector}
\end{equation}

The input is a flattened CIFAR-10 image with $3\times32\times32=3072$
coordinates.  The network has a 4000-unit first hidden state, eight
successive 500-unit hidden states, and a ten-unit classifier state.  Thus the
architecture vector has eleven state dimensions when the input is included,
and ten weight-bearing transformations:
\begin{equation}
 3072\rightarrow4000\rightarrow
 \underbrace{500\rightarrow\cdots\rightarrow500}_{8\ \text{hidden states}}
 \rightarrow10.
\end{equation}

The hidden states are assigned the spatial shapes
\begin{equation}
 3072=3\times32\times32,\qquad
 4000=10\times20\times20,\qquad
 500=5\times10\times10.
\label{eq:cim-v6-batch1-spatial-shapes}
\end{equation}
These shapes are used only to construct the local initialization masks; the
model stores the states as vectors.

The feedforward matrices have the standard CorInfoMax orientation
\begin{equation}
 \mathbf{W}_{\mathrm{ff}}^{(k)}
 \in\mathbb{R}^{N^{(k+1)}\times N^{(k)}},
 \qquad k=0,\ldots,9,
\end{equation}
and the model also allocates independently parameterized feedback predictors
$\mathbf{W}_{\mathrm{fb}}^{(k)}$ and lateral CorInfoMax operators
$\mathbf{B}^{(k)}$ (the $\mathbf C^{(k)}$ operators above). Score-broadcast
matrices $\mathbf{W}_{\mathrm{eb}}$ implement the adaptive score correlations.
The three-block expansion gives a hidden broadcast dimension of $30$.

\paragraph{Score used in the fully connected experiment.}
With $\bar a=10^{-1}\sum_j a_j$, the output score is
\[
 \boldsymbol\delta_{\mathrm{FC}}
 =\mathbf p-\left(0.97\mathbf y+0.03\mathbf1_{10}/10\right)
   +\frac{0.01}{10}(\mathbf a-\bar a\mathbf1_{10}).
\]
This is the derivative of label-smoothed CE plus the centered-logit penalty
$0.01\|\mathbf a-\bar a\mathbf1_{10}\|_2^2/(2\cdot10)$.
The hidden broadcast concatenates
$[\boldsymbol\delta_{\mathrm{FC}};
\mathbf p\odot\boldsymbol\delta_{\mathrm{FC}};
\operatorname{roll}_5(\mathbf p)\odot\boldsymbol\delta_{\mathrm{FC}}]$
and subtracts its running mean, updated with the same forgetting factor as
the score correlations. The output uses the uncentered, unexpanded score.
The additional logit penalty and running-mean subtraction are experimental
modifications; the conditional mean-zero theorem for the raw CE score does
not directly establish orthogonality for this modified broadcast.

\paragraph{Sparse local initialization.}
Local connectivity is enabled for feedforward stages $0$ through $8$, with
kernel sizes
\begin{equation}
 (11,\;7,\;5,\;5,\;5,\;5,\;5,\;5,\;5).
\end{equation}
The final $500\rightarrow10$ classifier, stage $9$, remains dense.  For each
masked stage, the implementation aligns the source and destination spatial
grids, retains only the corresponding local receptive-field entries, and
initializes the allowed edges using a mask-aware Xavier scale.  The local pattern sparsifies the
initial weights but is not reimposed after local plasticity updates.  This is
local structural initialization, not a permanently sparse LC network.

\paragraph{Dynamics and connection pathways.}
The basal pathway is the feedforward predictor, the apical pathway is the
feedback predictor, and the lateral pathway is the adaptive CorInfoMax
operator.  The active experiment uses
the feedforward, lateral, and CE broadcast pathways, and does not use feedback weights.  Each batch-1 sample is
processed with fifteen free neural-dynamics iterations, followed by local
CorInfoMax/SBD updates. The fully connected experiment additionally uses
an Oja-type activity-dependent weight shrinkage. The implementation performs
no backpropagation and does not use an autograd optimizer.

\subsubsection{Training settings}
Table~\ref{tab:cim-v6-batch1-ten-stage-hyperparameters} summarizes the
settings for the 10-layer fully-connected CorInfoMax experiments.
\begin{table}[t]
\centering
\small
\caption{Hyperparameters for 10-layer fully connected CorInfoMax network.}
\label{tab:cim-v6-batch1-ten-stage-hyperparameters}
\begin{tabular}{p{0.33\linewidth}p{0.59\linewidth}}
\hline
\textbf{Quantity} & \textbf{Value} \\
\hline
Dataset and batch & CIFAR-10; batch size $B=1$ \\
Architecture & $[3072,4000,500,500,500,500,500,500,500,500,10]$ \\
Trainable stages & 10 feedforward stages; 9 hidden states and 1 classifier \\
Objective & Smoothed CE plus centered-logit penalty; three-block expansion \\
CE temperature & $1.0$ \\
Label smoothing & $0.03$ \\
Free neural iterations & 15 per sample update \\
Feedback predictor pathway & Disabled\\
Broadcast weight ($\mathbf{W}_{eb}$) gain & $1.0$ Xavier gain \\
Supervised update reference coefficients & $(3.5,3.5,3.5,3.5,3.5,3.5,1.75,1.75,1.75,60000)$ \\
Lateral forgetting & $\lambda=0.999999949999981$ \\
Broadcast weight forgetting & $\lambda_{\mathrm{eb}}:0.999999687499268\rightarrow0.999999937499971$ \\
Schedule & Coupled cosine; warmup $1$, decay $20\rightarrow100$ epochs, floor $0.50$ \\
Direct/output scales & Output error scale $0.0025$; lateral output scale $0.07$; broadcast scale $1.0$ \\
Local initialization & Enabled on feedforward stages $0$--$8$; final classifier dense \\
Persistent mask & Disabled; local sparsity is initialization-only \\
Augmentation & Enabled from epoch 8; reflect padding, crop padding 2, horizontal flip \\
Activity controls & Scale $0.15$; power ceilings $(1,12,14,14,14,14,8,4,1.25,5)$ \\
\hline
\end{tabular}
\end{table}

\newpage

Figure~\ref{fig:cim-v6-batch1-accuracy} shows the train and test accuracy
trajectories for the 10-layer fully connected CorInfoMax-SBD model. Both
curves rise rapidly during the first epochs and pass approximately $50\%$
accuracy early in training. Test accuracy subsequently improves more
gradually and approaches $60\%$, whereas training accuracy continues to
increase to approximately $71\%$. The standard-deviation envelopes are widest
during the early transition and become narrower later in training; the
increasing train--test separation indicates a growing generalization gap after
the initial rapid improvement.

\begin{figure}[h!]
\centering
\includegraphics[width=0.95\linewidth]{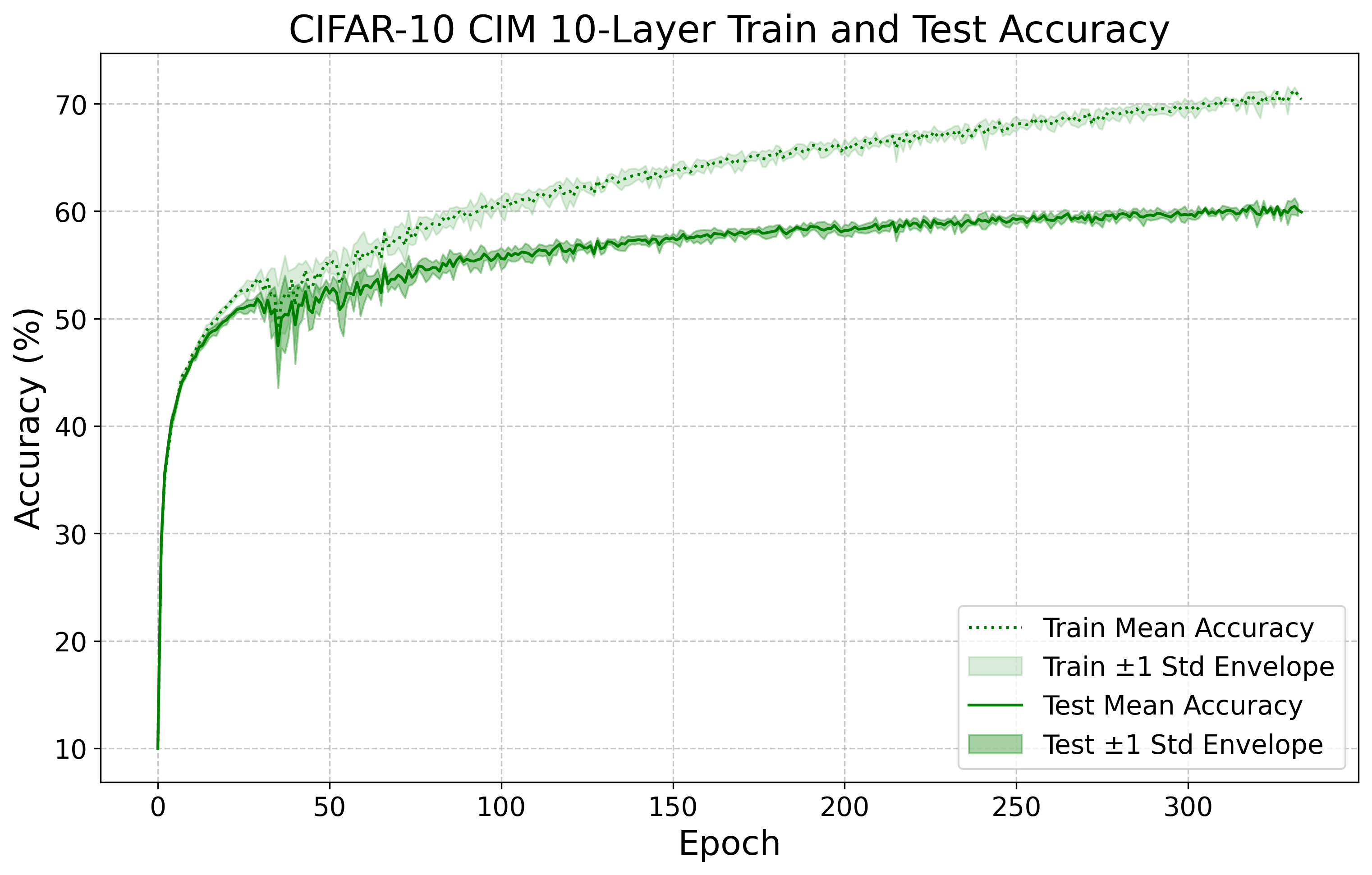}
\caption{Training and test accuracy of the 10-layer fully connected
CorInfoMax-SBD model on CIFAR-10. Curves show the mean over five independent
seed runs, and shaded regions indicate one standard deviation. The dotted and
solid curves denote training and test accuracy, respectively.}
\label{fig:cim-v6-batch1-accuracy}
\end{figure}

\newpage
\subsection{Locally Connected CorInfoMax Experiments}
\label{app:cim-lc}

Convolutional neural networks (CNNs) obtain much of their efficiency and visual
inductive bias by applying the same learned filters at every spatial location.
This exact weight sharing is difficult to reconcile with biologically local
plasticity, because it would require distant synapses to maintain identical
values without a known mechanism for transporting weights between them.
Locally connected networks (LCNs) preserve spatially restricted receptive
fields but assign independent filters to different locations, thereby removing
the weight-sharing requirement. The resulting flexibility comes with an
empirical challenge: biologically motivated learning methods have generally
scaled less successfully in locally connected architectures than
backpropagation in convolutional models on demanding visual tasks
\citep{bartunov2018assessing}. Pogodin et al. showed that lateral interactions
and a sleep-like Hebbian phase can coordinate initially unshared local filters
and recover performance close to that of convolutional networks
\citep{pogodin2021towards}. The CorInfoMax experiments below pursue the same
broad objective through a different normative mechanism, combining local
receptive fields with correlative-information and score-broadcast updates
without imposing exact convolutional weight tying.

\subsubsection{The locally connected CorInfoMax experiment}
The locally connected CorInfoMax implementation contains eight
spatially locally connected stages, followed by a dense 512-unit hidden stage
and a dense ten-class output stage: nine hidden stages and ten trainable
transformations in total. Its architecture vector is
\begin{equation}
 [3072,\;16384,\;2048,\;512,\;1024,\;1024,\;512,\;768,\;512,\;512,\;10].
\label{eq:cim-v11-architecture}
\end{equation}

The detailed LC topology is:
\begin{center}
\small
\begin{tabular}{c|c|c|c|c|c}
\hline
Stage & Input shape & Output shape & Kernel & Stride & Padding \\
\hline
LC0 & $3\times32\times32$ & $64\times16\times16$ & $5$ & $2$ & $2$ \\
LC1 & $64\times16\times16$ & $32\times8\times8$ & $5$ & $2$ & $2$ \\
LC2 & $32\times8\times8$ & $32\times4\times4$ & $2$ & $2$ & $0$ \\
LC3 & $32\times4\times4$ & $64\times4\times4$ & $3$ & $1$ & $1$ \\
LC4 & $64\times4\times4$ & $64\times4\times4$ & $3$ & $1$ & $1$ \\
LC5 & $64\times4\times4$ & $128\times2\times2$ & $3$ & $2$ & $1$ \\
LC6 & $128\times2\times2$ & $192\times2\times2$ & $1$ & $1$ & $0$ \\
LC7 & $192\times2\times2$ & $128\times2\times2$ & $1$ & $1$ & $0$ \\
\hline
\end{tabular}
\end{center}
The resulting $128\times2\times2=512$ representation is processed by a
dense $512\rightarrow512$ stage and a dense $512\rightarrow10$ classifier.
LC6 and LC7 are unshared pointwise channel transformations at each spatial
position; they increase depth without expanding the spatial support.

The lateral kernels are $(5,5,2,3,3,3,3,3)$ for LC0--LC7, with
same-channel lateral connectivity.  Feedforward, feedback, and lateral
updates use the LC index maps, so their outer products remain local.  The run
uses the raw, uncentered ten-dimensional CE score and a direct CE
classifier update without momentum filtering. It performs three free neural-dynamics iterations per
minibatch, uses batch size one, and disables the active feedback pathway with
\texttt{include\_backv=0}.
Learned biases enter the hidden projection as
$\mathbf h=\operatorname{clip}(\mathbf u+\mathbf b,0,50)$, while the
plasticity gate is evaluated at $\mathbf u$ before adding the bias. Thus
this gate is a local surrogate for the biased projection's derivative.

Figure~\ref{fig:cim-lc9-accuracy} shows the train and test accuracy
trajectories for the locally connected CorInfoMax-SBD model. 
Test accuracy settles around  $72\%$, whereas training accuracy reaches $77.5\%$. The standard-deviation envelopes are widest
during the early transition and become narrower later in training; the
increasing train--test separation indicates a growing generalization gap after
the initial rapid improvement.

\begin{figure}[h!]
\centering
\includegraphics[width=0.95\linewidth]{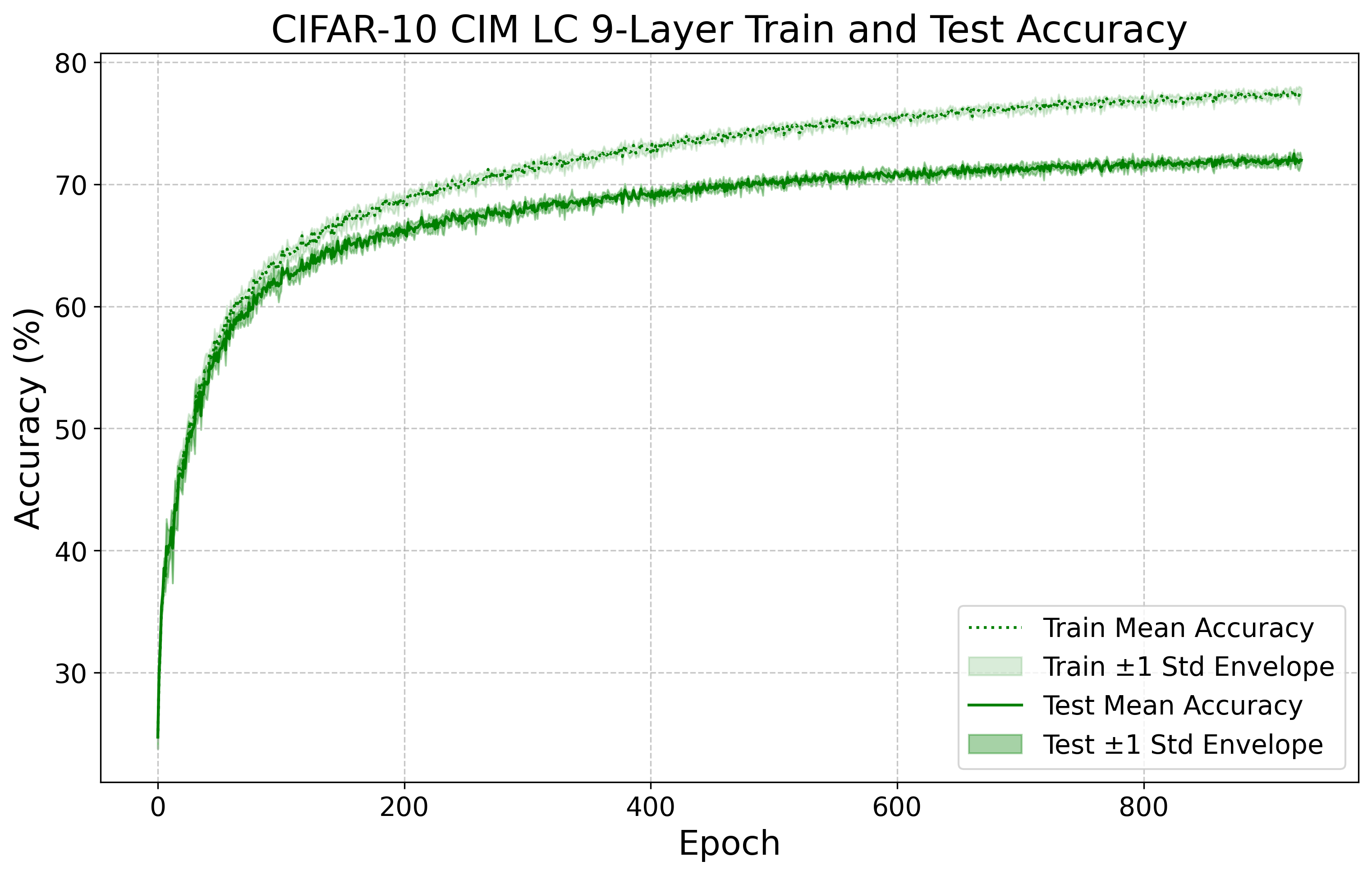}
\caption{Training and test accuracy of the locally connected
CorInfoMax-SBD model on CIFAR-10. Curves show the mean over five independent
seed runs, and shaded regions indicate one standard deviation. The dotted and
solid curves denote training and test accuracy, respectively.}
\label{fig:cim-lc9-accuracy}
\end{figure}

\begin{table}[t]
\centering
\small
\caption{Hyperparameters for the locally connected CorInfoMax network.}
\label{tab:cim-v11-lc8-hyperparameters}
\begin{tabular}{p{0.34\linewidth}p{0.58\linewidth}}
\hline
\textbf{Quantity} & \textbf{Value} \\
\hline
Dataset and batch & CIFAR-10; $B=1$ \\
LC stages & 8 unshared stages; dense $512\rightarrow512\rightarrow10$ head \\
Objective  & Cross-entropy \\
Free neural iterations & $T=3$ \\
Lateral kernel sizes & $(5,5,2,3,3,3,3,3)$; same-channel lateral maps \\
Lateral broadcast reference coefficient & $0.32$ \\
Supervised update reference coefficients & $(50,50,50,35,30,25,20,15,50,60000)$ \\
Forgetting factors & $\lambda_{\mathrm{eb}}:0.999999\rightarrow0.99999999$; $\lambda=0.999999949999981$ \\
Learning-rate schedule & Coupled cosine; warmup $1$, decay $20\rightarrow800$ epochs, floor $0.50$ \\
Hidden-update momentum & Increases from $0.99$ toward $0.999$ \\
Initialization & LC generic Xavier; local lateral identity initialization disabled; fan-aware initialization disabled \\
Augmentation & Enabled from epoch 0; constant padding 4 and horizontal flip \\
 Activity controls & Power ceilings $(2,12,14,4,4,4,4,4,4,12)$; layer targets $(1,1,1,1.25,1.25,1.5,1.5,1.5,2,0.1)$ \\
\hline
\end{tabular}
\end{table}

The table lists reference coefficients, rather than effective parameter
step sizes: they are combined with the feedforward step scale, cosine
schedule, and activity-dependent controller in the supplied implementation.
The classifier additionally scales its direct CE update by $0.0025$.
Both CorInfoMax experiments use the coupled schedule to warm up the step
size, hold it through epoch $20$, and then decay to half its reference scale.

\section{SBD in a residual network}
\label{app:resnet-experiments}

We evaluate an SBD residual-network variant on a wide CIFAR ResNet-20.
Hidden-layer learning uses the score-decorrelation construction of
Appendices~\ref{app:ce_sbd_derivation} and
\ref{app:cnn_sbd_derivation}, with pooled activation features, contextual
score expansion, and local update scaling. This appendix specifies the architecture,
the activation transform \(\mathbf g^{(k)}\) used in its score correlations,
the local derivatives through residual addition and batch normalization
(BN), and the expanded score and training settings used in this experiment.
Hidden layers use only the local score-decorrelation signals described below;
there is no auxiliary supervised hidden-layer loss.

\subsection{Residual architecture and local updates}
\label{app:resnet-architecture}

We use post-activation residual blocks \citep{he2016deep} with three blocks
per stage and widths \(128,256,512\), respectively
(Table~\ref{tab:sbd-resnet-architecture}). This is an eightfold widening of
the CIFAR architecture, including its stem, rather than the standard
wide-residual-network configuration \citep{zagoruyko2016wide}.
We retain \(u\) for preactivations. In the block equations below,
\(\mathbf u\) denotes the raw convolutional preactivation tensor;
\(\widehat{\mathbf u}\) and \(\widetilde{\mathbf u}\) denote
its normalized value and BN output, respectively
(Eq.~\eqref{eq:sbd-resnet-bn-forward}). We write \(f\) for the elementwise
ReLU activation, \(f(a)=\max\{0,a\}\), with \(f'(0)=0\).
For a block input \(\mathbf x\),
\begin{align}
    \mathbf u_1&=\mathbf x\ast\mathbf W_1,
    &\widetilde{\mathbf u}_1&=\operatorname{BN}_1(\mathbf u_1),
    &\mathbf H_1&=f(\widetilde{\mathbf u}_1),\nonumber\\
    \mathbf u_2&=\mathbf H_1\ast\mathbf W_2,
    &\widetilde{\mathbf u}_2&=\operatorname{BN}_2(\mathbf u_2),
    &\mathbf H&=f(\mathbf s+\widetilde{\mathbf u}_2).
    \label{eq:sbd-resnet-block}
\end{align}
Both main convolutions have \(3\times3\) kernels and padding one.
The shortcut is \(\mathbf s=\mathbf x\), except at stage transitions,
where \(\mathbf u_s=\mathbf x\ast\mathbf W_s\) and
\(\mathbf s=\widetilde{\mathbf u}_s=\operatorname{BN}_s(\mathbf u_s)\)
define a learned \(1\times1\) projection followed by BN. Both the first
main convolution and the projection have stride two at stage transitions;
all other strides are one. There is no activation between the second BN and addition.

\begin{table}[ht]
\centering
\caption{Wide CIFAR ResNet-20. Shapes exclude the minibatch dimension.}
\label{tab:sbd-resnet-architecture}
\small
\begin{tabular}{@{}lll@{}}
\toprule
Part & Operation & Output shape \\
\midrule
Input & Normalized RGB image & \(3\times32\times32\) \\
Stem & \(3\times3\) convolution; BN; \(f\) & \(128\times32\times32\) \\
Stage 1 & Three residual blocks & \(128\times32\times32\) \\
Stage 2 & Three blocks; first downsamples & \(256\times16\times16\) \\
Stage 3 & Three blocks; first downsamples & \(512\times8\times8\) \\
Readout & Global average pooling; affine classifier & \(10\) logits \\
\bottomrule
\end{tabular}
\end{table}

Convolutions are bias-free; BN has learned scale and offset, and the
classifier has a bias. There is no initial max-pooling layer. The network
has \(17{,}289{,}866\) trainable parameters. Its conventional depth count
comprises the stem, \(18\) main-path convolutions, and classifier; the two
projection convolutions give \(21\) trainable convolutions in total,
indexed by \(k\) below.
Convolutional weights use Kaiming-normal initialization with fan-in scaling
for ReLU~\citep{he2015rectifiers}. Classifier weights have independent
\(\mathcal N(0,1/512)\) entries, and its bias starts at zero.
BN initialization is specified in Section~\ref{app:resnet-bn}.

\paragraph{Score--activation correlations.}
Each convolution maintains its own running correlation
\(\widehat{\widetilde{\mathbf R}}^{(k)}\) between transformed activations
\(\mathbf g^{(k)}(\mathbf H^{(k)})\) and the expanded output score
\(\widetilde{\boldsymbol\delta}\), following
Appendix~\ref{app:score_expansion_appendix}. The expanded score is defined
in Section~\ref{app:resnet-context}.
For the stem and each first block convolution, \(\mathbf H^{(k)}\) is
the convolution's own activation after BN and \(f\). The second convolution
and the projection instead use the merged block output \(\mathbf H\)
in Eq.~\eqref{eq:sbd-resnet-block}, with a separate correlation matrix
for each convolution.

For example \(i\), write \(\mathbf H_i^{(k)}\in
\mathbb R^{P^{(k)}\times M^{(k)}\times N^{(k)}}\), using the channel and
spatial dimensions of Appendix~\ref{app:cnn_sbd_derivation}.
Here \(\mathbf g^{(k)}\) combines spatial averages of the activations
and their squares. Let \(A_4\) average each channel into a
\(4\times4\) grid, and set
\begin{equation}
    \mathbf g^{(k)}(\mathbf H)
    =\operatorname{vec}\!\left[
       A_4\mathbf H;\;\frac{\xi}{\sqrt2}A_4(\mathbf H\odot\mathbf H)
     \right],
    \qquad \xi=0.5.
    \label{eq:sbd-resnet-features}
\end{equation}
Each channel contributes sixteen averages and sixteen mean-square values,
giving \(32P^{(k)}\) coordinates per example. Squaring is performed before
pooling. As in Appendix~\ref{app:ce_sbd_derivation}, collect examples as
columns of \(\mathbf G^{(k)}\in\mathbb R^{32P^{(k)}\times B}\) and
\(\widetilde{\boldsymbol\Delta}\in\mathbb R^{110\times B}\).
At minibatch \(m=1,2,\ldots\), update the correlation estimate and then project the
expanded score:
\begin{equation}
\begin{aligned}
    \widehat{\widetilde{\mathbf R}}^{(k)}[m]
    &=\lambda\widehat{\widetilde{\mathbf R}}^{(k)}[m-1]
     +\frac{1-\lambda}{B}\mathbf G^{(k)}[m]
       \widetilde{\boldsymbol\Delta}[m]^{\top},\\
    \widetilde{\mathbf Q}^{(k)}[m]
    &=\widehat{\widetilde{\mathbf R}}^{(k)}[m]\mathbf M[m]
       \widetilde{\boldsymbol\Delta}[m].
\end{aligned}
    \label{eq:sbd-resnet-moment}
\end{equation}
The matrices \(\widehat{\widetilde{\mathbf R}}^{(k)}\in
\mathbb R^{32P^{(k)}\times110}\) are independently initialized with
\(\mathcal N(0,0.01^2)\) entries at \(m=0\). These are uncentered correlations, as in
the preceding SBD derivation. Each projected score uses the updated
full-minibatch correlation, including that example's own contribution.
The shared positive diagonal matrix
\(\mathbf M\) rescales the score coordinates and is specified in
Section~\ref{app:resnet-scaling}.
We suppress the minibatch index below. Using the scaling factor \(c_k\)
specified in Section~\ref{app:resnet-scaling}, the local updates use the
following weighted and rescaled version of the expanded-score objective
in Appendix~\ref{app:score_expansion_appendix}:
\[
    \widetilde{\mathcal J}_{\mathrm{Score}}^{(k)}
    =\frac{c_k}{2(1-\lambda)M^{(k)}N^{(k)}}
      \operatorname{tr}\!\left(
      \widehat{\widetilde{\mathbf R}}^{(k)}\mathbf M
      \widehat{\widetilde{\mathbf R}}^{(k)\top}\right).
\]

Reshape the two parts of column \(i\) of \(\widetilde{\mathbf Q}^{(k)}\) into
\(P^{(k)}\times4\times4\) grids, denoted
\(\widetilde{\mathbf Q}_{i,\mathrm{lin}}^{(k)}\) and
\(\widetilde{\mathbf Q}_{i,\mathrm{sq}}^{(k)}\). Differentiating
\(\mathbf g^{(k)}\) gives the direction at the full activation map:
\begin{equation}
    \mathbf q_{H,i}^{(k)}
    =A_4^*\widetilde{\mathbf Q}_{i,\mathrm{lin}}^{(k)}
     +\sqrt2\,\xi\,\mathbf H_i^{(k)}\odot
       A_4^*\widetilde{\mathbf Q}_{i,\mathrm{sq}}^{(k)}.
    \label{eq:sbd-resnet-pullback}
\end{equation}
Here \(A_4^*\) copies each grid value to the positions in its pooling cell
and divides by the cell area. The second term includes the derivative of
the squared activations. Thus \(\mathbf q_H^{(k)}\), stacked over the
minibatch, replaces \(\mathbf g'^{(k)}(\mathbf h^{(k)})\odot
\widetilde{\mathbf q}^{(k)}\) from the elementwise derivation: it includes the
derivative of \(\mathbf g^{(k)}\), before the activation and BN derivatives.

We use the local-derivative notation of
Appendix~\ref{app:ce_sbd_derivation} throughout: the subscript
\(\mathrm{local}\) means that only the dependence through the current
convolution's activations is differentiated. The broadcast score, previous
correlation estimate, convolution input, \(\mathbf M\), and \(c_k\) are
held fixed. The objective's scaling converts the correlation derivative's
prefactor \((1-\lambda)/B\) to \(c_k/[BM^{(k)}N^{(k)}]\).
Thus the local gradients average over batch and spatial positions, with
no channel-count divisor. This reduction is additional to the cell-area
division in \(A_4^*\).

\paragraph{Local derivatives through the residual block.}
For one channel and batch/spatial position \(n\), the local activation
derivatives are
\begin{equation}
\begin{aligned}
    \frac{\partial H_{1,n}}{\partial\widetilde u_{1,n}}
      &=f'(\widetilde u_{1,n}),\\
    \frac{\partial H_n}{\partial\widetilde u_{2,n}}
      =\frac{\partial H_n}{\partial\widetilde u_{s,n}}
      &=f'(s_n+\widetilde u_{2,n}).
\end{aligned}
    \label{eq:sbd-resnet-routing}
\end{equation}
Each convolution uses its own projected score. The second convolution
holds \(\mathbf H_1,\mathbf s\) fixed, and the projection holds
\(\mathbf x,\widetilde{\mathbf u}_2\) fixed. Addition has derivative one
with respect to either branch, so both evaluate \(f'\) at the summed input.
The first convolution instead receives its own signal at \(\mathbf H_1\);
the second convolution's signal does not propagate through it. An identity
shortcut has no parameters to update.
The stem uses \(f'(\widetilde u_n)\), as does the first convolution.

\subsection{Batch normalization}
\label{app:resnet-bn}

We use standard BN \citep{ioffe2015batch}, with one learned scale
\(\gamma\) and offset \(\beta\) per channel. For one fixed channel of
convolution \(k\), flatten the example and spatial indices into
\(n,j=1,\ldots,BM^{(k)}N^{(k)}\). Thus \(H_n^{(k)}\) and
\(q_{H,n}^{(k)}\) denote scalar entries of \(\mathbf H^{(k)}\) and
\(\mathbf q_H^{(k)}\) in this channel. Suppressing the channel and layer indices on the BN
quantities, the forward computation is
\begin{equation}
\begin{aligned}
    \mu&=\frac{1}{BM^{(k)}N^{(k)}}\sum_n u_n,\qquad
    v=\frac{1}{BM^{(k)}N^{(k)}}\sum_n(u_n-\mu)^2,\\
    \sigma&=\sqrt{v+\epsilon_{\rm BN}},\qquad
    \widehat u_n=\frac{u_n-\mu}{\sigma},\qquad
    \widetilde u_n=\gamma\widehat u_n+\beta.
\end{aligned}
    \label{eq:sbd-resnet-bn-forward}
\end{equation}
Here \(\epsilon_{\rm BN}=10^{-5}\), and the training variance uses the
\(BM^{(k)}N^{(k)}\) denominator.

\paragraph{Scale and offset updates.}
The update formulas below use \(f'(\widetilde u_n)\), as appropriate for
the stem and first block convolution. For the second convolution and
projection, replace this factor throughout by
\(f'(s_n+\widetilde u_{2,n})\), including inside the sums in the
convolutional weight update, following Eq.~\eqref{eq:sbd-resnet-routing}.
All BN quantities remain those of the convolution being updated.
For the minibatch mean cross-entropy loss \(\mathcal L_{\mathrm{CE}}\),
backpropagation gives
\[
    \frac{\partial\mathcal L_{\mathrm{CE}}}{\partial\gamma}
    =\sum_n\frac{\partial\mathcal L_{\mathrm{CE}}}{\partial H_n^{(k)}}
      f'(\widetilde u_n)\widehat u_n,
    \qquad
    \frac{\partial\mathcal L_{\mathrm{CE}}}{\partial\beta}
    =\sum_n\frac{\partial\mathcal L_{\mathrm{CE}}}{\partial H_n^{(k)}}
      f'(\widetilde u_n),
\]
since \(\partial\widetilde u_n/\partial\gamma=\widehat u_n\) and
\(\partial\widetilde u_n/\partial\beta=1\). Computing
\(\partial\mathcal L_{\mathrm{CE}}/\partial H_n^{(k)}\) requires differentiation
through downstream layers. SBD instead uses the local derivative of the
score-decorrelation objective,
\[
    \left.\frac{\partial\widetilde{\mathcal J}_{\mathrm{Score}}^{(k)}}
      {\partial H_n^{(k)}}\right|_{\mathrm{local}}
    =\frac{c_k}{BM^{(k)}N^{(k)}}q_{H,n}^{(k)},
\]
giving
\begin{equation}
\begin{aligned}
    \left.\frac{\partial\widetilde{\mathcal J}_{\mathrm{Score}}^{(k)}}
      {\partial\gamma}\right|_{\mathrm{local}}
    &=\frac{c_k}{BM^{(k)}N^{(k)}}\sum_n
       q_{H,n}^{(k)}f'(\widetilde u_n)\widehat u_n,\\
    \left.\frac{\partial\widetilde{\mathcal J}_{\mathrm{Score}}^{(k)}}
      {\partial\beta}\right|_{\mathrm{local}}
    &=\frac{c_k}{BM^{(k)}N^{(k)}}\sum_n
       q_{H,n}^{(k)}f'(\widetilde u_n).
\end{aligned}
    \label{eq:sbd-resnet-bn-directions}
\end{equation}
These local score-decorrelation gradients are supplied to the optimizer
for the scale and offset updates. The sums receive only the displayed
\(c_k/[BM^{(k)}N^{(k)}]\) factor; no further averaging is applied.

\paragraph{Convolutional weight updates.}
The convolutional weights determine \(u\), before normalization, so their
gradients require continuing the chain rule through \(\widehat u\).
Changing one \(u_j\) changes the channel mean and variance, and therefore
can affect every \(\widehat u_n\) in that channel. Because
\(\partial\mu/\partial u_j=1/[BM^{(k)}N^{(k)}]\) and
\(\partial v/\partial u_j=2(u_j-\mu)/[BM^{(k)}N^{(k)}]\),
\begin{equation}
    \frac{\partial\widehat u_n}{\partial u_j}
    =\frac1{\sigma}\left[
      \mathbf1\{n=j\}-\frac{1}{BM^{(k)}N^{(k)}}
      -\frac{\widehat u_n\widehat u_j}{BM^{(k)}N^{(k)}}\right].
    \label{eq:sbd-resnet-bn-jacobian}
\end{equation}
Here \(\mathbf1\{n=j\}\) is one when \(n=j\) and zero otherwise.
Using \(\partial\widetilde u_n/\partial u_j
=\gamma\,\partial\widehat u_n/\partial u_j\), the local derivative with
respect to the raw convolution output is
\[
\begin{aligned}
    \left.\frac{\partial\widetilde{\mathcal J}_{\mathrm{Score}}^{(k)}}
      {\partial u_j}\right|_{\mathrm{local}}
    =\frac{c_k\gamma}{BM^{(k)}N^{(k)}\sigma}
      \Biggl[&q_{H,j}^{(k)}f'(\widetilde u_j)\\
       &-\frac{1}{BM^{(k)}N^{(k)}}
        \sum_{n=1}^{BM^{(k)}N^{(k)}}
        q_{H,n}^{(k)}f'(\widetilde u_n)\\
       &-\frac{\widehat u_j}{BM^{(k)}N^{(k)}}
        \sum_{n=1}^{BM^{(k)}N^{(k)}}
        q_{H,n}^{(k)}f'(\widetilde u_n)\widehat u_n\Biggr].
\end{aligned}
\]
Both sums run over all batch and spatial positions in the same channel.
The subtraction terms account for the dependence of the batch
mean and variance on the convolution outputs; the expression is exact
for finite \(\epsilon_{\rm BN}\). The averages inside the brackets arise
from the BN Jacobian and do not introduce another objective-reduction factor.
For a filter coefficient in output channel \(p\), with input channel
\(h\) and kernel offsets \(a,b\), the local gradient is then
\[
    \left.\frac{\partial\widetilde{\mathcal J}_{\mathrm{Score}}^{(k)}}
      {\partial W_{h,a,b}^{(k,p)}}\right|_{\mathrm{local}}
    =\sum_j
      \left.\frac{\partial\widetilde{\mathcal J}_{\mathrm{Score}}^{(k)}}
        {\partial u_j}\right|_{\mathrm{local}}
      \frac{\partial u_j}{\partial W_{h,a,b}^{(k,p)}}.
\]
The last derivative is the corresponding cached input-patch entry, as in
Appendix~\ref{app:cnn_sbd_derivation}. The inputs are \(\mathbf x\),
\(\mathbf H_1\), and \(\mathbf x\) for the first convolution, second
convolution, and projection, respectively. No additional batch or spatial
averaging is applied. Local differentiation stops at these cached inputs.

BN initializes \(\gamma=1,\beta=0\), with running mean zero and running
variance one. Once per training forward pass, the running statistics use
new-observation weight \(0.1\): the mean receives \(\mu\), and the
variance receives
\(BM^{(k)}N^{(k)}v/[BM^{(k)}N^{(k)}-1]\).
This unbiased correction is used only for the running variance; training
normalization and its derivatives use \(v\). Local differentiation uses
the cached BN state from the same forward pass. Evaluation uses the fixed
running statistics.

\subsection{Contextual score expansion}
\label{app:resnet-context}

For example \(i\), let \(\mathbf a_i\) be the logits,
\(\mathbf p_i=\operatorname{softmax}(\mathbf a_i)\), and
\(\mathbf y_i\) the one-hot label vector. The raw ten-class CE score is
\(\boldsymbol\delta_i=\mathbf p_i-\mathbf y_i\), at unit temperature.
We use the deterministic-modulator
construction of Appendix~\ref{app:score_expansion_appendix}, with
modulators obtained from the predicted probabilities, the input image,
and the final hidden representation.

Let \(\mathbf x_{4,i}\in\mathbb R^{48}\) be the flattened
\(4\times4\) average-pooled, normalized RGB input, and let
\(\mathbf h_{\mathrm{pool},i}\in\mathbb R^{512}\) be the final global-average-pooled
representation immediately before the classifier. Normalize each vector
by its own root-mean-square magnitude, concatenate them into
\(\mathbf c_i\in\mathbb R^{560}\):
\begin{equation}
\begin{aligned}
    \operatorname{rn}(\mathbf z)
    &=\frac{\mathbf z}{\sqrt{\|\mathbf z\|_2^2/\dim(\mathbf z)+10^{-5}}},\\
    \mathbf c_i
    &=[\operatorname{rn}(\mathbf x_{4,i});\operatorname{rn}(\mathbf h_{\mathrm{pool},i})].
\end{aligned}
    \label{eq:sbd-resnet-context}
\end{equation}
The matrix \(\mathbf P_{\rm ctx}\in\mathbb R^{560\times8}\) is drawn
once per run, with independent \(\mathcal N(0,1/560)\) entries, and is
fixed thereafter. The normalization in Eq.~\eqref{eq:sbd-resnet-context}
is per example and does not center the vectors.

Using the modulator notation of Appendix~\ref{app:score_expansion_appendix},
define eleven vectors \(\boldsymbol\phi_{\ell,i}\in\mathbb R^{10}\):
\[
\begin{aligned}
    \boldsymbol\phi_{1,i}&=\mathbf1_{10},\\
    \boldsymbol\phi_{2,i}&=\sqrt{10}(\mathbf p_i-\mathbf1_{10}/10),\qquad
    \boldsymbol\phi_{3,i}=\operatorname{roll}_5(\boldsymbol\phi_{2,i}),\\
    \boldsymbol\phi_{\ell,i}
    &=\sqrt2\tanh\!\left([\mathbf P_{\rm ctx}^{\top}\mathbf c_i]_{\ell-3}\right)
      \mathbf1_{10},\qquad \ell=4,\ldots,11.
\end{aligned}
\]
Here \(\mathbf1_{10}\) is the all-ones vector, and
\(\operatorname{roll}_5\) cyclically shifts a vector by five positions,
as in Appendix~\ref{app:score_expansion_appendix}. Each of the last eight
modulators repeats one bounded context scalar across all ten classes.
These scalars vary across examples and as the hidden representation changes.
The complete broadcast contains eleven ten-dimensional blocks, so the
expanded dimension is \(D_\varepsilon=11\times10=110\):
\begin{equation}
    \widetilde{\boldsymbol\delta}_i
    =\frac1{\sqrt{11}}\left[
      \boldsymbol\phi_{1,i}\odot\boldsymbol\delta_i;\ldots;
      \boldsymbol\phi_{11,i}\odot\boldsymbol\delta_i
    \right]\in\mathbb R^{110}.
    \label{eq:sbd-resnet-expansion}
\end{equation}
The first three blocks use the raw score, centered probabilities and their
cyclic shift; the last eight provide context-weighted score--activation
correlations. The common factor \(1/\sqrt{11}\) controls the overall
scale and preserves the orthogonality argument. All convolutions use the
same expanded score, while the classifier uses the raw ten-dimensional score.

At fixed network parameters, the modulators depend only on the forward
inputs, including the other minibatch inputs through BN. For independently
sampled examples, the Bayes-optimal CE score satisfies
\(\mathbb E[\boldsymbol\delta_i^\star\mid X_1,\ldots,X_B]=0\), so the
orthogonality argument of Appendix~\ref{app:score_expansion_appendix}
applies with the full minibatch inputs conditioned on.
In the local hidden-layer updates, the entire computed broadcast is held
fixed: no derivatives are taken through \(\mathbf p_i\), \(\mathbf h_{\mathrm{pool},i}\),
or \(\boldsymbol\phi_{\ell,i}\). The derivatives of each convolution's own
BN statistics are included as in Section~\ref{app:resnet-bn}.

\subsection{Update scaling}
\label{app:resnet-scaling}

Before AdamW, \(\mathbf M\) weights broadcast coordinates and \(c_k\)
rescales each convolution's activation signal before local differentiation.
AdamW separately maintains moments of the resulting parameter gradients.

\paragraph{Score-coordinate scaling.}
To downweight broadcast coordinates with larger mean-square magnitudes,
track their second moments, initialized at \(v_{\delta,j}[0]=0\):
\[
    v_{\delta,j}[m]=0.99v_{\delta,j}[m-1]
      +\frac{0.01}{B}\sum_{i=1}^{B}\widetilde\delta_{j,i}[m]^2,
    \qquad j=1,\ldots,110.
\]
Initially \(\mathbf M[1]=\mathbf I_{110}\). Every \(50\) minibatches,
after all convolutions have used \(\mathbf M[m]\), form the bias-corrected
values \(\bar v_{\delta,j}[m]=v_{\delta,j}[m]/(1-0.99^m)\) and set
\begin{equation}
\begin{aligned}
    r[m]&=\max\left\{\frac{0.05}{110}\sum_{j=1}^{110}\bar v_{\delta,j}[m],10^{-8}\right\},\\
    w_j[m]&=(\bar v_{\delta,j}[m]+r[m])^{-1},\\
    \mathbf M[m+1]&=\operatorname{diag}\!\left(
      \frac{\mathbf w[m]}{110^{-1}\sum_{j=1}^{110}w_j[m]}\right).
\end{aligned}
    \label{eq:sbd-resnet-metric}
\end{equation}
Thus the inverse-second-moment weights have mean one. At other updates,
\(\mathbf M[m+1]=\mathbf M[m]\). Only these diagonal moments enter the
weighting; the correlation matrices remain in the original score coordinates,
with \(\mathbf M\) applied when forming \(\widetilde{\mathbf Q}^{(k)}\).

\paragraph{Convolution-wise scaling.}
For each convolution, normalize \(\mathbf q_H^{(k)}\) by a running
root-mean-square magnitude. Define its current mean square over all
batch, channel, and spatial entries as
\(a_k[m]=\|\mathbf q_H^{(k)}[m]\|_F^2/[BP^{(k)}M^{(k)}N^{(k)}]\).
Set \(v_{q,k}[1]=a_k[1]\), then use
\[
    v_{q,k}[m]=0.95v_{q,k}[m-1]+0.05a_k[m]\quad(m\ge2),
    \qquad c_k[m]=[\max\{v_{q,k}[m],10^{-12}\}]^{-1/2}.
\]
The current minibatch uses this updated \(c_k\), before the activation
gate, batch/spatial reduction, and BN derivatives. It normalizes the
activation signal rather than the final parameter step. Both
\(\mathbf M\) and \(c_k\) are held fixed in local differentiation.

\subsection{Training and accuracy curves}
\label{app:resnet-training}

Standard PyTorch AdamW \citep{paszke2019pytorch,loshchilov2019decoupled}
uses the local hidden-layer gradients and the classifier's minibatch CE
gradient, all computed from the same forward pass before any parameter
update. Table~\ref{tab:sbd-resnet-training} gives the data and optimizer settings.

\begin{table}[ht]
\centering
\caption{Training settings for the residual-network experiment.}
\label{tab:sbd-resnet-training}
\small
\begin{tabular}{@{}>{\raggedright\arraybackslash}p{0.23\linewidth}
                  >{\raggedright\arraybackslash}p{0.71\linewidth}@{}}
\toprule
Setting & Value \\
\midrule
Data & CIFAR-10: official \(50{,}000\) train / \(10{,}000\) test;
no held-out validation subset \\
Training budget & \(400\) epochs \\
Batches and precision & Train/evaluation batch sizes \(64/256\);
shuffled training batches; incomplete training batch dropped; FP32 \\
Augmentation & Four-pixel reflection padding, random \(32\times32\) crop,
horizontal flip with probability \(0.5\) \\
Normalization & RGB means \((0.4914,0.4822,0.4465)\); SDs
\((0.2470,0.2435,0.2616)\) \\
AdamW & \(\beta_1=0.9\), \(\beta_2=0.999\), \(\epsilon=10^{-8}\) \\
Weight decay & \(10^{-4}\) on convolution/classifier weights; none on BN
parameters or classifier bias \\
Evaluation & All \(10{,}000\) test images; no augmentation or ensemble;
running BN statistics \\
\bottomrule
\end{tabular}
\end{table}

For fractional epoch \(t\in[0,400]\), the common learning-rate multiplier is
\begin{equation}
 s(t)=\begin{cases}
 \max\{t/5,10^{-4}\},&0\le t<5,\\
 1,&5\le t\le30,\\
 0.02+0.98[1+\cos(\pi(t-30)/370)]/2,&30<t\le400.
 \end{cases}
 \label{eq:sbd-resnet-schedule}
\end{equation}
Let \(\tau_k=0\) for the stem and \(\tau_k=1/3,2/3,1\) for stages
1--3, with projections assigned to their stage. The convolution learning
rate and correlation retention factor are
\begin{align}
 \eta_k(t)&=2\times10^{-4}s(t)(0.35+0.65\tau_k)
 \left[1-0.8(1-\tau_k)\operatorname{clip}
 \!\left(\frac{t-240}{160},0,1\right)\right],\nonumber\\
 \lambda(t)&=1-10^{-5}[0.1+0.9s(\max\{t,5\})].
 \label{eq:sbd-resnet-layer-schedules}
\end{align}
BN uses \(0.1\eta_k(t)\), while the classifier uses
\(2\times10^{-3}s(t)\) without the depth factors.
The depth factors lower earlier-layer rates, with a further reduction after
epoch \(240\); the shared retention \(\lambda(t)\) is applied per minibatch.

\newpage
At the fixed epoch-\(400\) checkpoint, the wide CIFAR ResNet-20 SBD
variant reaches \(83.39\pm0.90\%\) training accuracy and
\(75.90\pm0.60\%\) test accuracy (mean \(\pm\) sample standard deviation
across five independent runs). Figure~\ref{fig:resnet-accuracy-curves}
shows the corresponding trajectories.

\begin{figure}[htbp]
\centering
\includegraphics[width=0.95\linewidth]{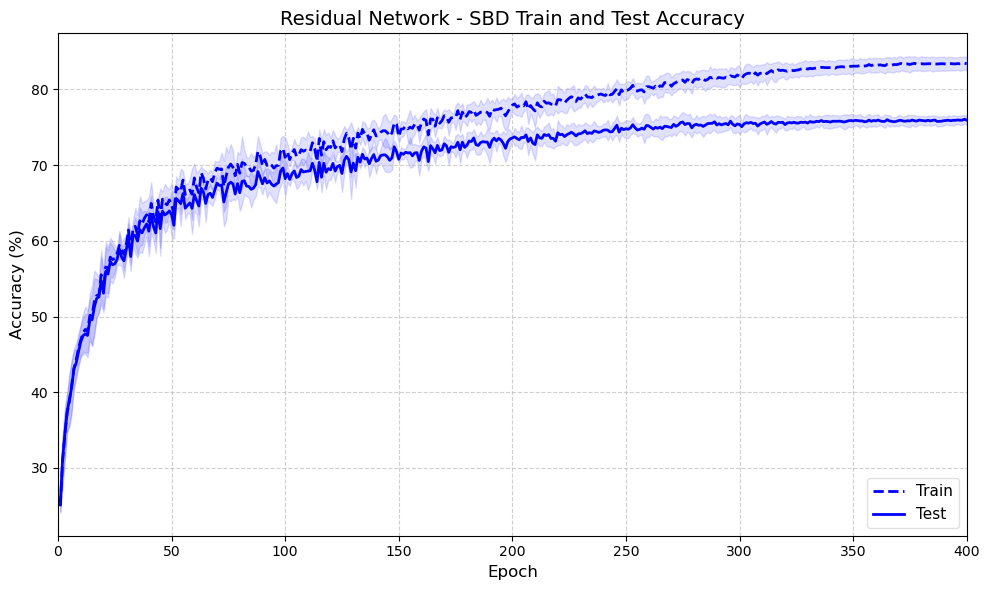}
\caption{Training (dashed) and test (solid) accuracy of the wide CIFAR
ResNet-20 SBD variant across 400 epochs. Curves show the mean across five
runs; shaded regions show one standard deviation.}
\label{fig:resnet-accuracy-curves}
\end{figure}

\end{document}